\documentclass[journal]{IEEEtran}
\usepackage{cite}
\ifCLASSINFOpdf
   \usepackage[pdftex]{graphicx}
\else
  \usepackage[dvips]{graphicx}
\fi
\usepackage{amsmath}
\usepackage{algorithm,algorithmic}
\usepackage{array}
\ifCLASSOPTIONcompsoc
  \usepackage[caption=false,font=normalsize,labelfont=sf,textfont=sf]{subfig}
\else
  \usepackage[caption=false,font=footnotesize]{subfig}
\fi
\usepackage[justification=centering]{caption}
\usepackage{amsfonts,amsthm}
\usepackage{subfig}
\usepackage{dsfont}
\usepackage{color}
\usepackage{pdfpages}

\def\one{\mathds{1}}
\def\qed{\mbox{\rule[0pt]{1.3ex}{1.3ex}}}
\def\bsigma{{\boldsymbol \sigma}}

\DeclareMathOperator*{\argmin}{arg\,min}

\newcommand{\Ex}{\mathbb{E}\hspace{0.05cm}}
\newcommand{\bm}[1]{\mbox{\boldmath $#1$}}
\newcommand{\ba}{\left[ \begin{array}}
\newcommand{\ea}{\\ \end{array} \right]}

\newcommand{\define}{\stackrel{\Delta}{=}}

\newtheorem{assumption}{Assumption}
\newtheorem{remark}{Remark}
\newtheorem{theorem}{Theorem}
\newtheorem{lemma}{Lemma}

\allowdisplaybreaks
\begin{document}
\title{Multi-Agent Fully Decentralized Value Function Learning with Linear Convergence Rates}

\author{Lucas~Cassano,
        Kun~Yuan
        and~Ali~H.~Sayed% <-this % stops a space
\thanks{L. Cassano and K. Yuan are with the Department of Electrical Engineering, UCLA, CA 90095 USA e-mails: \{cassanolucas,kunyuan\}@ucla.edu.}% <-this % stops a space
\thanks{A. H. Sayed is with the School of Engineering, Ecole Polytechnique
	Federale de Lausanne (EPFL), Switzerland. e-mail: ali.sayed@epfl.ch.}}% <-this % stops a space


\markboth{Journal of \LaTeX\ Class Files,~Vol.~4, No.~4, July~2019}%
{Shell \MakeLowercase{\textit{et al.}}: Bare Demo of IEEEtran.cls for IEEE Journals}



\maketitle

% in the abstract or keywords. Between 150 and 250 words.
\begin{abstract}
This work develops a fully decentralized multi-agent algorithm for policy evaluation. The proposed scheme can be applied to two distinct scenarios. In the first scenario, a collection of agents have distinct datasets gathered following different behavior policies (none of which is required to explore the full state space) in different instances of the same environment and they all collaborate to evaluate a common target policy. The network approach allows for efficient exploration of the state space and allows all agents to converge to the optimal solution even in situations where neither agent can converge on its own without cooperation. The second scenario is that of multi-agent games, in which the state is global and rewards are local. In this scenario, agents collaborate to estimate the value function of a target team policy. The proposed algorithm combines off-policy learning, eligibility traces and linear function approximation. The proposed algorithm is of the variance-reduced kind and achieves linear convergence with $O(1)$ memory requirements. The linear convergence of the algorithm is established analytically, and simulations are used to illustrate the effectiveness of the method. 
\end{abstract}


% \fi
\IEEEpeerreviewmaketitle

\section{Introduction}
\IEEEPARstart{T}{he} goal of a policy evaluation algorithm is to estimate the performance that an agent will achieve when it follows a particular policy to interact with an environment, usually modeled as a Markov Decision Process (MDP). Policy evaluation algorithms are important because they are often key parts of more elaborate solution methods where the ultimate goal is to find an optimal policy for a particular task (one such example is the class of actor-critic algorithms -- see \cite{grondman2012survey} for a survey). This work studies the problem of policy evaluation in a fully decentralized setting. We consider two distinct scenarios.

In the first case, $K$ independent agents interact with independent instances of the same environment following potentially different behavior policies to collect data the objective is for the agents to cooperate. In this scenario each agent only has knowledge of its own states and rewards, which are independent of the states and the rewards of the other agents. Various practical situations give rise to this scenario, for example, consider a task that takes place in a large geographic area. The area can be divided into smaller sections, each of which can be explored by a separate agent. This framework is also useful for collective robot learning (see, \cite{kehoe2013cloud,kehoe2015survey,gu2016deep}).

The second scenario we consider is that of multi-agent reinforcement learning (MARL). In this case a group of agents interact simultaneously with a unique MDP and with each other to attain a common goal. In this setting there is a unique global state known to all agents and each agent receives distinct local rewards, which are unknown to the other agents. Some examples that fit into this framework are teams of robots working on a common task such as moving a bulky object, trying to catch a prey, or putting out a fire.

\subsection{Related Work}
 
Our contributions belong to the class of works that deal with policy evaluation, distributed reinforcement learning, and multi-agent reinforcement learning.

There exist a plethora of algorithms for policy evaluation such as GTD \cite{GTD}, TDC \cite{GTD2/TDC}, GTD2 \cite{GTD2/TDC}, GTD-MP/GTD2-MP \cite{liu2015finite}, GTD($\lambda$) \cite{maei2011gradient}, and True Online GTD($\lambda$) \cite{van2014off}. The main feature of these algorithms is that they have guaranteed convergence (for small enough step-sizes) while combining off-policy learning and linear function approximation; and are applicable to scenarios with streaming data. They are also applicable to cases with a finite amount of data. However, in this latter situation, they have the drawback that they converge at a sub-linear rate because a decaying step-size is necessary to guarantee convergence to the minimizer. In most current applications, policy evaluation is actually carried out after collecting a finite amount of data (one example is the recent success in the game of GO \cite{GO}). Therefore, deriving algorithms with better convergence properties for the finite sample case becomes necessary. By leveraging recent developments in variance-reduced algorithms, such as SVRG \cite{johnson2013accelerating} and SAGA \cite{defazio2014saga}, the work \cite{du2017stochastic} presented SVRG and SAGA-type algorithms for policy evaluation. These algorithms combine GTD2 with SVRG and SAGA and they have the advantage over GTD2 in that linear convergence is guaranteed for fixed data sets. Our work is related to \cite{du2017stochastic} in that we too use a variance-reduced strategy, however we use the AVRG strategy \cite{AVRG} which is more convenient for distributed implementations because of an important balanced gradient calculation feature.

Another interesting line of work in the context of distributed policy evaluation is \cite{macua2015distributed}, \cite{stankovic2016multi} and \cite{cassano_ECC}. In \cite{macua2015distributed} and \cite{stankovic2016multi} the authors introduce Diffusion GTD2 and ALG2; which are extensions of GTD2 and TDC to the fully decentralized case, respectively. While \cite{cassano_ECC} is a shorter version of this work. These algorithms consider the situation where independent agents interact with independent instances of the same MDP. These strategies allow individual agents to converge through collaboration even in situations where convergence is infeasible without collaboration. The algorithm we introduce in this paper can be applied to this setting as well and has two main advantages over \cite{macua2015distributed} and \cite{stankovic2016multi}. First, the proposed algorithm has guaranteed linear convergence, while the previous algorithms converge at a sub-linear rate. Second, while in some instances, the solutions in \cite{macua2015distributed} and \cite{stankovic2016multi} may be biased \textcolor{black}{due to the use of the Mean Square Projected Bellman Error (MSPBE) as a surrogate cost (this point is further clarified in Section \ref{sec: Problem Setting})}, the proposed method allows better control of the bias term due to a modification in the cost function. We extend our previous work \cite{cassano_ECC} in four main ways which we discuss in the \textit{Contribution} sub-section.

There is also a good body of work on multi-agent reinforcement learning (MARL). However, most works in this area focus on the policy optimization problem instead of the policy evaluation problem. \textcolor{black}{The work that is closer to the current contribution is \cite{wai2018multi}, which was pursued simultaneously and independently of our current work. The goal of the formulation in \cite{wai2018multi} is to derive a linearly-convergent distributed policy evaluation procedure for MARL. The work \cite{wai2018multi} does not consider the case where independent agents interact with independent MDPs. In the context of MARL, our proposed technique has three advantages in comparison to the approach from \cite{wai2018multi}. First, the memory requirement of the algorithm in \cite{wai2018multi} scales linearly with the amount of data (i.e., $O(N)$), while the memory requirement for the proposed method in this manuscript is $O(1)$), i.e., it is independent of the amount of data. Second, the algorithm of \cite{wai2018multi} does not include the use of eligibility traces; a feature that is often necessary to reach state of the art performance (see, for example, \cite{dai2018sbeed,mnih2016asynchronous}). Finally, the algorithm from \cite{wai2018multi} requires all agents in the network to sample their data points in a synchronized manner, while the algorithm we propose in this work does not require this type of synchronization. Another paper that is related to the current work is \cite{zhang2018fully}, which considers the same distributed MARL as we do; although their contribution is different from ours. The main contribution in \cite{zhang2018fully} is to extend the policy gradient theorem to the MARL case and derive two fully distributed actor-critic algorithms with linear function approximation for policy optimization. The connection between \cite{zhang2018fully} and our work is that their actor-critic algorithms require a distributed policy evaluation algorithm. The algorithm they use is similar to \cite{macua2015distributed} and \cite{stankovic2016multi} (they combine diffusion learning \cite{sayed2014adaptive} with standard TD instead of GTD2 and TDC as was the case in \cite{macua2015distributed} and \cite{stankovic2016multi}). The algorithm we present in this paper is compatible with their actor-critic schemes (i.e., it could be used as the critic), and hence could potentially be used to augment their performance and convergence rate.}

\textcolor{black}{Our work is also related to the literature on distributed optimization. Some notable works in this area include \cite{sayed2014adaptive,nedic2009distributed,yuan2016convergence,admm,shi2015extra,nedic2017achieving,qu2017harnessing,xi2017dextra,yuan2017exact1,yuan2017exact2,sayed2014adaptation}. Consensus \cite{nedic2009distributed} and Diffusion \cite{sayed2014adaptive} constitute some of the earliest work in this area. These methods can converge to a neighborhood around, but not exactly to, the global minimizer when constant step-sizes are employed \cite{yuan2016convergence,sayed2014adaptation}. Another family of methods is based on distributed alternating direction method of multipliers (ADMM) \cite{admm}. While these methods can converge linearly fast to the exact global minimizer, they are computationally more expensive than previous methods since they need to optimize a sub-problem at each iteration. An exact first-order algorithm (EXTRA) was proposed in \cite{shi2015extra} for undirected networks to correct the bias suffered by consensus, (this work was later extended for the case of directed networks \cite{xi2017dextra}). EXTRA and DEXTRA \cite{xi2017dextra} can also converge linearly to the global minimizer while maintaining the same computational efficiency as consensus and diffusion. Several other works employ instead a gradient tracking strategy \cite{next,nedic2017achieving,qu2017harnessing}. These works guarantee linear convergence to the global minimizer even when they operate over time-varying networks. More recently, the Exact Diffusion algorithm \cite{yuan2017exact1,yuan2017exact2} has been introduced for static undirected graphs. This algorithm has a wider stability range than EXTRA (and hence exhibits faster convergence \cite{yuan2017exact2}), and for the case of static graphs is more communication efficient than gradient tracking methods since the gradient vectors are not shared among agents. Our current work closely related to Exact Diffusion since our MARL model is based on static undirected graphs and our distributed strategy is derived in a similar manner to \textit{Exact Diffusion}. We remark that there is a fundamental difference between the algorithm we present and the works in \cite{sayed2014adaptive,nedic2009distributed,yuan2016convergence,admm,shi2015extra,nedic2017achieving,qu2017harnessing,xi2017dextra,yuan2017exact1,yuan2017exact2,sayed2014adaptation}, namely, our algorithm finds the global {\em saddle-point} in a primal dual formulation while the cited works solve convex minimization problems.}

\subsection{Contribution}
The contribution of this paper is twofold. In the first place, we introduce \textit{Fast Diffusion for Policy Evaluation} (FDPE), a fully decentralized policy evaluation algorithm under which all agents have a guaranteed linear convergence rate to the minimizer of the global cost function. The algorithm is designed for the finite data set case and combines off-policy learning, eligibility traces, and linear function approximation. The eligibility traces are derived from the use of a more general cost function and they allow the control of the bias-variance trade-off we mentioned previously. In our distributed model, a fusion center is not required and communication is only allowed between immediate neighbors. The algorithm is applicable both to distributed situations with independent MDPs (i.e., independent states and rewards) and to MARL scenarios (i.e., global state and independent rewards). To the best of our knowledge, this is the first algorithm that combines all these characteristics. Our second contribution is a novel proof of convergence for the algorithm. \textcolor{black}{This proof is challenging due to the combination of three factors: the distributed nature of the algorithm, the primal-dual structure of the cost function, and the use of stochastic biased gradients as opposed to exact gradients.}

This work expands our short work \cite{cassano_ECC} in four ways. In the first place, in that work we used the MSPBE as a cost function, while in this work we employ a more general cost function. Second, we include the proof of convergence. Third, we show that our approach applies to MARL scenarios, while in our previous short paper we only discussed the distributed policy evaluation scenario with independent MDPs. Finally in this paper we provide more extensive simulations.

\subsection{Notation and Paper Outline}
Matrices are denoted by upper case letters, while vectors are denoted with lower case. Random variables and sets are denoted with bold font and calligraphic font, respectively. $\rho(A)$ indicates the spectral radius of matrix A. \textcolor{black}{$I_M$ is the identity matrix of size $M$}. $\Ex_{g}$ is the expected value with respect to distribution $g$. $\|\cdot\|_D$ refers to the weighted matrix norm, where $D$ is a diagonal positive definite matrix. We use $\preceq$ to denote entry-wise inequality. \textcolor{black}{$\text{col}\{v(n)\}_{n=1}^N$ is a column vector with elements $v(1)$ through $v(N)$ (where $v(N)$ is at the bottom)}. Finally $\mathbb{R}$ and $\mathbb{N}$ represent the sets of real and natural numbers, respectively.

The outline of the paper is as follows. In the next section we introduce the framework under consideration. In Section \ref{Sec: Distributed} we derive our algorithm and provide a theorem that guarantees linear convergence rate. In Section \ref{Sec:Marl} we discuss the MARL setting. Finally we show simulation results in Section \ref{Sec: Simulations}.

\section{Problem Setting}
\label{sec: Problem Setting}

\subsection{Markov Decision Processes and the Value Function}
We consider the problem of policy evaluation within the traditional reinforcement learning framework. \textcolor{black}{We recall that the objective of a policy evaluation algorithm is to estimate the performance of a known target policy using data generated by either the same policy (this case is referred as \textit{on-policy}), or a different policy that is also known (this case is referred as \textit{off-policy})}. We model our setting as a finite Markov Decision Process (MDP), with an MDP defined by the tuple ($\mathcal{S}$,$\mathcal{A}$,$\mathcal{P}$,$r$,$\gamma$), where $\mathcal{S}$ is a set of states of size $S=|\mathcal{S}|$, $\mathcal{A}$ is a set of actions of size $A=|\mathcal{A}|$, $\mathcal{P}(s'|s,a)$ specifies the probability of transitioning to state $s'\in\mathcal{S}$ from state $s\in\mathcal{S}$ having taken action $a\in\mathcal{A}$, $r:\mathcal{S}\times\mathcal{A}\times\mathcal{S}\rightarrow\mathbb{R}$ is the reward function $r(s,a,s')$ when the agent transitions to state $s'\in\mathcal{S}$ from state $s\in\mathcal{S}$ having taken action $a\in\mathcal{A}$), and $\gamma\in[0,1)$ is the discount factor.

Even though in this paper we analyze the distributed scenario, in this section we motivate the cost function for the single agent case for clarity of exposition and in the next section we generalize it to the distributed setting. We thus consider an agent that wishes to learn the value function, $v^{\pi}(s)$, for a target policy of interest $\pi(a|s)$, while following a potentially different behavior policy $\phi(a|s)$. Here, the notation $\pi(a|s)$ specifies the probability of selecting action $a$ at state $s$. We recall that the value function for a target policy $\pi$, starting from some initial state $s\in\mathcal{S}$ at time $i$, is defined as follows:
\begin{equation} \label{eq:value_function_definition}
v^{\pi}(s)=\Ex_{\mathcal{P},\pi}\bigg(\sum_{t=i}^{\infty}\gamma^{t-i}\bm{r}(\bm{s}_t,\bm{a}_t,\bm{s}_{t+1})\Big|\bm{s}_i=s\bigg)
\end{equation}
where $\bm{s}_t$ and $\bm{a}_t$ are the state and action at time $t$, respectively. Note that since we are dealing with a constant target policy $\pi$, the transition probabilities between states, which are given by $p_{s,s'}^{\pi}=\Ex_\pi{\mathcal{P}(s'|s,\bm{a})}$, are fixed and hence the MDP reduces to a Markov Rewards Process. In this case, the state evolution of the agent can be modeled with a Markov Chain with transition matrix $P^{\pi}$ whose entries are given by $(P^{\pi})_{ij}=p_{i,j}^{\pi}$.
\begin{assumption}
	We assume that the Markov Chain induced by the behavior policy $\phi(a|s)$ is aperiodic and irreducible. In view of the Perron-Frobenius Theorem \cite{sayed2014adaptation}, this condition guarantees that the Markov Chain under $\phi(a|s)$ will have a steady-state distribution in which every state has a strictly positive probability of visitation \cite{sayed2014adaptation}.\hfill\qed
\end{assumption}
Using the matrix $P^{\pi}$ and defining:
\begin{align}
	&v^\pi\hspace{-1.3mm}=\hspace{-0.5mm}\text{col}\{\hspace{-0.5mm}v^{\hspace{-0.3mm}\pi}\hspace{-0.7mm}(\hspace{-0.3mm}s\hspace{-0.3mm})\hspace{-0.5mm}\}_{s=1}^S,\hspace{1mm}r^{\pi}\hspace{-0.7mm}(s)\hspace{-0.7mm}=\Ex_{\pi,\mathcal{P}}\big(\bm{r}(s,\boldsymbol{a},\boldsymbol{s'})\big),\hspace{1mm}r^\pi\hspace{-1mm}=\text{col}\{r^{\hspace{-0.3mm}\pi}\hspace{-0.7mm}(\hspace{-0.3mm}s\hspace{-0.3mm})\hspace{-0.5mm}\}_{s=1}^S
\end{align}
we can rewrite \eqref{eq:value_function_definition} in matrix form as:
\begin{align}
v^\pi=\sum_{n=0}^{\infty}(\gamma P^{\pi})^nr^\pi=(I-\gamma P^{\pi})^{-1}r^\pi
\end{align}
Note that the inverse $(I-\gamma P^{\pi})^{-1}$ always exists; this is because $\gamma<1$ and the matrix $P^{\pi}$ is right stochastic with spectral radius equal to one. We further note that $v^\pi$ also satisfies the following $h-$stage Bellman equation for any $h\in\mathbb{N}$:\vspace*{-0.5mm}
\begin{align}\label{eq:K_Bellman_equation}
v^\pi=(\gamma P^{\pi})^{h}v^\pi+\sum_{n=0}^{h-1}(\gamma P^{\pi})^nr^\pi
\end{align}

\subsection{Definition of cost function}
We are interested in applications where the state space is too large (or even infinite) and hence some form of function approximation is necessary to reduce the dimensionality of the parameters to be learned. As we anticipated in the introduction, in this work we use linear approximations\footnote{We choose linear function approximation, not just because it is mathematically convenient (since with this approximation our cost function is strongly convex) but because there are theoretical justifications for this choice. In the first place, in some domains (for example Linear Quadratic Regulator problems) the value function is a linear function of known features. Secondly, when policy evaluation is used to estimate the gradient of a policy in a policy gradient algorithm, the policy gradient theorem \cite{sutton2000policy} assures that the exact gradient can be obtained even when a linear function is used to estimate $v^\pi$.}. More formally, for every state $s\in\mathcal{S}$, we approximate $v^{\pi}(s)\approx x_s^T\theta^{\star}$ where $x_s\in\mathbb{R}^{M}$ is a feature vector corresponding to state $s$ and $\theta^{\star}\in\mathbb{R}^{M}$ is a parameter vector such that $M\ll S$. Defining $X=[x_1,x_2,\cdots,x_S]^T\in\mathbb{R}^{S\times M}$, we can write a vector approximation for $v^{\pi}$ as $v^{\pi}\approx X\theta^{\star}$.  We assume that $X$ is a full rank matrix; this is not a restrictive assumption since the feature matrix is a design choice. It is important to note though that the true $v^{\pi}$ need not be in the range space of $X$. If $v^{\pi}$ is in the range space of $X$, an equality of the form $v^{\pi}=X\theta^{\star}$ holds exactly and the value of $\theta^{\star}$ is unique (because $X$ is full rank) and given by $\theta^{\star}\hspace{-0.5mm}=\hspace{-0.5mm}(\hspace{-0.5mm}X^{\hspace{-0.4mm}T}\hspace{-0.5mm}X)^{\hspace{-0.4mm}-\hspace{-0.4mm}1}X^{\hspace{-0.4mm}T}\hspace{-0.5mm}v^\pi$. For the more general case where $v^{\pi}$ is not in the range space of $X$, then one sensible choice for $\theta^{\star}$ is:
\begin{align}\label{eq:t_star}
\theta^{\star}=\argmin_\theta\|X\theta-v^\pi\|_D^2=(X^TDX)^{-1}X^TDv^\pi
\end{align}
where $D$ is some positive definite weighting matrix to be defined later. Although \eqref{eq:t_star} is a reasonable cost to define $\theta^{\star}$, it is not useful to derive a learning algorithm since $v^\pi$ is not known beforehand. As a result, for the purposes of deriving a learning algorithm, another cost (one whose gradients can be sampled) needs to be used as a surrogate for \eqref{eq:t_star}. One popular choice for the surrogate cost is the MSPBE (see, e.g., \cite{GTD2/TDC,liu2015finite,macua2015distributed,stankovic2016multi}); this cost has the inconvenience that its minimizer $\theta^o$ is different from \eqref{eq:t_star} and some bias is incurred\textcolor{black}{ \cite{GTD2/TDC}.} In order to control the magnitude of the bias, we shall derive a generalization of the MSPBE which we refer to as $H-$truncated $\lambda$-weighted Mean Square Projected Bellman Error (H$\lambda$-MSPBE). To introduce this cost, we start by writing a convex combination of equation \eqref{eq:K_Bellman_equation} with different $h$'s ranging from 1 to $H$ (we choose $H$ to be a finite amount instead of $H\rightarrow\infty$ because in this paper we deal with finite data instead of streaming data) as follows:
%\begin{align}
%v^\pi&=(1-\lambda)\sum_{h=1}^{H-1}\lambda^{h-1}\bigg((\gamma P^{\pi})^{h}v^\pi+\sum_{n=0}^{h-1}(\gamma P^{\pi})^nr^\pi\bigg)\nonumber\\
%&+\lambda^{H-1}\bigg((\gamma P^{\pi})^{H}v^\pi+\sum_{n=0}^{H-1}(\gamma P^{\pi})^nr^\pi\bigg)\nonumber\\
%&=\Gamma_2(\lambda,H)r^\pi+\rho_1(\lambda,H)\Gamma_1(\lambda,H)v^\pi\label{eq:l_weighted_BE}
%\end{align}
\begin{align}
v^\pi&=(1-\lambda)\sum_{h=1}^{H-1}\lambda^{h-1}\bigg((\gamma P^{\pi})^{h}v^\pi+\sum_{n=0}^{h-1}(\gamma P^{\pi})^nr^\pi\bigg)\nonumber\\
&\hspace{14mm}+\lambda^{H-1}\bigg((\gamma P^{\pi})^{H}v^\pi+\sum_{n=0}^{H-1}(\gamma P^{\pi})^nr^\pi\bigg)\\
&=\Gamma_2(\lambda,H)r^\pi+\rho_1(\lambda,H)\Gamma_1(\lambda,H)v^\pi\label{eq:l_weighted_BE}
\end{align}
where we introduced:
\begin{align}\label{eq:definitions_ggr}
&\rho_1(\lambda,H)\hspace{-0.5mm}=\hspace{-0.5mm}\frac{(1-\lambda)\gamma+(1-\gamma)(\gamma\lambda)^H}{1-\gamma\lambda}\\[-2.5pt]
&\Gamma_2(\lambda,H)\hspace{-0.5mm}=\hspace{-1mm}\sum_{n=0}^{H-1}(\gamma\lambda P^{\pi})^n\hspace{-1mm}=\hspace{-0.5mm}\big(I-(\gamma\lambda P^{\pi})^H\big)(I-\gamma\lambda P^{\pi})^{-1}\\[-3pt]
&\Gamma_{\hspace{-0.5mm}1}(\hspace{-0.5mm}\lambda,H)\hspace{-0.5mm}=\hspace{-0.5mm}\frac{1}{\rho_1\hspace{-0.5mm}(\lambda,H)}\hspace{-0.5mm}\bigg(\hspace{-1.2mm}(1\hspace{-0.5mm}-\hspace{-0.5mm}\lambda)\gamma P^{\pi}\hspace{-0.5mm}\sum_{n=0}^{H-1}\hspace{-0.5mm}(\gamma\lambda P^{\pi})^n\hspace{-0.5mm}+(\hspace{-0.5mm}\gamma\lambda P^{\pi})^{\hspace{-0.5mm}H}\hspace{-1mm}\bigg)
\end{align}
and $0\leq\lambda\leq1$ is a parameter that controls the bias.
\begin{remark}
	Note that $0<\rho_1(\lambda,H)\leq\gamma<1$.\hfill\qed
\end{remark}
\begin{remark}
	$\Gamma_1(\lambda,H)$ is a right stochastic matrix because it is defined as a convex combination of powers of $P^\pi$ (which are right stochastic matrices).\hfill\qed
\end{remark}
\textcolor{black}{Note that from now on for the purpose of simplifying the notation, we refer to $\rho_1(\lambda,H)$, $\Gamma_{\hspace{-0.5mm}1}(\lambda,H)$ and $\Gamma_{\hspace{-0.5mm}2}(\lambda,H)$ as $\rho_1$, $\Gamma_{\hspace{-0.5mm}1}$ and $\Gamma_{\hspace{-0.5mm}2}$, respectively.} Replacing $v^\pi$ in \eqref{eq:l_weighted_BE} by its linear approximation we get:\vspace{-0.5mm}
\begin{align}
	X\theta\approx\Gamma_2r^\pi+\rho_1\Gamma_1X\theta
\end{align}
Projecting the right hand side onto the range space of $X$ so that an equality holds, we arrive at:\vspace{-0.5mm}
\begin{align}\label{eq:l_weighted_Projected_Bellman_Equation}
	X\theta=\Pi\left[\Gamma_2r^\pi+\rho_1\Gamma_1X\theta\right]
\end{align}
where $\Pi\in\mathbb{R}^{S\times S}$ is the weighted projection matrix onto the space spanned by $X$, (i.e., $\Pi=X(X^TDX)^{-1}X^TD$). We can now use \eqref{eq:l_weighted_Projected_Bellman_Equation} to define our surrogate cost function:
\begin{align}\label{eq:surrogate}
S(\theta)&=\frac{1}{2}\Bigl{\|}\Pi\big(\Gamma_2r^\pi+\rho_1\Gamma_1X\theta\big)-X\theta\Bigr{\|}_D^2+\frac{\eta}{2}\bigl{\|}\theta-\theta_{\text{p}}\bigr{\|}_U^2
\end{align}
where the first term on the right hand side is the H$\lambda$-MSPBE, $\eta\geq0$ is a regularization parameter, $U>0$ is a symmetric positive-definite weighting matrix, and $\theta_{\text{p}}$ reflects prior knowledge about $\theta$. Two sensible choices for $U$ are $U=I$ and $U=X^TDX=C$, which reflect previous knowledge about $\theta$ or the value function $X\theta$, respectively. The regularization term can be particularly useful when the policy evaluation algorithm is used as part of a policy gradient loop (since subsequent policies are expected to have similar value functions and the value of $\theta$ learned in one iteration can be used as $\theta_{\text{p}}$ in the next iteration) like, for example, in \cite{macua2017diff}. One main advantage of using the proposed cost \eqref{eq:surrogate} instead of the more traditional MSPBE cost \textcolor{black}{is that the magnitude of the bias between its minimizer (denoted as $\theta^o(H,\lambda)$) and the desired solution $\theta^\star$ can be controlled through $\lambda$ and $H$}. To see this, we first rewrite $S(\theta)$ in the following equivalent form:
\begin{align} 
S(\hspace{-0.3mm}\theta\hspace{-0.3mm})\hspace{-0.5mm}=&\frac{1}{2}\bigl{\|}X^{\hspace{-0.5mm}T}\hspace{-1mm}D(I\hspace{-0.5mm}-\hspace{-0.5mm}\rho_1\hspace{-0.5mm}\Gamma_{\hspace{-0.7mm}1})X\theta-X^T\hspace{-1mm}D\Gamma_2r^\pi\bigr{\|}_{(X^{\hspace{-0.4mm}T}\hspace{-0.9mm}DX)^{\hspace{-0.4mm}-\hspace{-0.4mm}1}}^2\hspace{-0.6mm}+\frac{\eta}{2}\bigl{\|}\theta-\theta_{\text{p}}\bigr{\|}_U^2\label{eq:PBE}
\end{align}
Next, we introduce the quantities:
\begin{align}\label{eq:definition_AbC}
	A&=X^TD(I-\rho_1\Gamma_1)X,\hspace{2mm}b=X^TD\Gamma_2r^\pi,\hspace{2mm}C=X^TDX
\end{align}
\begin{remark}
	$A$ is an invertible matrix.
\end{remark}
\begin{proof}
	Due to remarks 1 and 2 we have that the spectral radius of $\rho_1(\lambda,H)\Gamma_1(\lambda,H)$ is strictly smaller than one, and hence $I-\rho_1(\lambda,H)\Gamma_1(\lambda,H)$ is invertible. The result follows by recalling that $X$ and $D$ are full rank matrices.
\end{proof}
The minimizer of \eqref{eq:PBE} is given by:
\begin{align}\label{eq:parc_min}
	\theta^o(H,\lambda)=(A^TC^{-1}A+\eta U)^{-1}(\eta U\theta_\text{p}+A^TC^{-1}b)
\end{align}
where $(\hspace{-0.5mm}A^T\hspace{-0.5mm}C^{-1}\hspace{-0.5mm}A+\eta U)^{-1}$ exists and hence $\theta^o(H,\lambda)$ is well defined. This is because $\eta U$ is positive-definite and $A$ is invertible. Also note that when $\lambda=1$, $H\rightarrow\infty$ and $\eta=0$, $\theta^o(H,\lambda)$ reduces to \eqref{eq:t_star} and hence the bias is removed. We do not fix $\lambda=1$ because while the bias diminishes as $\lambda\rightarrow1$, the estimate of the value function approaches a Monte Carlo estimate and hence the variance of the estimate increases. \textcolor{black}{Note from \eqref{eq:l_weighted_BE} and \eqref{eq:PBE} that in the particular case where the value function $v^\pi$ lies in the range space of $X$ (and there is no regularization, i.e., $\eta=0$) there is no bias \big(i.e., $\theta^\star=\theta^o(H,\lambda)$\big) independently of the values of $\lambda$ and $H$. This observation shows that when there is bias between $\theta^\star$ and $\theta^o(H,\lambda)$, the bias arises from the fact that the value function being estimated does not lie in the range space of $X$. In practice, $\lambda$ offers a valuable bias-variance trade-off, and its optimal value depends on each particular problem. Note that since we are dealing with finite data samples, in practice, $H$ will always be finite. Therefore, eliminating the bias completely is not possible (even when $\lambda=1$). The exact expression for the bias is obtained by subtracting \eqref{eq:parc_min} from \eqref{eq:t_star}. However, this expression does not easily indicate how the bias behaves as a function of $\gamma$, $\lambda$ and $H$. Lemma \ref{lemma:bias} provides a simplified expression.
\begin{lemma}\label{lemma:bias}
	The bias $\|\theta^o(\hspace{-0.6mm}H,\lambda\hspace{-0.6mm})-\theta^\star\hspace{-0.5mm}\|^2$ is approximated by:
	\begin{align}\label{eq:bias}
		&\|\hspace{-0.5mm}\theta^o(\hspace{-0.8mm}H,\lambda\hspace{-0.6mm})\hspace{-0.4mm}-\theta^\star\hspace{-0.5mm}\|^2\hspace{-1mm}\approx\hspace{-1mm}\left(\hspace{-1mm}\mathbb{I}\big(v^\pi\hspace{-1.5mm}\neq\hspace{-0.5mm}\Pi v^\pi\hspace{-0.5mm}\big)\hspace{-0.5mm}\frac{\kappa_2\rho_1}{(\hspace{-0.5mm}1\hspace{-0.5mm}+\hspace{-0.5mm}\kappa_{\hspace{-0.5mm}1}\hspace{-0.2mm}\eta)(\kappa_3\hspace{-0.5mm}-\hspace{-0.5mm}\rho_1\hspace{-0.5mm})}+\frac{\kappa_1\eta\|\theta_{\text{p}}-\theta^\star\|}{1+\kappa_1\eta}\right)^{\hspace{-1mm}2}
	\end{align}
	where
	\begin{align}
		&\frac{\rho_1}{\kappa_3-\rho_1}=\frac{(1-\lambda)\gamma+(1-\gamma)(\gamma\lambda)^H}{\kappa_3(1-\gamma\lambda)-(1-\lambda)\gamma-(1-\gamma)(\gamma\lambda)^H}
	\end{align}
	for some constants $\kappa_1$, $\kappa_2$ and $\kappa_3$.
\end{lemma}
\begin{proof}
	See Appendix \ref{app:bias}.
\end{proof}
\hspace{-3.7mm}In the statement of the lemma, the notation $\mathbb{I}$ is the indicator function. Note that expression \eqref{eq:bias} agrees with our previous discussion and with several intuitive facts. First, due to the indicator function, if $v^\pi$ lies in the range space of $X$ there is no bias independently  of the values of $\gamma$, $\lambda$ and $H$ (as long as $\eta=0$). Second, if $\lambda=0$, the bias is independent of $H$ (because when $\lambda=0$ all terms that depend on $H$ are zeroed). Third, if $H=1$ then the bias is independent of the value of $\lambda$ (because when $H=1$ all terms that depend on $\lambda$ are zeroed). Furthermore, the expression is monotone decreasing in $\lambda$ (for the case where $H>1$) which agrees with the intuition that the bias diminishes as $\lambda$ increases. Finally, we note that the bias is minimized for $\lambda=1$ and in this case there is still a bias, which if $\eta=0$, is on the order of $\mathcal{O}\big(\gamma^H/(\kappa_3-\gamma^H)\big)$. This explicitly shows the effect on the bias of having a finite $H$. The following lemma describes the behavior of the variance.
\begin{lemma}\label{lemma:variance}
	The variance of the estimate $\widehat{\theta}^o(H,\lambda)$ is approximated by:
	\begin{align}\label{eq:variance}
	&\Ex\hspace{-0.5mm}\big\|\hspace{-0.5mm}\widehat{\theta^o}(\hspace{-0.8mm}H,\lambda\hspace{-0.6mm})-\theta^o\hspace{-0.5mm}(\hspace{-0.8mm}H,\lambda\hspace{-0.6mm})\hspace{-0.5mm}\big\|^2\hspace{-1.5mm}\approx\hspace{-1mm}\frac{\kappa_4}{(1+\kappa_1\eta)^2(N\hspace{-0.5mm}-\hspace{-0.5mm}H)}\hspace{-0.5mm}\left(\hspace{-0.5mm}\frac{1-(\gamma\lambda)^{2H}}{1-(\gamma\lambda)^2}\right)
	\end{align}	
	for some constants $\kappa_1$ and $\kappa_4$.
\end{lemma}
\begin{proof}
	See Appendix \ref{app:variance}.
\end{proof}}
\textcolor{black}{
Note that \eqref{eq:variance} is monotone increasing as a function of $\lambda$ (for $H>1$) and as a function of $H$ (for $\lambda>0$). Adding expressions \eqref{eq:bias} and \eqref{eq:variance} shows explicitly the bias-variance trade-off handled by the parameter $\lambda$ and the finite horizon $H$. We remark that the idea of an eligibility trace parameter $\lambda$ as a bias-variance trade-off is not novel to this paper and has been previously used in algorithms such as $TD(\lambda)$ \cite{sutton1988learning}, $TD(\lambda)$ with \textit{replacing traces} \cite{replacing_traces}, GTD($\lambda$) \cite{maei2011gradient} and True Online GTD($\lambda$) \cite{van2014off}. Note however, that these works derive algorithms for the on-line case (as opposed to the batch setting) using different cost functions. Therefore, the expressions we present in this paper are different from previous works, which is why we derive them in detail. Moreover, the expressions corresponding to Lemmas \ref{lemma:bias} and \ref{lemma:variance} that quantify such bias-variance trade-off for are new and specific for our batch model.}

At this point, all that is left to fully define the surrogate cost function $S(\theta)$ is to choose the positive definite matrix $D$. The algorithm that we derive in this paper is of the stochastic gradient type. With this in mind, we shall choose $D$ such that the quantities $A$, $b$ and $C$ turn out to be expectations that can be sampled from data realizations. Thus, we start by setting $D$ to be a diagonal matrix with positive entries; we collect these entries into a vector $d^{\phi}$ and write $D^{\phi}$ instead of $D$, i.e., $D=D^{\phi}=\text{diag}(d^{\phi})$. We shall select $d^{\phi}$ to correspond to the steady-state distribution of the Markov chain induced by the behavior policy, $\phi(a|s)$. This choice for $D$ not only is convenient in terms of algorithm derivation, it is also physically meaningful; since with this choice for $D$, states that are visited more often are weighted more heavily while states which are rarely visited receive lower weights. As a consequence of Assumption 1 and the Perron-Frobenius Theorem \cite{sayed2014adaptation}, the vector $d^{\phi}$ is guaranteed to exist and all its entries will be strictly positive and add up to one. Moreover, this vector satisfies ${d^{\phi}}^TP^{\phi}={d^{\phi}}^T$ where $P^{\phi}$ is the transition probability matrix defined in a manner similar to $P^{\pi}$.
\begin{lemma}\label{lemma:expectations}
	Setting $D=\text{diag}(d^{\phi})$, the matrices $A$, $b$ and $C$ can be written as expectations as follows:
	\begin{subequations}
		\begin{align} 
		&A=\Ex_{d^\phi,\mathcal{P},\pi}\hspace{-1mm}\bigg[\hspace{-0.5mm}\bm{x}_{t}\hspace{-0.5mm}\bigg(\hspace{-1mm}\bm{x}_{t}-\hspace{-0.5mm}\gamma(1\hspace{-0.5mm}-\hspace{-0.5mm}\lambda)\hspace{-1mm}\sum_{n=0}^{H-1}\hspace{-1mm}(\hspace{-0.3mm}\gamma\lambda\hspace{-0.3mm})^n\bm{x}_{t+n+1}-(\hspace{-0.2mm}\gamma\lambda\hspace{-0.2mm})^H\hspace{-0.5mm}\bm{x}_{t+H}\hspace{-1mm}\bigg)^{\hspace{-1.4mm}T}\hspace{-0.5mm}\bigg]\label{eq:A_ex}\\
		&b=\Ex_{d^\phi,\mathcal{P},\pi}\hspace{-1mm}\bigg[\bm{x}_{t}\sum_{n=0}^{H-1}(\gamma\lambda)^n\bm{r}_{t+n}\bigg],\hspace{5mm}C=\Ex_{d^\phi}\left[\bm{x}_{t}\bm{x}_{t}^T\right]
		\end{align}
	\end{subequations}
	where, with a little abuse of notation, we defined $\bm{x}_t=\bm{x}_{\bm{s}_t}$ and $\bm{r}_t=\bm{r}^{\pi}(\bm{s}_t)$, where $\bm{s}_t$ is the state visited at time $t$.
\end{lemma}
\begin{proof}
	See Appendix \ref{app:expectations}.
\end{proof}
\subsection{Optimization problem}
Since the signal distributions are not known beforehand and we are working with a finite amount of data, say, of size $N$, we need to rely on empirical approximations to estimate the expectations in $\{A,b,C\}$. We thus let $\widehat{A}$, $\widehat{b}$, $\widehat{C}$ and $\widehat{U}$ denote estimates for $A$, $b$, $C$ and $U$ from data and replace them in \eqref{eq:PBE} to define the following empirical optimization problem:
\begin{align}\label{eq:single_empirical_problem}
\min_\theta J_{\text{emp}}(\theta)=\frac{1}{2}\bigl{\|}\widehat{A}\theta-\widehat{b}\bigr{\|}_{\widehat{C}^{-1}}^2+\frac{\eta}{2}\bigl{\|}\theta-\theta_{\text{p}}\bigr{\|}_{\widehat{U}}^2
\end{align}
Note that whether an empirical estimate for $U$ is required depends on the choice for $U$. For instance, if $U=I$ then obviously no estimate is needed. However, if $U=C$ then an empirical estimate is needed, (i.e., $\widehat{U}=\widehat{C}$).

To fully characterize the empirical optimization problem, expressions for the empirical estimates still need to be provided. The following lemma provides the necessary estimates.
\begin{lemma}\label{lemma:estimates}
	For the general off-policy case, the following expressions provide unbiased estimates \textcolor{black}{for $A$, $b$ and $C$:}
	\begin{subequations}\label{eq:AbC}
		\begin{align}
		&\widehat{A}_{n}=x_{n}\bigg(\hspace{-1mm}\rho_{n,0}^{H}x_{n}-\gamma(1-\lambda)\hspace{-1mm}\sum_{h=0}^{H-1}\hspace{-0.5mm}(\gamma\lambda)^h\xi_{n,n+h+1}x_{n+h+1}\nonumber\\
		&\hspace{5mm}-(\gamma\lambda)^H\xi_{n,n+H}x_{n+H}\hspace{-1mm}\bigg)^T,\hspace{5mm}\widehat{A}=\frac{1}{N-H}\hspace{-1.5mm}\sum_{n=1}^{N-H}\hspace{-1mm}\widehat{A}_{n}\label{eq:AbC_1}\\
		&\widehat{b}_{n}=x_{n}\sum_{h=0}^{H-1}(\gamma\lambda)^h\rho_{n,h}^{H}r_{n+h},\hspace{0.5cm}\widehat{b}=\frac{1}{N-H}\sum_{n=1}^{N-H}\widehat{b}_{n}\label{eq:AbC_2}\\
		&\widehat{C}_n=x_nx_n^T,\hspace{5mm}\widehat{C}=\frac{1}{N-H}\sum_{n=1}^{N-H}\widehat{C}_n\label{eq:AbC_3}
		\end{align}
	\end{subequations}
	where
	\begin{align}
	&\rho_{t,n}^H=(1-\lambda)\sum_{h=n}^{H-1}\lambda^{h-n}\xi_{t,t+h+1}+\lambda^{H-n}\xi_{t,t+H}\\
	&\xi_{t,t+h}=\prod_{j=t}^{t+h-1}\hspace{-2mm}\pi(a_j|s_j)/\phi(a_j|s_j)
	\end{align}
\end{lemma}
\begin{proof}
	See Appendix \ref{app:empirical}.
\end{proof}
Note that $\xi_{t,t+h}$ is the importance sample weight corresponding to the trajectory that started at some state $s_t$ and took $h$ steps before arriving at some other state $s_{t+h}$. Note that even if we have $N$ transitions, we can only use $N-H$ training samples because every estimate of $\widehat{x}_n$ and $\widehat{b}_n$ looks $H$ steps into the future.

\section{Distributed Policy Evaluation}
\label{Sec: Distributed}
In this section we present the distributed framework and use \eqref{eq:single_empirical_problem} to derive \textit{Fast Diffusion for Policy Evaluation} (FDPE). The purpose of this algorithm is to deal with situations where data is dispersed among a number of nodes and the goal is to solve the policy evaluation problem in a fully decentralized manner.

\subsection{Distributed Setting}
We consider a situation in which there are $K$ agents that wish to evaluate a target policy $\pi(a|s)$ for a common MDP. Each agent has $N$ samples, which are collected following its own behavior policy $\phi_k$ (with steady state distribution matrix $D^{\phi_k}$). Note that the behavior policies can be potentially different from each other. The goal for all agents is to estimate the value function of the target policy $\pi(a|s)$ leveraging all the data from all other agents in a fully decentralized manner.

To do this, they form a network in which each agent can only communicate with other agents in its immediate neighborhood. The network is represented by a graph in which the nodes and edges represent the agents and communication links, respectively. The topology of the graph is defined by a combination matrix $L$ whose $kn$-th entry (i.e., $\ell_{kn}$) is a scalar with which agent $n$ scales information arriving from agent $k$. If agent $k$ is not in the neighborhood of agent $n$, then $\ell_{kn}=0$.
%A sample network is shown in Figure \ref{fig:sample_network}.
%\begin{figure}[!t]
%	\centering
%	\captionsetup{justification=centering}
%	\includegraphics[width=1.5in]{sample_graph.png}
%	\caption{Sample network.}
%	\label{fig:sample_network}
%\end{figure}
\begin{assumption}\label{assumption_L}
	We assume that the network is strongly connected. This implies that there is at least one path from any node to any other node and that at least one node has a self-loop (i.e. that at least one agent uses its own information). We further assume that the combination matrix $L$ is symmetric and doubly-stochastic.
\end{assumption}
\begin{remark}\label{remark:L}
	In view of the Perron-Frobenius Theorem, assumption \ref{assumption_L} implies that the matrix L can be diagonalized as $L=H\Lambda H^T$, where one element of $\Lambda$ is equal to 1 and its corresponding eigenvector is given by $\one/\sqrt{K}$ (where $\one$ is the all ones vector). The remaining eigenvalues of $L$ lie strictly inside the unit circle.
\end{remark}
A combination matrix satisfying assumption 2 can be generated using the Laplacian rule, the maximum-degree rule, or the Metropolis rule (see Table 14.1 in \cite{sayed2014adaptation}).

\subsection{Algorithm Derivation}

Mathematically, the goal for all agents is to minimize the following aggregate cost:
\begin{align} \label{eq:dist_PBE}
S_{\text{M}}(\theta)\hspace{-0.5mm}&=\hspace{-0.5mm}\sum_{k=1}^{K}\tau_k\bigg(\hspace{-0.5mm}\frac{1}{2}\Bigl{\|}\Pi\big(\Gamma_{\hspace{-0.5mm}2}r^\pi\hspace{-0.5mm}+\rho_1\Gamma_{\hspace{-0.5mm}1}X\theta\big)-X\theta\Bigr{\|}_{D^{\phi_k}}^2\hspace{-2mm}+\frac{\eta}{2}\bigl{\|}\theta-\theta_{\text{p}}\bigr{\|}_{U_k}^2\hspace{-0.5mm}\bigg)
\end{align}
where the purpose of the nonnegative coefficients $\tau_k$ is to scale the costs of the different agents; this is useful since the costs of agents whose behavior policy is closer to the target policy might be assigned higher weights. For \eqref{eq:dist_PBE}, we define the matrices $D$ and $U$ to be:
\begin{align}
D=\sum_{k=1}^{K}\tau_kD^{\phi_k}\hspace{1cm}U=\sum_{k=1}^{K}\tau_kU_k
\end{align}
so that equation \eqref{eq:dist_PBE} becomes:
\begin{align} \label{eq:dist_PBE2}
S_{\text{M}}(\theta)&=\frac{1}{2}\Bigl{\|}\Pi\big(\Gamma_2r^\pi+\rho_1\Gamma_1X\theta\big)-X\theta\Bigr{\|}_{D}^2+\frac{\eta}{2}\bigl{\|}\theta-\theta_{\text{p}}\bigr{\|}_{U}^2
\end{align}
Note that \eqref{eq:dist_PBE2} has the same form as \eqref{eq:PBE}; the only difference is that in \eqref{eq:dist_PBE2} the matrices $D$ and $U$ are defined by linear combinations of the individual matrices $D^{\phi_k}$ and $U_k$, respectively. Matrices $D^{\phi_k}$ are therefore not required to be positive definite, only $D$ is required to be a positive definite diagonal matrix. Since the matrices $D^{\phi_k}$ are given by the steady-state probabilities of the behavior policies, this implies that each agent does not need to explore the entire state-space by itself, but rather all agents collectively need to explore the state-space. This is one of the advantages of our multi-agent setting. In practice, this could be useful since the agents can divide the entire state-space into sections, each of which can be explored by a different agent in parallel.
\begin{assumption}\label{assumption:behavior_policies}
	We assume that the behavior policies are such that the aggregated steady state probabilities \big(i.e., $\sum_{k=1}^{K}\tau_kD^{\phi_k}$\big) are strictly positive for every state.
\end{assumption}
The empirical problem for the multi-agent case is then given by:
\begin{align} \label{eq:em_PBE2}
&\min_\theta\hspace{1mm} J_{\text{emp}}(\theta)=\min_\theta\hspace{1mm}\frac{1}{2}\bigl{\|}\widehat{A}\theta-\widehat{b}\bigr{\|}_{\widehat{C}^{-1}}^2+\frac{\eta}{2}\bigl{\|}\theta-\theta_{\text{p}}\bigr{\|}_{\widehat{U}}^2
\end{align}
\begin{subequations}\label{eq:MARL_AbC_2}
	\begin{align}
	&\widehat{A}_k=\hspace{-1mm}\sum_{n=1}^{N\hspace{-0.5mm}-\hspace{-0.5mm}H}\frac{\widehat{A}_{k,n}}{N-H},\hspace{2mm}\widehat{b}_k=\hspace{-1mm}\sum_{n=1}^{N\hspace{-0.5mm}-\hspace{-0.5mm}H}\frac{\widehat{b}_{k,n}}{N-H},\hspace{2mm}\widehat{C}_k=\hspace{-1mm}\sum_{n=1}^{N\hspace{-0.5mm}-\hspace{-0.5mm}H}\frac{\widehat{C}_{k,n}}{N-H}\\
	&\widehat{A}=\sum_{k=1}^{K}\tau_k\widehat{A}_k,\hspace{2mm}\widehat{b}=\sum_{k=1}^{K}\tau_k\widehat{b}_k,\hspace{2mm}\widehat{C}=\sum_{k=1}^{K}\tau_k\widehat{C}_k
	\end{align}
\end{subequations}
\textcolor{black}{\begin{assumption}\label{assumption:data}
	We assume that $\widehat{C}$ and $\widehat{A}$ are positive definite and invertible, respectively.
\end{assumption}
It is easy to show that Assumption \ref{assumption:data} is equivalent to assuming that each state has been visited at least once while collecting data. Intuitively, this assumption is necessary for any policy evaluation algorithm since one cannot expect to estimate the value function of states that have never been visited.}
Since we are interested in deriving a distributed algorithm we define local copies $\{\theta_k\}$ and rewrite \eqref{eq:em_PBE2} equivalently in the form:
\begin{align} \label{eq:em_PBE3}
&\min_\theta\frac{1}{2}\biggl{\|}\sum_{k=1}^{K}\tau_k\big(\widehat{A}_k\theta_k-\widehat{b}_k\big)\biggr{\|}_{\left(\sum_{k=1}^{K}\tau_k\widehat{C}_k\right)^{-1}}^2\hspace{-0.7mm}+\sum_{k=1}^{K}\tau_k\frac{\eta}{2}\bigl{\|}\theta_k-\theta_{\text{p}}\bigr{\|}_{\widehat{U}_k}^2\nonumber\\
&\text{s.t}\hspace{1cm}\theta_1=\theta_2=\cdots=\theta_K
\end{align}
The above formulation although correct is not useful because the gradient with respect to any individual $\theta_k$ depends on all the data from all agents and we want to derive an algorithm that only relies on local data. To circumvent this inconvenience, we reformulate \eqref{eq:em_PBE2} into an equivalent problem. To this end, we note that every quadratic function can be expressed in terms of its conjugate function as:
\begin{align}
	\frac{1}{2}\|A\theta-b\|_{C^{-1}}^2=\max_{\omega}\left(-(A\theta-b)^T\omega-\frac{1}{2}\|\omega\|_{C}^{2}\right)
\end{align}
Therefore, expression \eqref{eq:em_PBE2} can equivalently be rewritten as:
\begin{align}\label{eq:multi_L1}
&\min_{\theta}\hspace{0.5mm}\max_{\omega}\sum_{k=1}^{K}\tau_k\hspace{-0.1mm}\bigg(\hspace{-0.2mm}\frac{\eta}{2}\|\theta-\theta_{\text{p}}\|_{\widehat{U}_{k}}^2\hspace{-0.8mm}-\omega^T(\widehat{A}_{k}\theta-\widehat{b}_{k})\hspace{-0.2mm}-\hspace{-0.2mm}\frac{1}{2}\|\omega\|_{\widehat{C}_{k}}^2\hspace{-0.2mm}\bigg)
\end{align}
\begin{remark}
	The saddle-point of \eqref{eq:multi_L1} is given by
	\begin{align}\label{eq:saddle_point}
	\ba{c}\widehat{\theta}^o\\\widehat{\omega}^o\hspace{-1mm}\ea=\Bigg[\hspace{-2mm}\begin{array}{c}\hspace{-0.7mm}\left(\hspace{-0.7mm}\widehat{A}^T\widehat{C}^{-1}\hspace{-0.4mm}\widehat{A}\hspace{-0.3mm}+\hspace{-0.3mm}\eta\widehat{U}\right)^{\hspace{-0.9mm}-\hspace{-0.5mm}1}\hspace{-1mm}\left(\eta\widehat{U}\theta_{\text{p}}\hspace{-0.5mm}+\hspace{-0.5mm}\widehat{A}^T\widehat{C}^{-1}\widehat{b}\right)\\
	\widehat{C}^{-1}\widehat{b}-\widehat{C}^{-1}\widehat{A}\hspace{0.5mm}\widehat{\theta}^o
	\end{array}\hspace{-2mm}\Bigg]
	\end{align}
\end{remark}
\begin{proof}
	$\widehat{\theta}^o$ and $\widehat{\omega}^o$ are obtained by equating the gradient of \eqref{eq:multi_L1} to zero and solving for $\theta$ and $\omega$.
\end{proof}
Defining local copies for the primal and dual variables we can write:
\begin{align}\label{eq:multi_L}
&\min_{\theta}\hspace{0.5mm}\max_{\omega}\hspace{1.5mm}\sum_{k=1}^{K}\tau_k\left(\frac{\eta}{2}\|\theta_k-\theta_{\text{p}}\|_{\widehat{U}_{k}}^2\hspace{-1mm}-\omega_k^T(\widehat{A}_{k}\theta_k-\widehat{b}_{k})\hspace{-0.4mm}-\hspace{-0.4mm}\frac{1}{2}\|\omega_k\|_{\widehat{C}_{k}}^2\right)\nonumber\\
&\text{s.t}\hspace{1cm}\theta_1=\theta_2=\cdots=\theta_K\hspace{1cm}\omega_1=\omega_2=\cdots=\omega_K
\end{align}
\textcolor{black}{Now to derive a learning algorithm we rewrite \eqref{eq:multi_L} in an equivalent more convenient manner (the following steps can be seen as an extension to the primal-dual case of similar steps used in \cite{yuan2017exact1}). We start by defining the following network-wide magnitudes:
\begin{align}\label{eq:defs}
	&\check{\theta}=\textrm{col}\{\theta_k\}_{k=1}^K,\hspace{3mm}\check{\omega}=\textrm{col}\{\omega_k\}_{k=1}^K, \hspace{3mm}\check{b}=\textrm{col}\{\tau_k\widehat{b}_k\}_{k=1}^K\nonumber\\
	&\check{A}=\textrm{diag}\{\tau_k\widehat{A}_k\}_{k=1}^K,\hspace{3mm}\check{C}=\textrm{diag}\{\tau_k\widehat{C}_k\}_{k=1}^K,\hspace{3mm}\check{\theta}_{\text{p}}=\one\otimes\theta_{\text{p}}\nonumber\\
	&\check{L}=L\otimes I_M,\hspace{3mm}V=H(I_K\hspace{-0.5mm}-\hspace{-0.5mm}\Lambda)^{1/2}H^T\hspace{-1mm}/\sqrt{2},\hspace{3mm}\check{V}=V\otimes I_M
\end{align}
We remind the reader that $H$ and $\Lambda$ were defined in Remark \ref{remark:L}. We further clarify that $(I_K-\Lambda)^{\frac{1}{2}}$ is the entrywise square root of the positive definite diagonal matrix $I_K-\Lambda$. The notation $\textrm{col}\{y\}_{k=1}^K$ refers to stacking vectors $y_k$ from $1$ to $K$ into one larger vector. Moreover, $\textrm{diag}\{Y_k\}_{k=1}^K$ is a block diagonal matrix with matrices $Y_k$ as its diagonal elements.
\begin{remark}\label{remark:constraints}
	Due to Remark \ref{remark:L}, it follows that the bases of the null-spaces of $V$ and $\check{V}$ are given by $\{\one\}$ and $\{\one\otimes I_M\}$, respectively. Therefore, we get:
	\begin{subequations}\label{eq:constraints}
		\begin{align}
		&\theta_1\hspace{-1mm}=\theta_2\hspace{-0.5mm}=\hspace{-0.5mm}\cdots\hspace{-0.5mm}=\hspace{-0.5mm}\theta_K\hspace{-1mm}\iff\hspace{-1mm}\check{V}\check{\theta}\hspace{-0.5mm}=\hspace{-0.5mm}0\\
		&\omega_1\hspace{-1mm}=\omega_2\hspace{-0.5mm}=\hspace{-0.5mm}\cdots\hspace{-0.5mm}=\hspace{-0.5mm}\omega_K\hspace{-1mm}\iff\hspace{-1mm}\check{V}\check{\omega}\hspace{-0.5mm}=\hspace{-0.5mm}0
		\end{align}
	\end{subequations}
\end{remark}
Using \eqref{eq:constraints} we transform \eqref{eq:multi_L} into the following equivalent formulation:
\begin{align}\label{eq:multi_L2}
&\min_{\theta}\hspace{0.5mm}\max_{\omega}\hspace{1.5mm}\underbrace{\frac{\eta}{2}\|\check{\theta}-\check{\theta}_{\text{p}}\|_{\check{U}}^2-\check{\omega}^T(\check{A}\check{\theta}-\check{b})-\frac{1}{2}\|\omega\|_{\check{C}}^2}_{=F(\check{\theta},\check{\omega})}\nonumber\\
&\text{s.t}\hspace{1cm}\check{V}\check{\theta}=0\hspace{1cm}\check{V}\check{\omega}=0
\end{align}
We next introduce the constraints into the cost by using Lagrangian and extended Lagrangian terms as follows:
\begin{align}\label{eq:multi_L3}
\min_{\check{\theta},y^\omega}\hspace{0.5mm}\max_{\check{\omega},y^\theta}\hspace{1.5mm}&F(\check{\theta},\check{\omega})+{y^\theta}^T\check{V}\check{\theta}-{y^\omega}^T\check{V}\check{\omega}+\frac{\|\check{V}\check{\theta}\|^2}{2}-\frac{\|\check{V}\check{\omega}\|^2}{2}
\end{align}
where $y^\omega$ and $y^\theta$ are the dual variables of $\check{\omega}$ and $\check{\theta}$, respectively. Now we perform incremental gradient ascent on $\check{\omega}$ and gradient descent on $y^\omega$ to obtain the following updates:
\begin{subequations}\label{eq:in_ascent}
	\begin{align}
	\psi_{i+1}^\omega&=\check{\omega}_{i}+\mu_{\omega}\nabla_{\check{\omega}}F(\check{\theta}_i,\check{\omega}_i)\\
	\phi_{i+1}^\omega&=\psi_{i+1}^\omega-\mu_{\omega,2}\check{V}^2\psi_{i+1}^\omega\stackrel{\mu_{\omega,2}=1}{=}(I+\check{L})/2\psi_{i+1}^\omega\label{eq:upd_2}\\
	\check{\omega}_{i+1}&=\phi_{i+1}^\omega-\mu_{\omega,3}\check{V}y^\omega\stackrel{\mu_{\omega,3}=1}{=}\phi_{i+1}^\omega-\check{V}y_i^\omega\label{eq:upd_3}\\
	y_{i+1}^\omega&=y_i^\omega+\mu_{\omega,4}\check{V}\check{\omega}_{i+1}\stackrel{\mu_{\omega,4}=1}{=}y_i^\omega+\check{V}\check{\omega}_{i+1}
	\end{align}
\end{subequations}
where in \eqref{eq:upd_2} we used $\check{V}^2=I-\check{L}$. Combining \eqref{eq:upd_2} and \eqref{eq:upd_3} we get:
\begin{subequations}\label{eq:1_order_form}
	\begin{align}
	\psi_{i+1}^\omega&=\check{\omega}_{i}+\mu_{\omega}\nabla_{\check{\omega}}F(\check{\theta}_i,\check{\omega}_i)\\
	\check{\omega}_{i+1}&=(I+\check{L})\psi_{i+1}^\omega/2-\check{V}y_i^\omega\label{eq:step1}\\
	y_{i+1}^\omega&=y_{i}^\omega+\check{V}\check{\omega}_{i+1}\label{eq:step2}
	\end{align}
\end{subequations}
Using \eqref{eq:step1} to calculate $\check{\omega}_{i+1}-\check{\omega}_{i}$ we get:
\begin{align}\label{eq:in_ascent_3}
\check{\omega}_{i+1}-\check{\omega}_{i}&=(I+\check{L})\left(\psi_{i+1}^\omega-\psi_{i}^\omega\right)/2-\check{V}\left(y_i^\omega-y_{i-1}^\omega\right)
\end{align}
Substituting \eqref{eq:step2} into \eqref{eq:in_ascent_3} we get:
\begin{align}\label{eq:in_ascent_4}
\check{\omega}_{i+1}&=(I+\check{L})\left(\psi_{i+1}^\omega+\check{\omega}_{i}-\psi_{i}^\omega\right)/2
\end{align}
Which we rewrite as:
\begin{subequations}\label{eq:in_ascent_5}
	\begin{align}
	\psi_{i+1}^\omega&=\check{\omega}_{i}+\mu_{\omega}\nabla_{\check{\omega}}F(\check{\theta}_i,\check{\omega}_i)\\
	\phi_{i+1}^\omega&=\psi_{i+1}^\omega+\check{\omega}_{i}-\psi_{i}^\omega\\
	\check{\omega}_{i+1}&=(I+\check{L})\phi_{i+1}^\omega/2
	\end{align}
\end{subequations}
Notice that steps \eqref{eq:in_ascent}-\eqref{eq:in_ascent_5} allow us to get rid of $y_{i}^\omega$. Performing incremental gradient descent on $\check{\theta}$ and gradient ascent on $y^\theta$ and following equivalently \eqref{eq:in_ascent}-\eqref{eq:in_ascent_5} we get:
\begin{subequations}\label{eq:in_ascent_theta}
	\begin{align}
	\psi_{i+1}^\theta&=\check{\theta}_{i}-\mu_{\theta}\nabla_{\check{\theta}}F(\check{\theta}_i,\check{\omega}_i)\\
	\phi_{i+1}^\theta&=\psi_{i+1}^\theta+\check{\theta}_{i}-\psi_{i}^\theta\\
	\check{\theta}_{i+1}&=(I+\check{L})\phi_{i+1}^\theta/2
	\end{align}
\end{subequations}
Combining \eqref{eq:in_ascent_5} and \eqref{eq:in_ascent_theta} and defining $\psi=\text{col}\{\psi^\omega,\psi^\theta\}$ (and similarly for $\phi$) we arrive at Algorithm 1, which is a fully distributed algorithm.}
\textcolor{black}{\begin{theorem}\label{theorem_1_linear_convergence}
	If Assumption \ref{assumption:data} is satisfied and the step-sizes $\mu_\omega$ and $\mu_\theta$ are small enough while satisfying the following inequality:
	\begin{align}\label{eq:condition}
	\frac{\mu_\omega}{\mu_\theta}>\eta\frac{\lambda_{\text{max}}(\widehat{U})}{\lambda_{\text{max}}(\widehat{C})}+2\sqrt{\frac{\mu_\omega}{\mu_\theta}\frac{\lambda_{\text{max}}(\widehat{A}\widehat{C}^{-1}\widehat{A}^T)}{\lambda_{\text{max}}(\widehat{C})}}
	\end{align}
	then the iterates $\theta_{k,i}$ and $\omega_{k,i}$ generated by Algorithm 1 converge linearly to \eqref{eq:saddle_point}.
\end{theorem}
Note that the above condition can always be satisfied by making $\mu_\omega/\mu_\theta$ sufficiently large.
\begin{algorithm}[!t]\label{alg:algorithm1}
	\textcolor{black}{\small\textbf{Algorithm 1: Processing steps at node $k$}\vspace{-3mm}\\
		\rule{8.65cm}{0.5pt}\\
		\textbf{Initialize}: $\theta_{k,0}$ and $\omega_{k,0}$ arbitrarily and let $\psi_{k,0}=[{\theta_{k,0}}^T,{\omega_{k,0}}^T]^T$.\\
		\textbf{For} $i=0,1,2\ldots$:
		\begin{subequations}\label{eq:updates_1}
			\begin{align}
			\hspace{-6mm}&\psi_{k,i+1}=\ba{c}
			\omega_{k,i}+\tau_k\mu_\omega\big(\widehat{b}_k-\widehat{A}_k\theta_{k,i}-\widehat{C}_k\omega_{k,i}\big)\\
			\theta_{k,i}-\tau_k\mu_\theta\big(\eta \widehat{U}_k(\theta_{k,i}-\theta_{\text{p}})-\widehat{A}_k\omega_{k,i}\big)
			\ea\\[-2.5pt]
			&\phi_{k,i+1}=\psi_{k,i+1}+\ba{c}
			\omega_{k,i}\\
			\theta_{k,i}\ea-\psi_{k,i}\\[-2.5pt]
			&\ba{c}
			\omega_{k,i+1}\\
			\theta_{k,i+1}
			\ea=\Big(\phi_{k,i+1}+\sum_{n\in\mathcal{N}_k}l_{nk}\phi_{n,i+1}\Big)/2
			\end{align}\vspace{-2mm}
	\end{subequations}}
\end{algorithm}
\begin{proof}
	We start by setting $\mu=\mu_\theta$ and introducing the following definitions:
	\begin{subequations}
		\begin{align}
		&\zeta_{k,i}\define\ba{c}
		\theta_{k,i}\\
		\frac{1}{\sqrt{\upsilon}}\omega_{k,i}
		\ea,\hspace{5mm}\zeta_{i}\define\text{col}\{\zeta_{k,i}\}_{k=1}^K,\hspace{5mm}\upsilon\define\frac{\mu_\omega}{\mu_\theta}\label{eq:z}\\
		&G_{k}\define\tau_k\ba{cc}
		\eta\widehat{U}_{k} & -\sqrt{\upsilon}\widehat{A}_{k}^T\\
		\sqrt{\upsilon}\widehat{A}_{k} & \upsilon\widehat{C}_{k}\ea\hspace{5mm}\mathcal{G}\define\text{diag}\{G_k\}_{k=1}^K\\
		&G\define\sum_{k=1}^{K}G_k,\hspace{5mm}p_{k}\define \tau_k\bigg[\hspace{-1mm}\begin{array}{c}
		\eta\widehat{U}_{k}\theta_{\text{p}}\\
		\sqrt{\upsilon}b_{k}\end{array}\hspace{-1mm}\bigg],\hspace{5mm}p\define\text{col}\{p_k\}_{k=1}^K\\
		&\bar{L}\define(L+I_K)/2,\hspace{5mm}\bar{\mathcal{L}}\define\bar{L}\otimes I_{2M},\hspace{5mm}\mathcal{V}\define V\otimes I_{2M}
		\end{align}
	\end{subequations}
	With these definitions, we write the update equations of Algorithm 1	in the form of equations \eqref{eq:1_order_form} (for both the primal and dual) in the following first order network-wide recursion:
	\begin{subequations}\label{eq:first_order_network_recursion}
		\begin{align}
		\zeta_{i+1}&=\bar{\mathcal{L}}\left(\bm\zeta_{i}-\mu\left(\mathcal{G}\bm\zeta_{i}-p\right)\right)-\mathcal{V}\mathcal{Y}_i\\
		\mathcal{Y}_{i+1}&=\mathcal{Y}_i+\mathcal{V}\bm\zeta_{i+1}
		\end{align}
	\end{subequations}
	for which $\mathcal{Y}_0=0$. Note that the variable $\mathcal{Y}_{i}$ is a network-wide variable that includes both $y^\omega$ and $y^\theta$.
	\begin{lemma}\label{lemma:optimality}
		Recursion \eqref{eq:first_order_network_recursion} has a unique fixed point $(\zeta^o,\mathcal{Y}^o)$, where $\mathcal{Y}^o$ lies in the range space of $\mathcal{V}$. This fixed point satisfies the following conditions:
		\begin{subequations}\label{eq:fixed_point}
			\begin{align}
			\mu\bar{\mathcal{L}}\left(\mathcal{G}\zeta^o-p\right)+\mathcal{V}\mathcal{Y}^o&=0\label{eq:fixed_point_1}\\
			\mathcal{V}\zeta^o&=0\label{eq:fixed_point_2}
			\end{align}
		\end{subequations}
		It further holds that $\zeta^o=\one_K\otimes[\widehat{\theta}^{oT},\widehat{\omega}^{oT}]^T$, where $\widehat{\theta}^o$ and $\widehat{\omega}^o$ are given by \eqref{eq:saddle_point}.\hfill\qed
	\end{lemma}
	\begin{proof}
		See appendix \ref{app:optimality}.
	\end{proof}
	Subtracting $\zeta^o$ and $\mathcal{Y}^o$ from \eqref{eq:first_order_network_recursion} and defining the error quantities $\widetilde{\zeta}_{i}=\zeta^o-\zeta_{i}$ and $\widetilde{\mathcal{Y}}_{i}=\mathcal{Y}^o-\mathcal{Y}_{i}$ we get:
	\begin{align}
	\bigg[\hspace{-2mm}\begin{array}{cc}
	I & 0\\
	-\mathcal{V} & I\end{array}\bigg]\Bigg[\hspace{-2mm}\begin{array}{c}
	\widetilde{\zeta}_{i+1}\\
	\widetilde{\mathcal{Y}}_{i+1}
	\end{array}\hspace{-2mm}\Bigg]&\hspace{-0.5mm}=\hspace{-0.5mm}
	\bigg[\hspace{-2mm}\begin{array}{cc}
	\bar{\mathcal{L}}\left(I-\mu\mathcal{G}\right) & -\mathcal{V}\\
	0 & I\end{array}\hspace{-2mm}\bigg]
	\Bigg[\hspace{-2mm}\begin{array}{c}
	\widetilde{\zeta}_{i}\\
	\widetilde{\mathcal{Y}}_{i}
	\end{array}\hspace{-2mm}\Bigg]
	\end{align}
	Multiplying by the inverse of the leftmost matrix we get:
	\begin{align}\label{eq:error_rec_1}
	\Bigg[\hspace{-2mm}\begin{array}{c}
	\widetilde{\zeta}_{i+1}\\
	\widetilde{\mathcal{Y}}_{i+1}
	\end{array}\hspace{-2mm}\Bigg]&\hspace{-0.5mm}=\hspace{-0.5mm}
	\bigg[\hspace{-2mm}\begin{array}{cc}
	\bar{\mathcal{L}}\left(I-\mu\mathcal{G}\right) & -\mathcal{V}\\
	\mathcal{V}\bar{\mathcal{L}}\left(I-\mu\mathcal{G}\right) & \bar{\mathcal{L}}\end{array}\hspace{-2mm}\bigg]
	\Bigg[\hspace{-2mm}\begin{array}{c}
	\widetilde{\zeta}_{i}\\
	\widetilde{\mathcal{Y}}_{i}
	\end{array}\hspace{-2mm}\Bigg]
	\end{align}
	\begin{lemma}\label{lemma:coord_transform}
		Through a coordinate transformation applied to \eqref{eq:error_rec_1} we obtain the following error recursion:
		\begin{align}\label{eq:diag_rec}
		&\bigg[\hspace{-2.3mm}\begin{array}{c}\bar{x}_{i+1}\\
		\hat{x}_{i+1}\end{array}\hspace{-2.5mm}\bigg]\hspace{-1.2mm}=\hspace{-1.2mm}\bigg[\hspace{-2.5mm}\begin{array}{cc}
		I_{2M}-\mu K^{-1}G & 
		-\mu K^{-\frac{1}{2}}\mathcal{I}^T\mathcal{G}\mathcal{H}_{u}\\
		-\frac{\mu}{\sqrt{K}}(\mathcal{H}_{l}^T\hspace{-0.7mm}+\hspace{-0.3mm}\mathcal{H}_{r}^T\mathcal{V})\bar{\mathcal{L}}\mathcal{G}\mathcal{I} & \mathcal{D}_1\hspace{-0.5mm}-\hspace{-0.5mm}\mu(\mathcal{H}_{l}^T\hspace{-0.5mm}+\hspace{-0.5mm}\mathcal{H}_{r}^T\mathcal{V})\bar{\mathcal{L}}\mathcal{G}\mathcal{H}_{u}\end{array}\hspace{-2.5mm}\bigg]\hspace{-1mm}
		\bigg[\hspace{-2.2mm}\begin{array}{c}\bar{x}_{i}\\
		\hat{x}_{i}\end{array}\hspace{-2.2mm}\bigg]
		\end{align}
		where $\mathcal{H}_{l},\mathcal{H}_{r},\mathcal{H}_{u},\mathcal{H}_{d}\in\mathds{R}^{2KM\times4KM-4M}$ are some constant matrices, $\bar{x}_{i}\in\mathbb{R}^{2M}$, $\hat{x}_{i}\in\mathbb{R}^{2M}$ and $\mathcal{D}_1$ is a diagonal matrix with $\|\mathcal{D}_1\|_2^2=\lambda_2(\bar{L})<1$. Furthermore $\bar{x}_{i}$ and $\hat{x}_{i}$ satisfy:
		\begin{align}\label{eq:def_x_bar}
			\|\widetilde{\zeta}_{i}\|^2&\leq \|\bar{x}_{i}\|^2+\|\mathcal{H}_{u}\|^2\|\hat{x}_{i}\|^2,\hspace{5mm}\|\widetilde{\mathcal{Y}}_{i}\|^2\leq\|\mathcal{H}_{d}\|^2\|\hat{x}_{i}\|^2
		\end{align}\hfill\qed
	\end{lemma}
	\begin{proof}
		See Appendix \ref{app:coord_transform}.
	\end{proof}
	Due to Theorem 2.1 from \cite{shen2008condition}, if Assumption \ref{assumption:data} and the following condition are satisfied:
	\begin{align}
		\upsilon\lambda_{\text{min}}(\widehat{C})>\eta\lambda_{\text{max}}(\widehat{U})+2\sqrt{\upsilon\lambda_{\text{min}}(\widehat{C})\lambda_{\text{max}}(\widehat{A}\widehat{C}^{-1}\widehat{A}^T)}
	\end{align}
	then matrix $G$ is diagonalizable with strictly positive eigenvalues. Hence, we can write $G=Z\Lambda_GZ^{-1}$. Therefore, defining $\check{x}_i=Z^{-1}\bar{x}_i$ we can transform \eqref{eq:diag_rec} into:
	\begin{align}\label{eq:diag_rec_2}
	&\bigg[\hspace{-2.3mm}\begin{array}{c}\check{x}_{i+1}\\
	\hat{x}_{i+1}\end{array}\hspace{-2.5mm}\bigg]\hspace{-1.2mm}=\hspace{-1.2mm}\bigg[\hspace{-2.5mm}\begin{array}{cc}
	I_{2M}-\mu K^{-1}\Lambda_G & 
	-\mu K^{-\frac{1}{2}}Z^{-1}\mathcal{I}^T\mathcal{G}\mathcal{H}_{u}\\
	\frac{-\mu}{\sqrt{K}}(\mathcal{H}_{l}^T\hspace{-0.7mm}+\hspace{-0.3mm}\mathcal{H}_{r}^T\mathcal{V})\bar{\mathcal{L}}\mathcal{G}\mathcal{I}Z & \mathcal{D}_1\hspace{-0.5mm}-\hspace{-0.5mm}\mu(\mathcal{H}_{l}^T\hspace{-0.5mm}+\hspace{-0.5mm}\mathcal{H}_{r}^T\mathcal{V})\bar{\mathcal{L}}\mathcal{G}\mathcal{H}_{u}\end{array}\hspace{-2.5mm}\bigg]\hspace{-1mm}
	\bigg[\hspace{-2.2mm}\begin{array}{c}\check{x}_{i}\\
	\hat{x}_{i}\end{array}\hspace{-2.2mm}\bigg]
	\end{align}
	\begin{lemma}\label{lemma:squared_recursion}
		If $\mu<K/\rho(\Lambda_G)$ then the following inequality holds:
		\begin{align}\label{eq:ms_rec}
		\underbrace{\bigg[\hspace{-2mm}\begin{array}{c}\|\check{x}_{i+1}\|^2\\
		\|\hat{x}_{i+1}\|^2\end{array}\hspace{-1.5mm}\bigg]}_{\define z_{i+1}}\hspace{-1mm}\preceq\hspace{-1mm}\underbrace{\bigg[\hspace{-1.5mm}\begin{array}{cc}
		\hspace{-0.5mm}\rho(I_{2M}\hspace{-0.5mm}-\hspace{-0.5mm}\mu K^{\hspace{-0.5mm}-\hspace{-0.5mm}1}\hspace{-0.5mm}\Lambda_G) & \hspace{-4mm}\mu a_2\\
		\mu^2a_3 & \hspace{-4mm}\sqrt{\lambda_2(\bar{L})}+\mu^2a_4\end{array}\hspace{-1.5mm}\bigg]}_{\define B(\mu)}\hspace{-1mm}
		\bigg[\hspace{-2mm}\begin{array}{c}\|\check{x}_{i}\|^2\\
		\|\hat{x}_{i}\|^2\end{array}\hspace{-1.5mm}\bigg]
		\end{align}
		where $a_2$, $a_3$ and $a_4$ are positive constants.\hfill\qed
	\end{lemma}
	\begin{proof}
		See Appendix \ref{app:squared_recursion}.
	\end{proof}
	Computing the $1-$norm on both sides of the above inequality and using the fact that $\|B(\mu)z_i\|_1\leq\|B(\mu)\|_1\|z_i\|_1$ we get:
	\begin{align}
		\|\check{x}_{i+1}\|^2+\|\hat{x}_{i+1}\|^2\leq\|B(\mu)\|_1(\|\check{x}_i\|^2+\|\hat{x}_i\|^2)
	\end{align}
	Iterating we get:
	\begin{align}
	&\|\check{x}_{i}\|^2+\|\hat{x}_{i}\|^2\leq\alpha^i(\|\check{x}_0\|^2+\|\hat{x}_0\|^2)\\
	&\alpha\hspace{-0.5mm}=\hspace{-0.5mm}\max\{\rho(I\hspace{-0.5mm}-\hspace{-0.5mm}\mu K^{-1}\Lambda_G)+\mu^2a_3,\lambda_2(\bar{L})^{\frac{1}{2}}+\mu a_2+\mu^2a_4\}
	\end{align}
	Recalling that $\check{x}_i=Z^{-1}\bar{x}_i$ and \eqref{eq:def_x_bar} we get:
	\begin{align}
		\|\widetilde{\zeta}_i\|^2\leq\max\{\|Z\|^2,\|\mathcal{H}_u\|^2\}\left(\|\check{x}_{i}\|^2+\|\hat{x}_{i}\|^2\right)\leq\alpha^i\kappa
	\end{align}
	where $\kappa=\max\{\|Z\|^2,\|\mathcal{H}_u\|^2\}\left(\|\check{x}_{0}\|^2+\|\hat{x}_{0}\|^2\right)$. Since $\alpha<1$ (for small enough $\mu$) we conclude that the iterates $\theta_{k,i}$ and $\omega_{k,i}$ generated by Algorithm 1 converge linearly to \eqref{eq:saddle_point} for every agent $k$; which completes the proof.
\end{proof}}
\textcolor{black}{Note that when the step-size is small enough, the convergence rate $\alpha$ depends on two factors: the spectrum of $\rho(I\hspace{-0.5mm}-\hspace{-0.5mm}\mu K^{-1}\Lambda_G)$ (which is also the convergence rate of a centralized gradient descent implementation) and the second biggest eigenvalue of the combination matrix. This implies that when the network is densely connected, the factor that determines the convergence rate of the algorithm is the eigenstructure of the saddle-point matrix $G$. On the contrary, when the network is sparsely connected, the rate at which the agents' information diffuses across the network is the factor that determines the convergence rate of the algorithm.}

\textcolor{black}{Algorithm 1 has the inconvenience that at every iteration each agent has to calculate the exact local gradient (i.e., all data samples have to be used), which computationally might be demanding for cases with big data. Therefore we add a variance reduced gradient strategy to Algorithm 1. More specifically, we use the AVRG \cite{AVRG} (amortized variance-reduced gradient) technique (Algorithm 2).}

\textcolor{black}{The AVRG strategy is a single agent algorithm designed to minimize functions of the form:
\begin{align}
	\min_{\theta}\sum_{n=1}^{N}Q_n(\theta)
\end{align}	
where $Q_n$ is some loss function evaluated at the $n-$th data point. The main difference between standard stochastic gradient descent and AVRG is that in AVRG at every epoch the estimated gradients are collected in a vector $g$, which is used in the following epoch to reduce the variance of the gradient estimates. In the listing corresponding to Algorithm 2 we introduced an epoch index $e$ and a uniform random permutation function $\boldsymbol{\sigma}^e$. The epoch index is due to the fact that AVRG relies on random reshuffling (which is why the permutation function is necessary) and sampling without replacement (hence, one epoch is one pass over each data point). The gradient estimates produced by the AVRG strategy are subject to both bias and variance, however both decay over time and therefore do not jeopardize convergence. The resulting algorithm from the combination of Algorithm 1 and AVRG (which we refer to as Fast Diffusion for Policy Evaluation or FDPE) relies on stochastic gradients (as opposed to exact gradient calculations) and as a consequence is more computationally efficient than Algorithm 1, while still retaining the convergence guarantees.}
\begin{algorithm}[!t]
	\small\textbf{\small Algorithm 2:} AVRG for $N$ data points and loss function $Q$\vspace{-2mm}\\
	\rule{8.65cm}{0.5pt}\\
	\textbf{Initialize}: $\theta_{0}^0$ arbitrarily; $g^0=0$; $\nabla Q_n(\theta_{0}^0)\leftarrow0$, $\;1\leq n \leq N$.\\
	\textbf{For} $e=0,1,2\ldots$:\vspace{-1.5mm}
	\setlength{\belowdisplayskip}{-2pt}
	\begin{subequations}
		\begin{align}
			&\text{Generate a random permutation function $\boldsymbol{\sigma}^e$ and set $g^{e+1}\hspace{-1mm}=0$}\nonumber\\
			&\textbf{For } i=0,1,\ldots,N-1:\nonumber\\
			&\hspace{5mm}n=\bsigma^e(i)\\
			&\hspace{5mm}\theta_{i+1}^{e}=\theta_{k,i}^e-\mu\big(\nabla Q_n(\theta_{i}^e)-\nabla Q_n(\theta_{0}^e)+g^{e}\big)\\
			&\hspace{5mm}g^{e+1}\leftarrow g^{e+1}+\nabla Q_n(\theta_{i}^e)/N\\
			&\theta_{0}^{e+1}=\theta_{N}^e
		\end{align}
	\end{subequations}
\end{algorithm}
\begin{algorithm}[!t]\label{algorithm}
	\textbf{\small{Algorithm 3: \textit{Fast Diffusion for Policy Evaluation}} at node $k$}\vspace{-2mm}\\
	\rule{8.65cm}{0.5pt}\\
	\small{Distribute the $N-H$ data points into $J$ mini-batches of size $|\mathcal{J}_j|$; where $\mathcal{J}_j$ is the $j$-th mini-batch.\\
		\textbf{Initialize}:
		$\theta_{k,0}^0 \text{ and } \omega_{k,0}^0 \text{ arbitrarily}$; let $\psi_{k,0}^0=[{\theta_{k,0}^0}^T,{\omega_{k,0}^0}^T]^T$, $g_k^0=0$; $\beta_{k,n}(\theta_{k,0}^0,\omega_{k,0}^0)\leftarrow0$, $\;1\leq n \leq N-H$\\
		\textbf{For} $e=0,1,2\ldots$:\\
		\hspace*{3mm}Generate a random permutation function of the mini-batches $\boldsymbol{\sigma}_k^e$\\
		\hspace*{3mm}Set $g_k^{e+1}=0$\\
		\hspace*{3mm}\textbf{For} $i=0,1,\ldots,J-1$:\\
		\hspace*{6mm}Generate the local stochastic gradients:\vspace{-1.5mm}
		\setlength{\belowdisplayskip}{-2pt}
		\begin{subequations}\label{alg:stoc_grad}
			\begin{align}
			\hspace{6mm}&j=\bsigma_k^e(i)\\[-3.0pt]
			&{\beta}_{k}(\theta_{k,i}^e,\omega_{k,i}^e)\hspace{-0.3mm}=\hspace{-0.3mm}g_k^e\hspace{-0.5mm}+\hspace{-0.5mm}\frac{1}{|\mathcal{J}_j|}\hspace{-0.5mm}\sum_{l\in\mathcal{J}_j}\hspace{-1.3mm}\left(\beta_{k,l}(\theta_{k,i}^e,\omega_{k,i}^e)\hspace{-0.4mm}-\hspace{-0.4mm}\beta_{k,l}(\theta_{k,0}^e,\omega_{k,0}^e)\right)\label{eq:variance_reduction1}\\[-6.5pt]
			&g_{k}^{e+1}\leftarrow g_{k}^{e+1}+\frac{1}{N-H}\sum_{l\in\mathcal{J}_j}\beta_{k,l}(\theta_{k,i}^e,\omega_{k,i}^e)\label{eq:variance_reduction2}
			\end{align}
		\end{subequations}
		\hspace{6mm}Update $[\theta_{k,i+1}^e,\omega_{k,i+1}^e]^T$ with exact diffusion:
		\vspace{-1mm}
		\begin{subequations}\label{eq:updates}
			\begin{align}
			\hspace{6mm}&\psi_{k,i+1}^e=\ba{c}
			\theta_{k,i}^e\\
			\omega_{k,i}^e
			\ea-\tau_k\ba{cc}\mu_\theta&0\\0&\mu_\omega\ea{\beta}_{k}(\theta_{k,i}^e,\omega_{k,i}^e)\\[-2.5pt]
			&\phi_{k,i+1}^e=\psi_{k,i+1}^e+\ba{c}
			\theta_{k,i}^e\\
			\omega_{k,i}^e
			\ea-\psi_{k,i}^e\\[-2.5pt]
			&\ba{c}
			\theta_{k,i+1}^e\\
			\omega_{k,i+1}^e
			\ea=\bigg(\phi_{k,i+1}^e+\sum_{n\in\mathcal{N}_k}l_{nk}\phi_{n,i+1}^e\bigg)/2
			\end{align}
		\end{subequations}
		\begin{equation}
		\hspace{-53mm}\ba{c}
		\theta_{k,0}^{e+1}\\
		\omega_{k,0}^{e+1}
		\ea=\ba{c}
		\theta_{k,J}^e\\
		\omega_{k,J}^e
		\ea
		\end{equation}}
\end{algorithm}

In the listing of \textit{FDPE}, we introduce $\boldsymbol{\sigma}_k^e$, $\mathcal{J}_j$, and $\beta_{k,j}(\theta,\omega)$, where $\boldsymbol{\sigma}_k^e$ indicates a random permutation of the $J$ mini-batches of the $k$-th agent, which is generated at the beginning of epoch $e$; $\mathcal{J}_j$ is the $j$-th mini-batch and $\beta_{k,l}(\theta,\omega)$ is defined as follows:
\begin{align}\label{eq:gradient}
\beta_{k,l}(\theta,\omega)&=\bigg[\hspace{-2mm}\begin{array}{c}\nabla_\theta F_{k,l}(\theta,\omega)\\-\nabla_\omega F_{k,l}(\theta,\omega)\end{array}\hspace{-2mm}\bigg]=
\bigg[\hspace{-2mm}\begin{array}{c}\eta\widehat{U}_{k,l}(\hspace{-0.5mm}\theta\hspace{-0.2mm}-\hspace{-0.2mm}\theta_{\text{p}}\hspace{-0.5mm})-\widehat{A}_{k,l}^T\omega\\
\widehat{A}_{k,l}\theta-\widehat{b}_{k,l}+\widehat{C}_{k,l}\omega
\end{array}\hspace{-2mm}\bigg]
\end{align}
\textcolor{black}{We remark that Algorithm 1 is a special case of FDPE, which corresponds to the case where the mini-batch size is selected equal to the whole batch of data}. Note that the choice of the mini-batch size provides a communication-computation trade-off. As the number of mini-batches diminishes so do the communication requirements per epoch. However, more gradients need to be calculated per update and hence more gradient calculations might be required to achieve a desired error. Obviously the optimal amount of mini-batches $J$ to minimize the overall time of the optimization process depends on the particular hardware availability for each implementation. Note that the only difference between update equations \eqref{eq:updates_1} and \eqref{eq:updates} is that the updates which correspond to algorithm 1 use exact local gradients, while \eqref{eq:updates} use stochastic approximations obtained through equations \eqref{alg:stoc_grad}.\textcolor{black}{
\begin{theorem}\label{theorem_linear_convergence}
	If Assumption \ref{assumption:data} is satisfied and the step-sizes $\mu_\omega$ and $\mu_\theta$ are small enough while satisfying inequality \eqref{eq:condition}, then the iterates $\theta_{k,i}^e$ and $\omega_{k,i}^e$ generated by FDPE converge linearly to \eqref{eq:saddle_point}.
\end{theorem}
\begin{proof}
	The proof is demanding and lengthy, therefore due to length constraints we include it as supplementary material --- see though the arXiv version [??] for the detailed steps. We remark however that such proof is based on the proof of Theorem 1. The main difference is that for the FDPE algorithm, the error recursions have extra terms that account for the gradient noise.
\end{proof}
The main difference between Theorems 1 and 2 is that there are more constraints on the step sizes due to the gradient noise (i.e., smaller step sizes might be necessary).}

\section{Multi-Agent Reinforcement Learning}
\label{Sec:Marl}
In this section we derive a cost for the MARL case that has the same form as \eqref{eq:em_PBE2} and therefore shows that Algorithms 1 and 3 are also applicable for this scenario.

The network structure is the same as in the previous section. The difference with the previous section is that in the MARL case the agents interact with a unique environment and with each other, and have a common goal. Therefore, in this section we refer to the collection of all agents as a team. In this case the MDP is defined by the tuple ($\mathcal{S}$,$\mathcal{A}^k$,$\mathcal{P}$,$r^k$). $\mathcal{S}$ is a set of global states shared by all agents and $\mathcal{A}^k$ is the set of actions available to agent $k$ of size $A^k=|\mathcal{A}^k|$. We refer to $\mathcal{A}=\prod_{k=1}^{K}\mathcal{A}^k$ as the set of team actions. $\mathcal{P}(s'|s,a)$ is the defined as before but considering global states and team actions, and $r^k:\mathcal{S}\times\mathcal{A}\times\mathcal{S}\rightarrow\mathbb{R}$ is the reward function of agent $k$. Specifically, $r^k(s,a,s')$ is the expected reward of decision maker $k$ when the team transitions to state $s'\in\mathcal{S}$ from state $s\in\mathcal{S}$ having taken team action $a\in\mathcal{A}$. We clarify that we refer to the team's action (i.e., the collection of all individual actions) as $a$, while $a^k$ refers to the individual action of agent $k$. What distinguishes this model from the one in the previous section is that the transition probabilities and the reward functions of the individual agents depend not only on their own actions but on the actions of all other agents. The goal of all the agents is to maximize the aggregated return and hence in this case the value function is defined as:
\begin{align}\label{eq:marl_value_function_definition}
v^\pi(s)&=\sum_{t=i}^{\infty}\frac{\gamma^{t-i}}{K}\Bigg(\sum_{k=1}^{K}\Ex\left[\bm{r}^k(\bm{s}_{t},\bm{a}_t,\bm{s}_{t+1})|\bm{s}_i=s\right]\Bigg)
\end{align}
introducing a global reward as $\bm{r}(\bm{s}_{t},\bm{a}_t,\bm{s}_{t+1})=K^{-1}\sum_{k=1}^{K}\bm{r}^k(\bm{s}_{t},\bm{a}_t,\bm{s}_{t+1})$ equation \eqref{eq:marl_value_function_definition} becomes identical to \eqref{eq:value_function_definition}. \textcolor{black}{Therefore, with the understanding that in this case the states and the policies, $\pi(a|s)=\pi(a^1,\cdots,a^K|s)$ and $\phi(a^1,\cdots,a^K|s)$, are global (and hence also the feature vectors and sampling weights are global\footnote{\textcolor{black}{Note that in this case if sampling weights were to be used for off-policy operation, every agent would require knowledge of the teams behavior policy.}}), the rest of the derivation follows identically to Section \ref{sec: Problem Setting}. Therefore, the empirical problem becomes like \eqref{eq:single_empirical_problem} with the following estimates:
\begin{align}\label{eq:MARL_AbC_3}
	&\widehat{A}\hspace{-0.5mm}=\hspace{-2mm}\sum_{n=1}^{N-H}\hspace{-1.5mm}\frac{\widehat{A}_{n}}{N-H},\hspace{2mm}\widehat{b}\hspace{-0.5mm}=\hspace{-2mm}\sum_{n=1}^{N-H}\sum_{k=1}^{K}\frac{\widehat{b}_{k,n}}{K(N-H)},\hspace{2mm}\widehat{C}\hspace{-0.5mm}=\hspace{-2mm}\sum_{n=1}^{N-H}\hspace{-1mm}\frac{\widehat{C}_{n}}{N-H}
\end{align}
Defining $\widehat{A}_{k,n}=\widehat{A}_{n}$ we can write $\widehat{A}_{n}=\sum_{k=1}^{K}\widehat{A}_{k,n}/K$ (and similarly for $\widehat{C}_n$). Equations \eqref{eq:MARL_AbC_3} become exactly like \eqref{eq:MARL_AbC_2} with $\tau_K=K^{-1}$, and therefore both algorithms can be applied to MARL scenarios without changes.} 
\begin{figure*}[!t]
	\centering
	\hspace{5mm}\subfloat[MDP]{\includegraphics[width=1.7in]{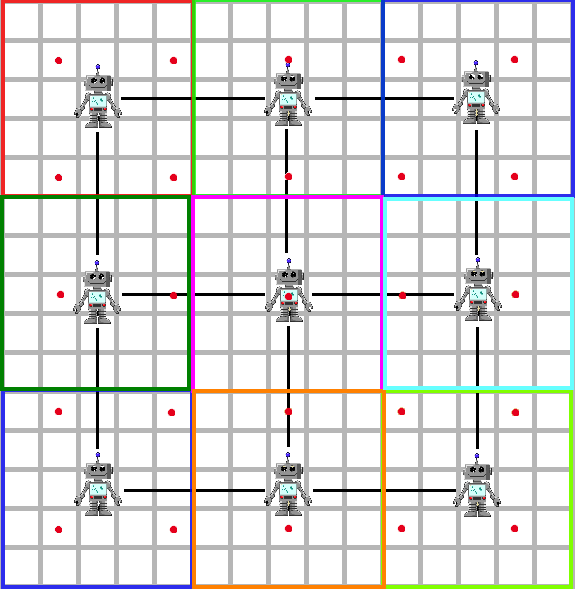}
		\label{fig:robots}}
	\hspace{3mm}\subfloat[Bias]{\includegraphics[width=2.5in]{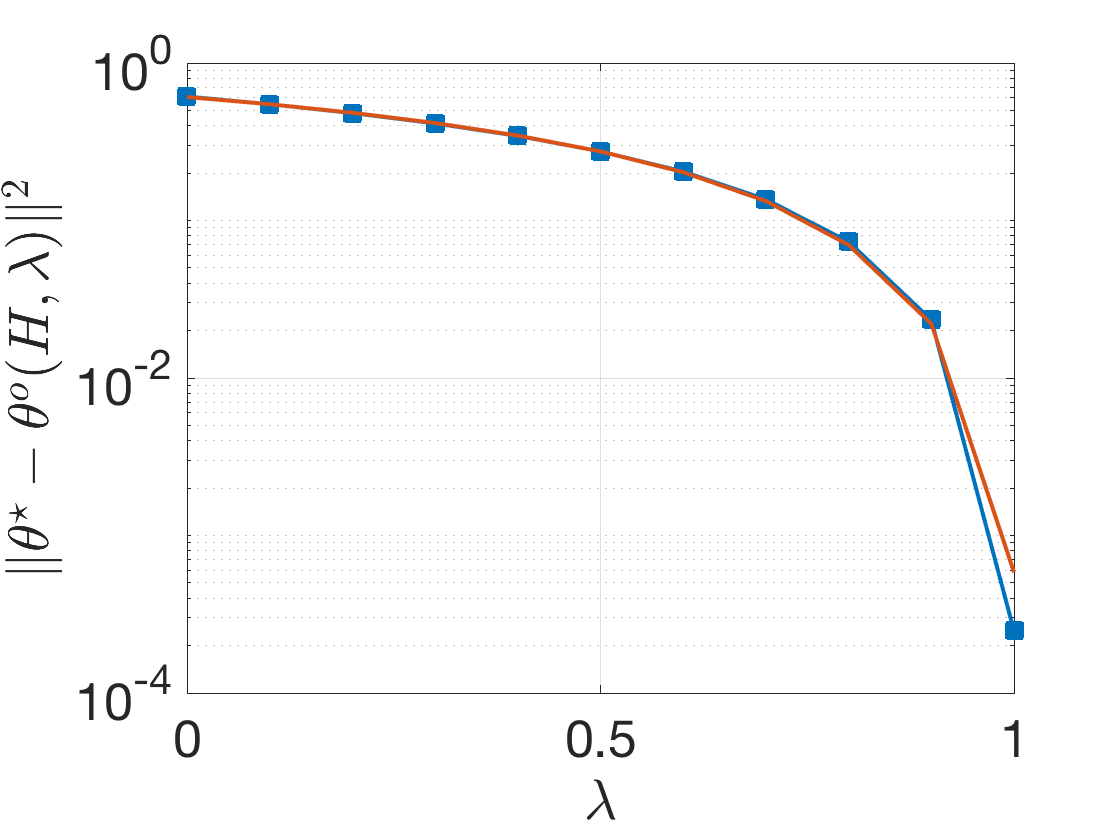}%
		\label{fig:exp1_bias}}
	\subfloat[Bias variance trade-off]{\includegraphics[width=2.5in]{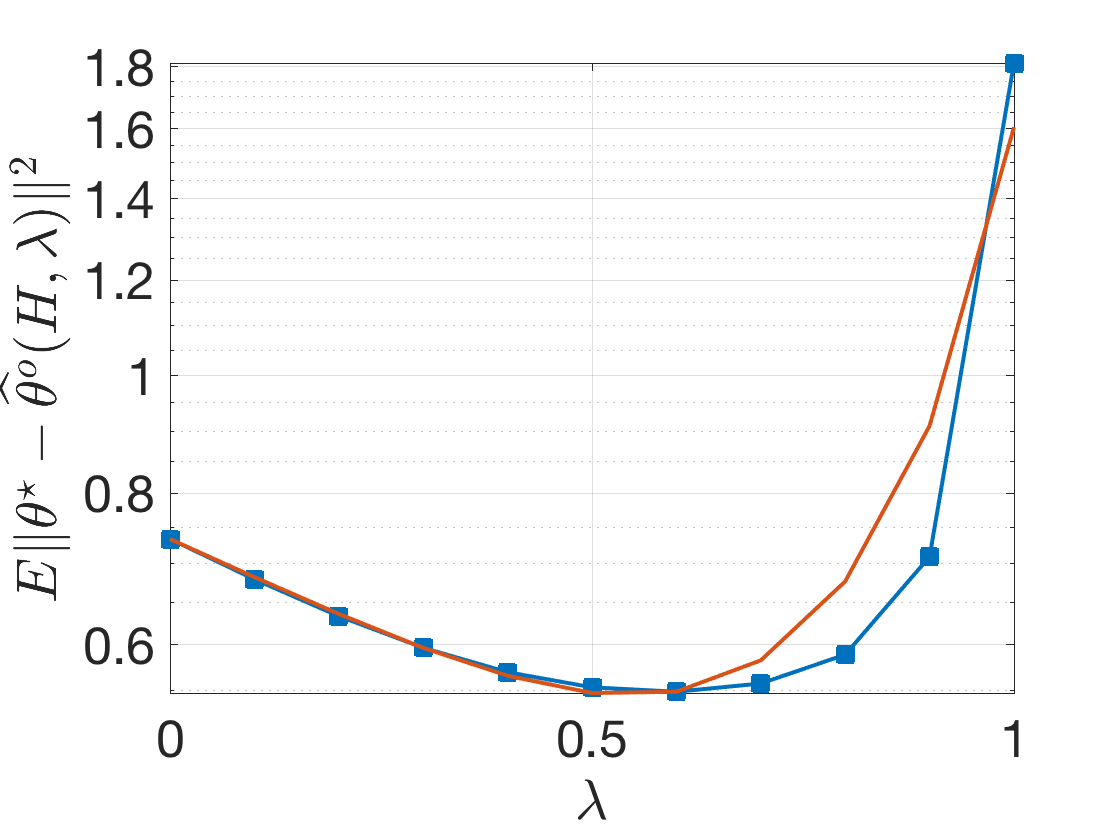}%
		\label{fig:exp1_bias_and variance}}
	\hfil
	\subfloat[Communication vs computation]{\includegraphics[width=2.4in]{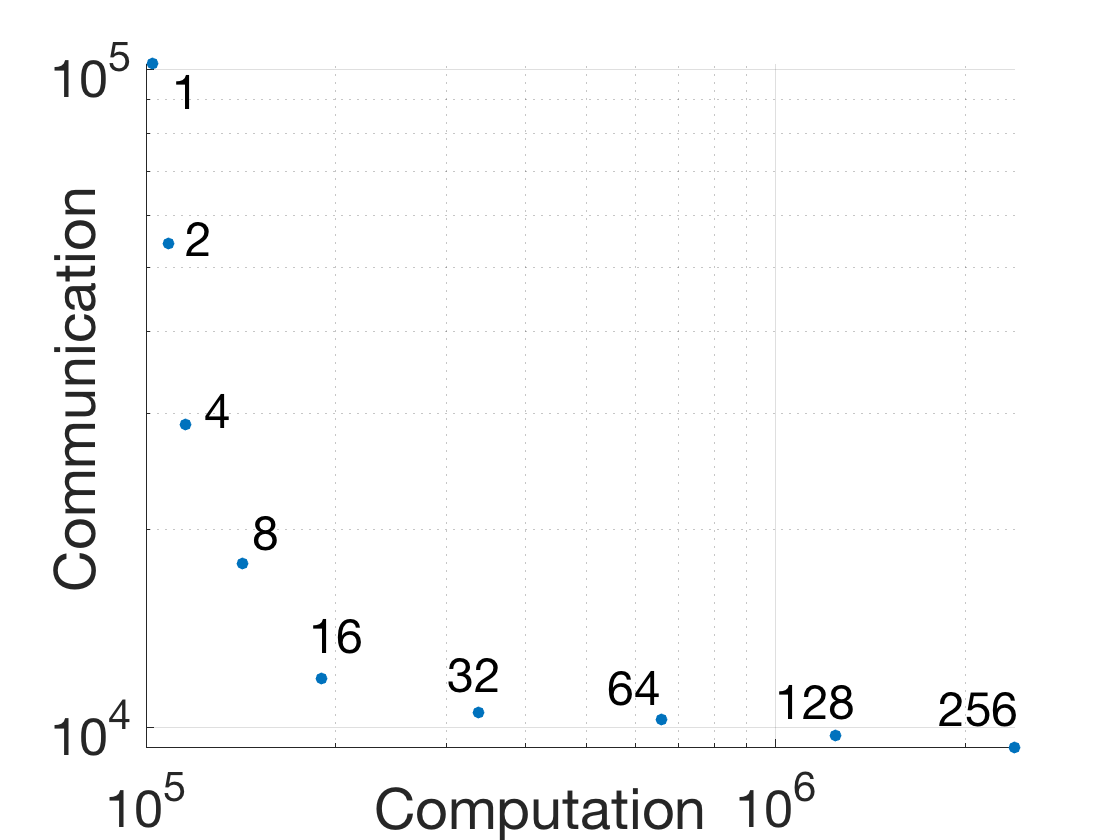}%
		\label{fig:exp1_comp_comm}}
	\subfloat[Empirical error]{\includegraphics[width=2.4in]{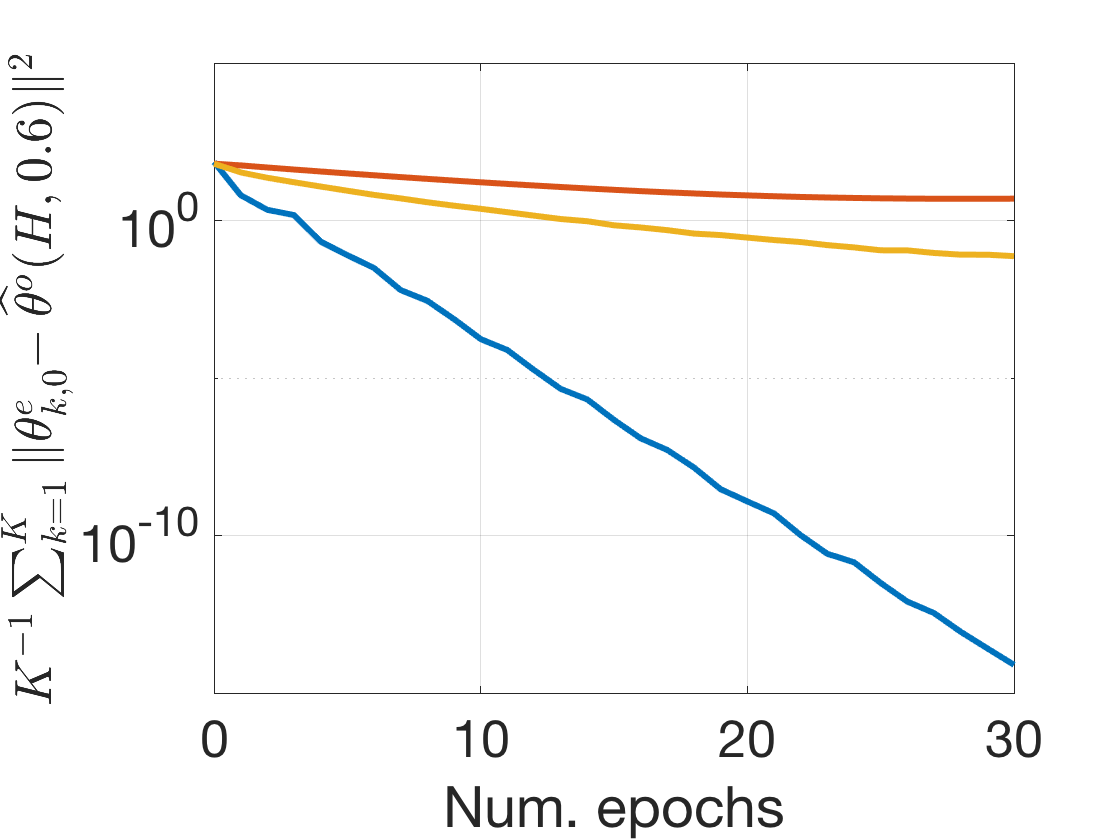}%
		\label{fig:exp1_emse}}
	\subfloat[Bias variance trade-off]{\includegraphics[width=2.4in]{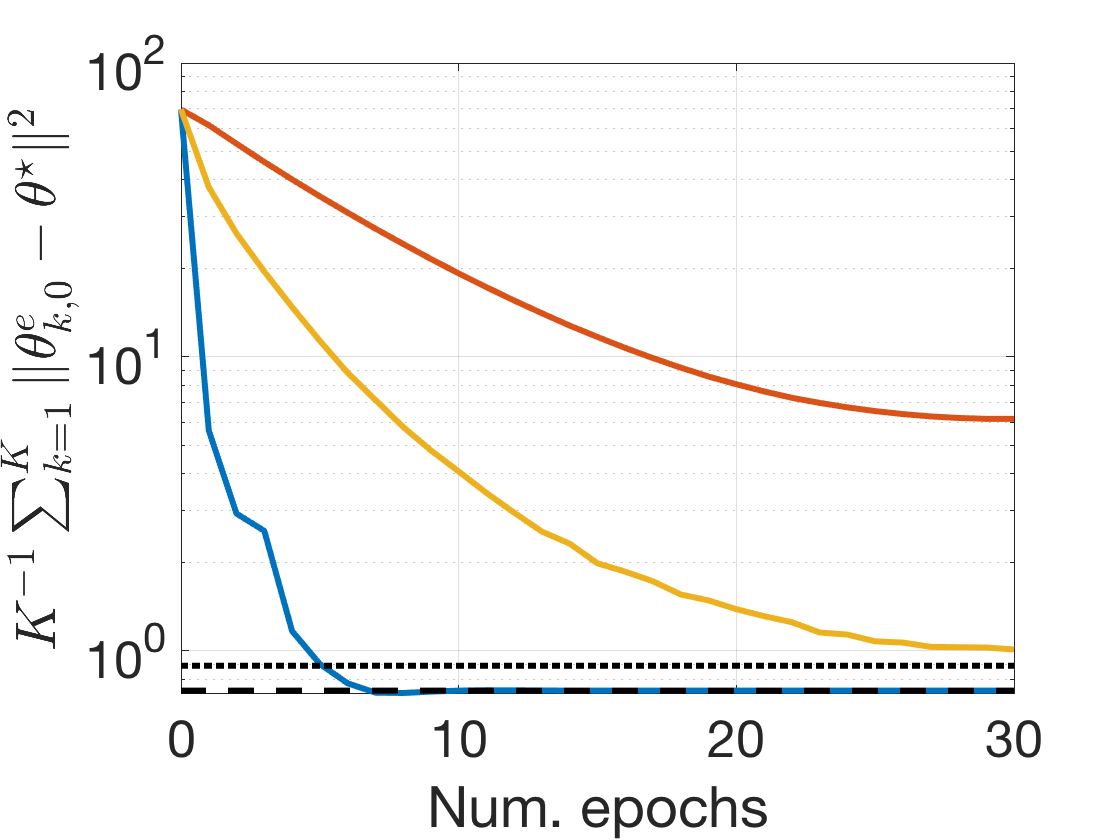}%
		\label{fig:msd1}}
%	\hfil
%	\subfloat[Variance]{\includegraphics[width=2.5in]{variance_exp1_revision.png}%
%		\label{fig:exp1_variance}}
	\caption{In (b) and (c) the red curves are the approximations from lemmas \ref{lemma:bias} and \ref{lemma:variance}. In (e) and (f), the blue curve corresponds to \textit{FDPE} and red and yellow correspond to \textit{Diffusion GTD2} and \textit{ALG2}, respectively. \textcolor{black}{In (f) the dotted line is at $\|\widehat{\theta}(H,\lambda=0)-\theta^\star\|^2$ while the dashed line is at $\|\widehat{\theta}(H,\lambda=0.6)-\theta^\star\|^2$.}}
	\label{fig_sim}
\end{figure*}

\section{Simulations}
\label{Sec: Simulations}

In this section we show two simulations corresponding to the two distinct scenarios that \textit{FDPE} can be applied to. A case application to a swarm of UAVs (Underwater Autonomous Vehicles) can be found in \cite{cassano_ECC}. 

\subsection{Experiment I}

This experiment corresponds to the scenario of Section \ref{Sec: Distributed}. We consider a situation in which the MDP's state space is divided among the agents for exploration.

The MDP's specifications are as follows. The state space is given by a $15\times15$ grid and the possible actions are: UP, DOWN, LEFT and RIGHT. The reward structure of the MDP and the target policy were generated randomly. We consider a network of $9$ agents that divide the state space in 9 regions. The topology of the network and the regions assigned to the agents for exploration are shown in Figure \ref{fig:robots}. The feature vectors consist of 26 features, 25 given by radial basis functions\footnote{Each RBF is given by $\exp\big(0.5\big((x-x_c)^2+(y-y_c)^2\big)\big)$, where $x_c$ and $y_c$ are the coordinates of the center of that feature and $x$ and $y$ are the coordinates of the agent.} (RBF) centered in the red marks show in Figure \ref{fig:robots} plus one bias feature fixed at $+1$. The behavior policy of every agent is equal to the target policy, except in the edges of its exploration region, where the probabilities of the actions that would take the agent beyond its exploration region are zeroed (the policy is further re-normalized).

In this experiment we show the bias-variance trade-off handled by the eligibility trace parameter $\lambda$, the communication-computation trade-off handled by the mini-batch size, and the performance of \textit{FDPE} compared to the existing algorithms that can be applied to this scenario (namely \textit{Diffusion GTD2} and \textit{ALG2}). The hyper-parameters chosen for \textit{FDPE} are $\tau_k=1/9$, $H=20$, $\eta=0$ and $N_k=2^{15}+H-1$.

\textcolor{black}{In figures \ref{fig:exp1_bias} and \ref{fig:exp1_bias_and variance} we show the bias and bias+variance curves as functions of $\lambda$ and its approximations using lemmas \ref{lemma:bias} and \ref{lemma:variance}, respectively\footnote{\textcolor{black}{The constants $\kappa_1$, $\kappa_2$, $\kappa_3$ and $\kappa_4$ were chosen to better fit the curves.}}. Note that the expressions provided in Lemmas \ref{lemma:bias} and \ref{lemma:variance} accurately capture the dependence of the bias and variance as functions of $\lambda$. The bias curve was calculated using \eqref{eq:t_star} and \eqref{eq:parc_min}. To estimate the combined effects of the bias and variance we calculated $\|\widehat{\theta}^o(H,\lambda)-\theta^\star\|^2$ (using expressions \eqref{eq:AbC}) 20 times with independently generated data and averaged the results. Note that the obtained curves agree with our previous discussion on the effect of the parameter $\lambda$. In this experiment, the optimal value is approximately $\lambda=0.6$. Figure \ref{fig:exp1_comp_comm} shows how communication and computation can be traded through the use of the mini-batch size. To obtain this figure, we fixed $\lambda=0.6$ and run \textit{FDPE} until an error smaller than $10^{-10}$ was obtained for the different batch sizes. For each case, the step-sizes were adjusted to maximize performance for a fair comparison. Note that all points are Pareto optimal, and hence the optimal choice of mini-batch size depends on every particular implementation. The y-axis displays the amount of communication rounds that took place over the entire optimization process, while the x-axis shows the amount of sample gradients calculated. Figure \ref{fig:exp1_emse} shows the empirical squared error for the three algorithms (each curve was obtained by averaging the squared errors from all the agents), where clearly the linear convergence of our algorithm can be seen. Finally, Figure \ref{fig:msd1} shows the mean square deviation (MSD), i.e., $\|\theta_{k,0}^e-\theta^\star\|^2$. As can be seen, our algorithm still outperforms the other algorithms in terms of convergence speed. It also has the advantage that it converges to a lower error (the dashed line) versus the other algorithms which converge to the dotted line, although in this case since the variance is high this advantage is not very significant. Note that as more data becomes available the variance of $\widehat{\theta}^o(H,\lambda)$ becomes smaller (see Lemma \ref{lemma:variance}) and hence the advantages (in terms of convergence speed and the minimizer that the algorithms converges to) of using \textit{FDPE} over the other algorithms becomes more pronounced.} The remaining hyper-parameters for \textit{FDPE} were: $\tau_k=1/9$, $J=2^{10}$ (i.e., batch size equal to $32$), $\mu_{\theta}=10$ and $\mu_{\omega}=16$. For \textit{Diffusion GTD2} and \textit{ALG2} decaying step-sizes were employed to guarantee convergence. The step-sizes decayed as $\mu(1+0.01e)^{-1}$, where $e$ is the epoch number (we used this decaying rule because it provided the best results). The initial step-sizes were $\mu_{\theta}=1.1$ and $\mu_{\omega}=1.7$ for \textit{Diffusion GTD2}, and $\mu_{\theta}=2.5$ and $\mu_{\omega}=4$ for \textit{ALG2}.

\subsection{Experiment II}
\begin{figure*}[!t]
	\centering
	\subfloat[Network]{\includegraphics[width=2.4in]{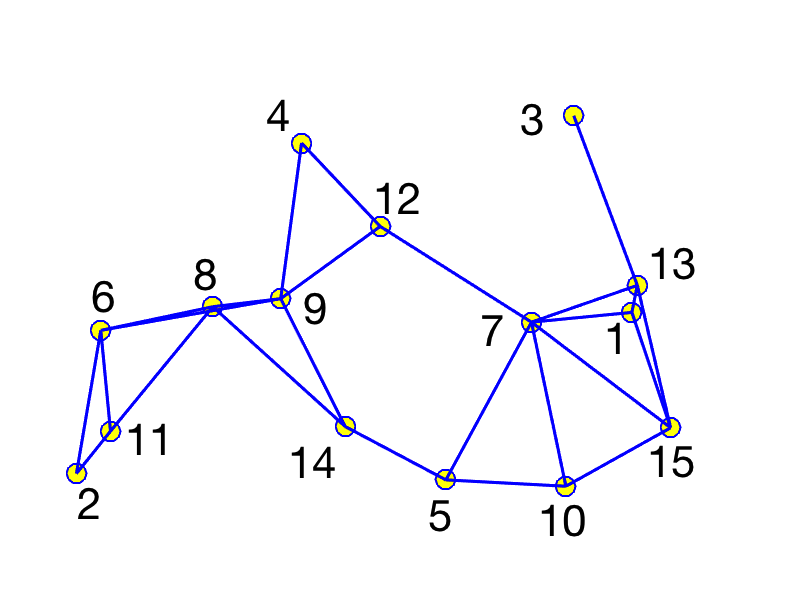}
		\label{fig:networks}}
	\subfloat[Bias]{\includegraphics[width=2.37in]{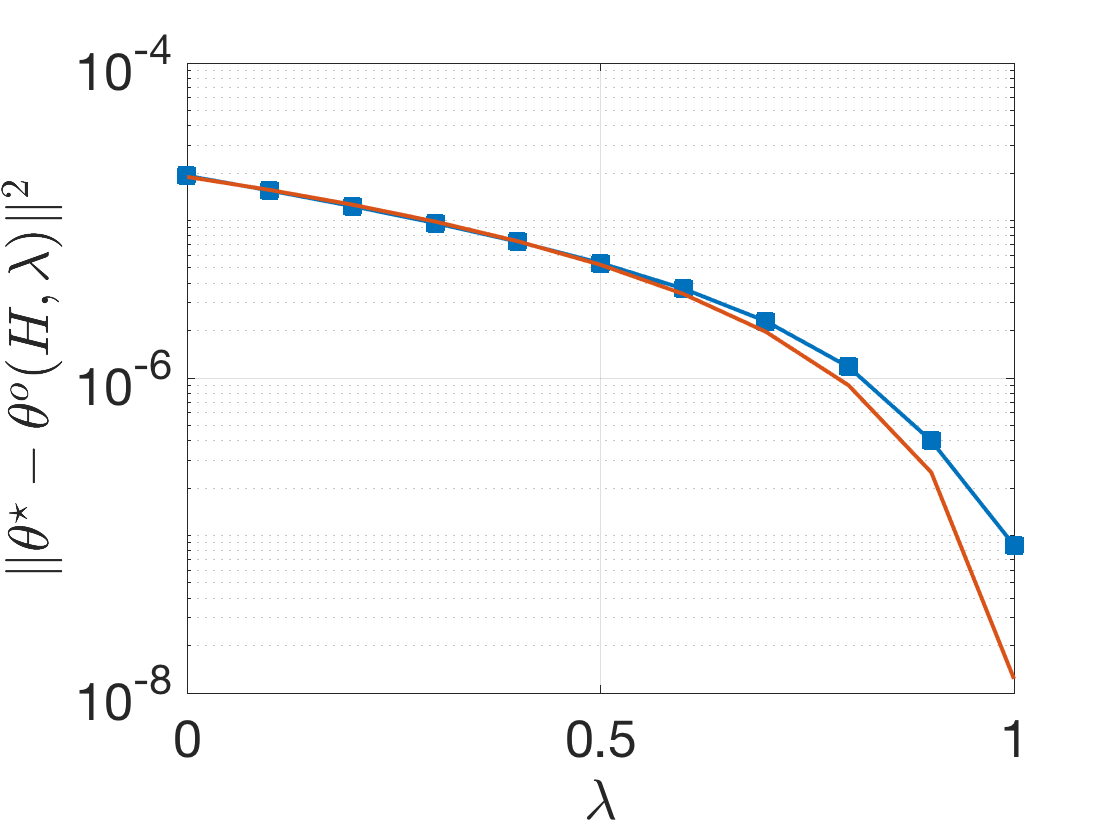}%
		\label{fig:exp2_bias}}
	\subfloat[Bias variance trade-off]{\includegraphics[width=2.37in]{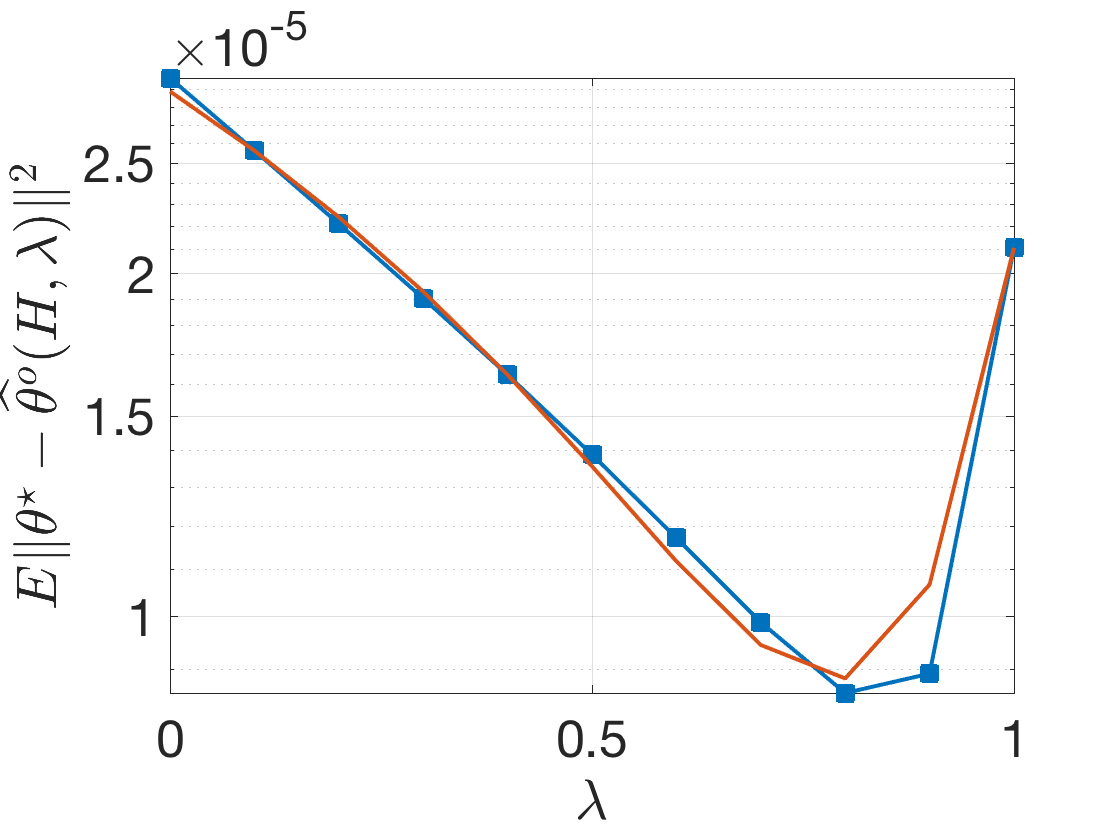}%
		\label{fig:exp2_bias_and_variance}}
	\hfil
	\subfloat[Empirical error]{\includegraphics[width=2.4in]{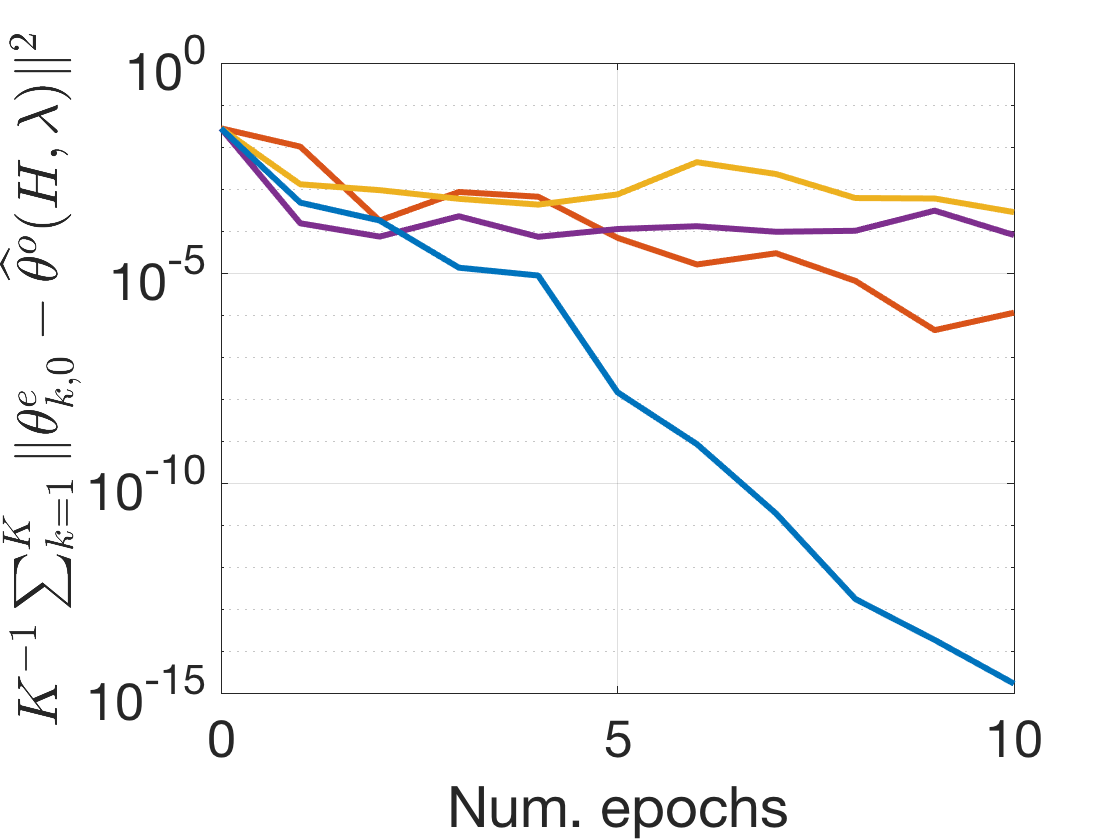}%
		\label{fig:exp2_empirical}}
	\subfloat[Mean square deviation]{\includegraphics[width=2.4in]{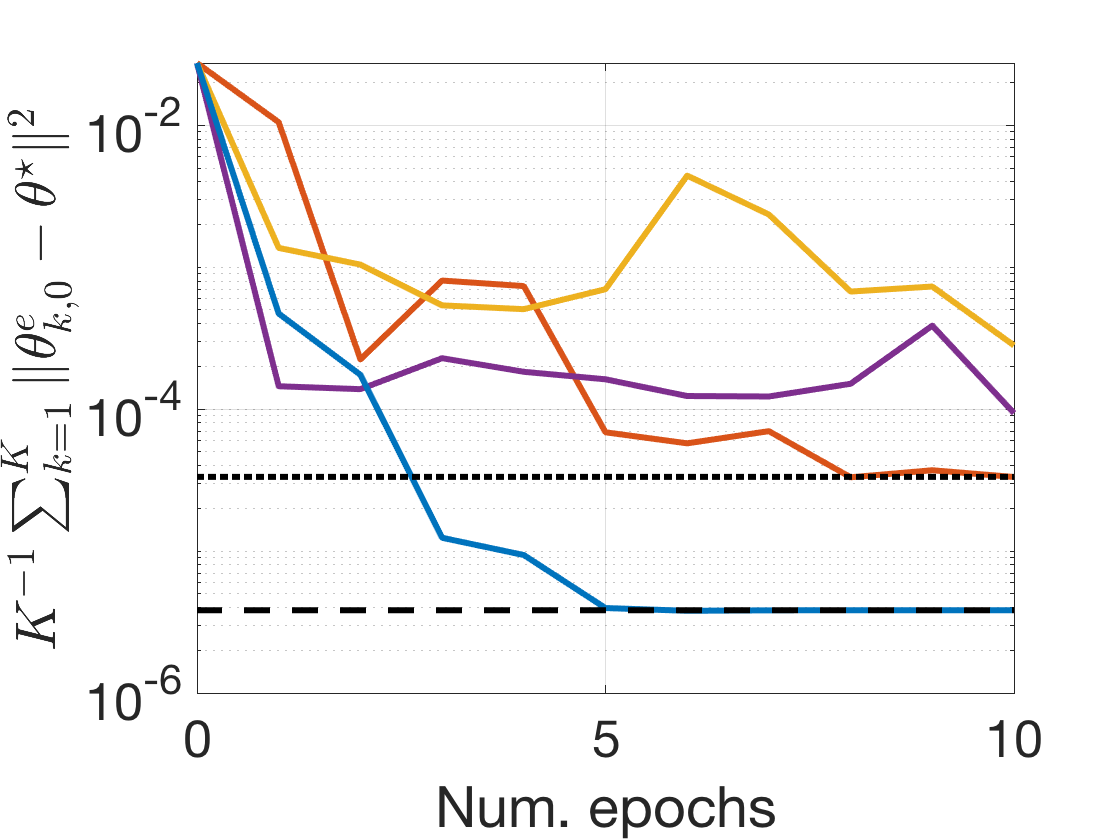}%
			\label{fig:msd2}}
	\caption{In (b), (c) and (d) the red curves are the approximations from lemmas \ref{lemma:bias} and \ref{lemma:variance}. In (e) the blue, purple, yellow and red curves correspond to \textit{FDPE}, \textit{Diffusion GTD2}, \textit{ALG2} and \textit{PD-distIAG}, respectively. \textcolor{black}{In (d) the dotted line is at $\|\widehat{\theta}(H,\lambda=0)-\theta^\star\|^2$ while the dashed line is at $\|\widehat{\theta}(H,\lambda=0.6)-\theta^\star\|^2$.}}
	\label{fig_sim_2}
\end{figure*}
The second experiment relates to the MARL scenario of section \ref{Sec:Marl}. Similarly to \cite{dann}, we consider randomly generated MDP's. We consider a random network of $K=15$ agents. To construct the network, $K$ agents are randomly distributed in a unit square (using a uniform distribution) and agents that are within a distance smaller than $r=0.27$ become neighbors, which results in a sparsely connected network. The resulting network is shown in Figure \ref{fig:networks}. The combination weights are determined according to the Metropolis rule which is given by:
\begin{align}
\ell_{nk}=\begin{cases}
\frac{1}{\max[|\mathcal{N}_n|,|\mathcal{N}_k|]} &n\in \mathcal{N}_k\backslash\{k\}\\
1-\sum\limits_{m\in\mathcal{N}_k\backslash\{k\}}\ell_{mk} &n=k
\end{cases}
\end{align}
The generated MDP's have 50 states and 10 actions. The transition probabilities $p(s'|s,a)$ are zero with 0.98 probability and otherwise are sampled from a uniform distribution from the interval [0,1]. The probabilities are further normalized. With this sampling strategy, we produce realistic MDPs in which from any given state it is only possible to transition to a small subset of the total states.  The rewards $r(s'|s,a)$ are zero with 0.99 probability and otherwise are sampled from a Gaussian distribution with zero mean and standard deviation equal to 10. This sampling strategy is also devised to produce more realistic MDPs where rewards are obtained occasionally in specific state-action pairs. The entries of the target policy $\pi(a|s)$ are sampled from a uniform distribution and subsequently normalized. The transition probabilities and target policy are sampled until Assumption 1 is satisfied. We set the discount factor $\gamma=0.93$ and the length of the feature vectors $M=5$, where one feature is set to 1 and the remaining ones are sampled from a uniform distribution with interval $[0,1]$. We generated $N=2^{18}+H-1$ team transitions. The remaining parameters for the learning algorithm are the following: $H=20$, $\eta=10^{-3}$, $U=I$, $\theta_{\text{p}}=\theta^o(H,\lambda)+\theta_{\text{n}}$ (where $\theta_{\text{n}}$ is a noise vector whose entries are sampled from a uniform distribution with variance equal to $2.5\times10^{-5}$) and $\tau_k=1/15$.

We compare our algorithm with \textit{PD-distIAG} from \cite{wai2018multi}, \textit{Diffusion GTD2} and \textit{ALG2}. In this experiment, we test the on-policy case because \textit{PD-distIAG} only works for this scenario. In Figures \ref{fig:exp2_bias} and \ref{fig:exp2_bias_and_variance}, we show the bias variance trade-off as a function of $\lambda$ (the curves were obtained in the same manner as done in the previous experiment). Results are consistent with the ones obtained in the previous section. In this particular case the most convenient value is approximately $\lambda=0.8$. In Figure \ref{fig:exp2_empirical} we compare the convergence rates of the different algorithms to solve the empirical problem. The hyper-parameters of all algorithms were tunned to maximize performance. The parameters for \textit{FDPE} were: $J=2^{12}$ (i.e., batch size equal to $64$), $\mu_{\theta}=10$ and $\mu_{\omega}=10$. For \textit{PD-distIAG} we used the same batch size, and the step-sizes were $\mu_{\theta}=1.15$ and $\mu_{\omega}=23$. For \textit{Diffusion GTD2} and \textit{ALG2} decaying step-sizes were employed to guarantee convergence. The step-sizes decayed as $\mu(1+0.01e)^{-1}$, where $e$ is the epoch number (we used this decaying rule because it provided the best results). The initial step-sizes were $\mu_{\theta}=15$ and $\mu_{\omega}=7.5$ for \textit{Diffusion GTD2}, and the same values for \textit{ALG2}. In this experiment \textit{FDPE}, \textit{Diffusion GTD2} and \textit{ALG2} show performance in accordance to our theory and the results in the previous section. \textit{PD-distIAG} shows a linear convergence rate in accordance to the theory from \cite{wai2018multi}. However, the rate is slower than the one from our algorithm. \textcolor{black}{Finally, in Figure \ref{fig:msd2} we show the MSD. Again, the faster convergence of our algorithm to its empirical minimizer implies a faster convergence in the MSD plot. In this case, the advantage of the parameter $\lambda$ becomes more noticeable. Note indeed that the minimizer obtained by \textit{FDPE} is approximately one order of magnitude smaller than the minimizer to which the other algorithms will converge (the dotted line).}

\appendices
\section{Proof of Lemma \ref{lemma:bias}}\label{app:bias}
\textcolor{black}{To simplify the notation we refer to $\theta^o(H,\lambda)$ as $\theta^o$. Using \eqref{eq:parc_min} and defining $R=A^{\hspace{-0.5mm}-\hspace{-0.5mm}1}CA^{\hspace{-0.5mm}-\hspace{-0.5mm}T}U$, we can write $\theta^o=(I+\eta R)^{\hspace{-0.5mm}-\hspace{-0.5mm}1}(\eta R\theta_\text{p}+A^{\hspace{-0.5mm}-\hspace{-0.5mm}1}b)$. Using the Jordan decomposition $R=J_R\Lambda_RJ_R^{-1}$ we get:
\begin{align}
	&\theta^o=(I+\eta R)^{-1}A^{-1}b+\eta J_R(I+\eta\Lambda_R)^{-1}\Lambda_RJ_R^{-1}\theta_\text{p}
\end{align}
Note that $\theta^o$ can be approximated by the following expression, which has the advantage that it separates the effect of the regularization term on the bias:
\begin{align}\label{eq:bias_separation}
&\theta^o=(I+\kappa_1\eta)^{-1}A^{-1}b+\kappa_1\eta(1+\kappa_1\eta)^{-1}\theta_{\text{p}}
\end{align}
We proceed to calculate the bias of $\theta^\bullet=A^{-1}b$ with respect to $\theta^\star$. Using \eqref{eq:l_weighted_BE} and \eqref{eq:l_weighted_Projected_Bellman_Equation} we can write:
\begin{align}
	v^\pi&=\Gamma_2r^\pi+\rho_1\Gamma_1v^\pi=\sum_{n=0}^{\infty}(\rho_1\Gamma_1)^n\Gamma_2r^\pi\label{eq:v_id}\\
	X\theta^\bullet&=\Pi\left(\Gamma_2r^\pi+\rho_1\Gamma_1X\theta^\bullet\right)=\Pi\sum_{n=0}^{\infty}(\rho_1\Gamma_1\Pi)^n\Gamma_2r^\pi\\
	X\theta^\star&=\Pi v^\pi=\Pi\sum_{n=0}^{\infty}(\rho_1\Gamma_1)^n\Gamma_2r^\pi
\end{align}		
Combining the expressions from above we get:
\begin{align}
X(\theta^\bullet-\theta^\star)&=\Pi\sum_{n=0}^{\infty}\big((\rho_1\Gamma_1\Pi)^n\Gamma_2r^\pi-v^\pi\big)\label{eq:an_bias1}\\
&=\Pi\sum_{n=1}^{\infty}\big((\rho_1\Gamma_1\Pi)^n-(\rho_1\Gamma_1)^n\big)\Gamma_2r^\pi\label{eq:an_bias2}
\end{align}
Note that if $v^\pi=\Pi v^\pi=X\theta^\star$, then \eqref{eq:v_id} can be rewritten as:
\begin{align}
	v^\pi&=\sum_{n=0}^{\infty}(\rho_1\Gamma_1\Pi)^n\Gamma_2r^\pi\label{eq:an_bias3}
\end{align}
Combining \eqref{eq:an_bias3} with \eqref{eq:an_bias1} we get that if $v^\pi=\Pi v^\pi$, then $\check{\theta}-\theta^\star=0$. Therefore, we can write \eqref{eq:an_bias2} as:
\begin{align}
X(\theta^\bullet-\theta^\star)&=\mathbb{I}(v^\pi\neq\Pi v^\pi)\Pi\sum_{n=1}^{\infty}\rho_1^n\big((\Gamma_1\Pi)^n-\Gamma_1^n\big)\Gamma_2r^\pi\label{eq:an_bias4}
\end{align}
where $\mathbb{I}$ is the indicator function. We now approximate the right stochastic matrix $P^\pi$ by it's steady state limit (given by $(P^\pi)^\infty=\one p^T$, where $\one$ is the all ones vector and $p$ is the vector with the steady state distribution induced by the transition matrix $P^\pi$). Consequently, we get:
\begin{align}
	&\Gamma_1\approx\one p^T\\
	&\Gamma_2r^\pi\approx\frac{1-(\gamma\lambda)^H}{1-\gamma\lambda}\one p^Tr^\pi=\left(\frac{1-\rho_1}{1-\gamma}\right)\one p^Tr^\pi\\
	&\epsilon=p^T\Pi\one\\
	&\sum_{n=1}^{\infty}\rho_1^n\left((\Gamma_1\Pi)^n\hspace{-1mm}-\Gamma_1^n\right)\Gamma_2r^\pi\hspace{-1mm}\approx\hspace{-1mm}\sum_{n=1}^{\infty}\rho_1^n\left((\one p^T\Pi)^n\hspace{-1mm}-\one p^T\right)\Gamma_2r^\pi\\
	&\hspace{2.5cm}=\one\left(\frac{1-(\gamma\lambda)^H}{1-\gamma\lambda}\right)p^Tr^\pi\sum_{n=1}^{\infty}\rho_1^n\left(\epsilon^n-1\right)
\end{align}
Since $p^T\one=1$ and $\Pi$ is a projection matrix (and therefore its eigenvalues are either $0$ or $1$) we get that $|\epsilon|\leq1$ and therefore we can write: 
\begin{align}
	&\sum_{n=1}^{\infty}\rho_1^n\left(\epsilon^n-1\right)=\frac{\rho_1\epsilon}{1-\rho_1\epsilon}-\frac{\rho_1}{1-\rho_1}=\frac{\rho_1(\epsilon-1)}{(1-\rho_1\epsilon)(1-\rho_1)}\\
	&\sum_{n=1}^{\infty}\rho_1^n\left((\Gamma_1\Pi)^n\hspace{-1mm}-\Gamma_1^n\right)\Gamma_2r^\pi\hspace{-1mm}\approx\one p^T\hspace{-0.5mm}r^\pi\hspace{-1mm}=\hspace{-0.5mm}\mathcal{O}\hspace{-0.5mm}\bigg(\hspace{-1mm}\frac{\rho_1(\epsilon-1)}{(1-\rho_1\epsilon)(1-\gamma)}\hspace{-1mm}\bigg)\label{eq:an_bias5}
\end{align}
Therefore, combining \eqref{eq:bias_separation}, \eqref{eq:an_bias4} and \eqref{eq:an_bias5} and grouping constants we can finally write:
\begin{align}\label{eq:app_bias}
\|\hspace{-0.5mm}\theta^o\hspace{-1mm}-\theta^\star\hspace{-0.5mm}\|^2&\hspace{-1mm}\approx\hspace{-1mm}\left(\hspace{-1mm}\mathbb{I}\big(v^\pi\hspace{-1.5mm}\neq\hspace{-0.5mm}\Pi v^\pi\hspace{-0.5mm}\big)\hspace{-0.5mm}\frac{\kappa_2\rho_1}{(\hspace{-0.5mm}1\hspace{-0.5mm}+\hspace{-0.5mm}\kappa_{\hspace{-0.5mm}1}\hspace{-0.2mm}\eta)(\kappa_3\hspace{-0.5mm}-\hspace{-0.5mm}\rho_1\hspace{-0.5mm})}+\frac{\kappa_1\eta\|\theta_{\text{p}}-\theta^\star\|}{1+\kappa_1\eta}\right)^{\hspace{-1mm}2}
\end{align}}

\section{Proof of Lemma \ref{lemma:variance}}\label{app:variance}
\textcolor{black}{Since the regularization term of \eqref{eq:bias_separation} is not subject to variance, to calculate the variance in the estimate of $\widehat{\theta^o}$ we need to calculate the variance of $\widehat{\theta^\bullet}=\widehat{A}^{-1}\hat{b}$ as follows:
\begin{align}
\Ex\|\widehat{\theta^\bullet}-\theta^\bullet\|^2&=\Ex\|\widehat{A}^{-1}\hat{b}-\theta^\bullet\|^2=\Ex\left\|\widehat{A}^{-1}\big(\hat{b}-\widehat{A}\theta^\bullet\big)\right\|^2
\end{align}
where we are using the squared Euclidean norm. Due to the Rayleigh-Ritz' Theorem we have $\sigma_\textrm{min}(\widehat{A}^{-1})\leq\|\widehat{A}^{-1}\|\leq\sigma_\textrm{max}(\widehat{A}^{-1})$ and since $\sigma_\textrm{max}(\widehat{A})=\sigma_\textrm{min}(\widehat{A}^{-1})$ and $\sigma_\textrm{max}(\widehat{A}^{-1})=\sigma_\textrm{min}(\widehat{A})$ we can write:
\begin{align}
\Ex\|\widehat{\theta^\bullet}-\theta^\bullet\|^2&\leq\Ex\sigma_\textrm{min}^2(\widehat{A})\left\|\hat{b}-\widehat{A}\theta^\bullet\right\|^2\\
\Ex\|\widehat{\theta^\bullet}-\theta^\bullet\|^2&\geq\Ex\sigma_\textrm{max}^2(\widehat{A})\left\|\hat{b}-\widehat{A}\theta^\bullet\right\|^2
\end{align}
Now using Proposition 9.6.8 of \cite{bernstein2009matrix} and the fact that $\widehat{A}=\widehat{C}-X^TD\widehat{\Gamma}_1X$ we can write:
\begin{align}
	&\sigma_\textrm{min}(\widehat{A})=\lambda_\textrm{min}(\widehat{C})\pm\rho_1\sigma_\textrm{max}(\widehat{E})=\lambda_\textrm{min}(\widehat{C})+\mathcal{O}(\rho_1)\\
	&\sigma_\textrm{max}(\widehat{A})=\lambda_\textrm{max}(\widehat{C})\pm\rho_1\sigma_\textrm{max}(\widehat{E})=\lambda_\textrm{max}(\widehat{C})+\mathcal{O}(\rho_1)\\
	&\widehat{E}=X^TD\widehat{\Gamma}_1X
\end{align}
Using the above results and the intermediate value theorem we can write:
\begin{align}
\Ex\|\widehat{\theta^\bullet}-\theta^\bullet\|^2&=\left(\epsilon_2+\mathcal{O}(\rho_1)\right)^2\big\|\hat{b}-\widehat{A}\theta^\bullet\big\|^2\label{eq:app_var_sing}
\end{align}
for some constant $\epsilon_2$. For $\Ex\big\|\hat{b}-\widehat{A}\theta^\bullet\big\|^2$ we can write:
\begin{align}
	\Ex\big\|\hat{b}-\widehat{A}\theta^\bullet\big\|^2&=\Ex\left\|\sum_{t=1}^{N-H}\frac{\hat{b}_t-\widehat{A}_t\theta^\bullet}{N-H}\right\|^2\stackrel{(b)}{=}\frac{\Ex\|\hat{b}_t-\widehat{A}_t\theta^\bullet\|^2}{N-H}\label{eq:app_var_ind}
\end{align}
where in $(b)$ we assumed that the total amount of data collected is significantly larger than the mixing rate of the Markov Chain, and therefore the terms $\hat{b}_t-\widehat{A}_t\theta^\bullet$ corresponding to different times can be considered independent of each other. Note that this is a standard assumption (similar to the one stated in Assumption \ref{assumption:data}) in the sense that in practice it demands that enough data is gathered so that the effect of the policy on every state can be accurately estimated. To approximate $\Ex\big\|\hat{b}_t-\widehat{A}_t\theta^\bullet\big\|^2$ we proceed as follows:
\begin{align}
&\Ex\big\|\hat{b}_t-\widehat{A}_t\theta^\bullet\big\|^2=\Ex\|x_t\|^2\bigg((\gamma\lambda)^{H-1}(r_{t+H-1}+\gamma x_{t+H}^T\theta^\bullet)\nonumber\\
&+\sum_{n=0}^{H-2}(\gamma\lambda)^n\left(r_{t+n}+\gamma(1-\lambda)x_{t+n+1}^T\theta^\bullet\right)-x_t^T\theta^\bullet\bigg)^2\\
&\stackrel{(a)}{=}\Ex\|x_t\|^2\left(\sum_{n=0}^{H-1}(\gamma\lambda)^n\left(x_{t+n}^T\theta^\bullet-r_{t+n}-\gamma x_{t+n+1}^T\theta^\bullet\right)\hspace{-0.5mm}\right)^{\hspace{-0.5mm}2}\\
&\stackrel{(b)}{\approx}\hspace{-0.5mm}\Ex\hspace{-0.5mm}\|x_t\|^2\left(\sum_{n=0}^{H-1}(\gamma\lambda)^n\left(v(s_{t+n})-r_{t+n}\hspace{-0.5mm}-\gamma v(s_{t+n+1})\right)\hspace{-1mm}\right)^{\hspace{-1.2mm}2}\\
&\stackrel{(c)}{=}\Ex\|x_t\|^2\sum_{n=0}^{H-1}(\gamma\lambda)^{2n}\left(v(s_{t+n})-r_{t+n}-\gamma v(s_{t+n+1})\right)^2\label{eq:app_var}
\end{align}
where in $(a)$ we simply reorganized the terms, in $(b)$ we used $x_{t}^T\theta^\bullet\approx v(s_t)$ and in $(c)$ we used the fact that due to the Markov property each term $v(s_{t+n})-r_{t+n}-\gamma v(s_{t+n+1})$ is conditionally independent from all previous terms (conditioned on $s_{t+n}$) and that by definition $\Ex\left[v(s_{t+n})-r_{t+n}-\gamma v(s_{t+n+1})|s_{t+n}\right]=0$. Now we lower and upper bound \eqref{eq:app_var} as follows:
\begin{align}
	&m\sum_{n=0}^{H-1}(\gamma\lambda)^{2n}\leq\Ex\left\|\hat{b}_t-\widehat{A}_t\theta^\bullet\right\|^2\leq M\sum_{n=0}^{H-1}(\gamma\lambda)^{2n}\\
	&M=\max_{t,n}\|x_t\|^2\left(v(s_{t+n})-r_{t+n}-\gamma v(s_{t+n+1})\right)^2\\
	&m=\min_{t,n}\|x_t\|^2\left(v(s_{t+n})-r_{t+n}-\gamma v(s_{t+n+1})\right)^2
\end{align}
Combining the above result with the intermediate value theorem we get:
\begin{align}\label{eq:app_var_individual}
&\Ex\big\|\hat{b}_t-\widehat{A}_t\theta^\bullet\big\|^2\approx\epsilon_3\left(\frac{1-(\gamma\lambda)^{2H}}{1-(\gamma\lambda)^2}\right)
\end{align}
for some constant $\epsilon_3$. Finally, combining \eqref{eq:bias_separation}, \eqref{eq:app_var_individual}, \eqref{eq:app_var_sing} and \eqref{eq:app_var_ind} and assuming $N>>H$ (which in practice should be the case) we get:
\begin{align}
&\Ex\big\|\widehat{\theta^o}-\theta^o\big\|^2\approx\frac{\kappa_4}{(1+\kappa_1\eta)^2(N-H)}\left(\frac{1-(\gamma\lambda)^{2H}}{1-(\gamma\lambda)^2}\right)
\end{align}
where $\kappa_4=\epsilon_2\epsilon_3$.}

\section{Proof Lemma \ref{lemma:expectations}}\label{app:expectations}
To write $A$, $b$ and $C$ as expectations we use definitions \eqref{eq:definition_AbC} and \eqref{eq:definitions_ggr}, and then expand the matrix operations.
\begin{align} 
&A=X^TD(I-\rho_1\Gamma_1)X\nonumber\\
&=X^TD\left(I-\gamma(1-\lambda)P^{\pi}\sum_{n=0}^{H-1}(\gamma\lambda P^{\pi})^n+(\gamma\lambda P^{\pi})^H\right)X\nonumber\\
&=\hspace{-1mm}\sum_{s_{t}\in\mathcal{S}}\hspace{-1mm}d^{\phi}(s_t)x_{s_t}\hspace{-1mm}\bigg(x_{s_t}\hspace{-1mm}-\gamma(1-\lambda)\hspace{-1mm}\sum_{n=0}^{H-1}(\gamma\lambda)^n\hspace{-4mm}\sum_{s_{t+1+n}\in\mathcal{S}}\hspace{-3mm}p_{s_t,s_{t+1+n}}^\pi x_{s_{t+1+n}}\nonumber\\
&+\hspace{-0.6mm}(\gamma\lambda)^H\hspace{-3mm}\sum_{s_{t+H}\in\mathcal{S}}\hspace{-2mm}p_{s_t,s_{t+H}}^\pi x_{s_{t+H}}\bigg)^T\nonumber\\
&=\Ex_{d^\phi,\mathcal{P},\pi}\Bigg[\hspace{-0.5mm}\bm{x}_{t}\hspace{-0.5mm}\bigg(\hspace{-1mm}\bm{x}_{t}-\gamma(1-\lambda)\hspace{-0.5mm}\sum_{n=0}^{H-1}\hspace{-0.5mm}(\hspace{-0.2mm}\gamma\lambda\hspace{-0.2mm})^n\bm{x}_{t+n+1}-(\hspace{-0.2mm}\gamma\lambda\hspace{-0.2mm})^H\bm{x}_{t+H}\hspace{-1mm}\bigg)^{\hspace{-1.2mm}T}\Bigg]\\
&b=X^TD\sum_{n=0}^{H-1}(\gamma\lambda P^{\pi})^nr^\pi\nonumber\\
&=\hspace{-1mm}\sum_{s_{t}\in\mathcal{S}}\hspace{-1mm}d^{\phi}(s_t)x_{s_t}\hspace{-0.5mm}\sum_{a\in\mathcal{L}}\sum_{\hspace{2mm}s_{t+1}\in\mathcal{S}}\hspace{-2mm}\pi(a|s_{t})\mathcal{P}(s_{t+1}|s_{t},a) r(s_{t},a,s_{t+1})\nonumber\\
&+\sum_{s_{t}\in\mathcal{S}}d^{\phi}(s_t)x_{s_t}\hspace{-1mm}\sum_{n=1}^{H-1}\hspace{-0.8mm}(\gamma\lambda)^np_{s_t,s_{t+n}}^\pi\hspace{-3mm}\sum_{s_{t+n}\in\mathcal{S}}\sum_{\hspace{2mm}a\in\mathcal{L}}\cdot\nonumber\\
&\sum_{s_{t+n+1}\in\mathcal{S}}\hspace{-4mm}\pi(a|s_{t+n})\mathcal{P}(s_{t+n+1}|s_{t+n},a)r(s_{t+n},\hspace{-0.2mm}a,\hspace{-0.2mm}s_{t+n+1})\nonumber\\
&=\Ex_{d^\phi,\mathcal{P},\pi}\bigg(\bm{x}_{t}\sum_{n=0}^{H-1}(\gamma\lambda)^n\bm{r}_{t+n}\bigg)\\
&C=X^TDX=\sum_{s_{t}\in\mathcal{S}}d^{\phi}(s_t)x_{s_t}x_{s_t}=\Ex_{d^\phi}\left[\bm{x}_{t}\bm{x}_{t}^T\right]
\end{align}
\hfill\qed
\section{Proof of Lemma \ref{lemma:estimates}}\label{app:empirical}
We start with \eqref{eq:A_ex}:
\begin{align}
&A=\Ex_{d^\phi\hspace{-1mm},\mathcal{P},\pi}\hspace{-0.8mm}\Bigg[\hspace{-0.5mm}\bm{x}_{t}\hspace{-0.5mm}\bigg(\hspace{-1mm}\bm{x}_{t}-\gamma(1-\lambda)\hspace{-1mm}\sum_{h=0}^{H-1}\hspace{-1mm}(\hspace{-0.2mm}\gamma\lambda\hspace{-0.2mm})^h\bm{x}_{t+h+1}-(\hspace{-0.2mm}\gamma\lambda\hspace{-0.2mm})^H\hspace{-0.5mm}\bm{x}_{t+H}\bigg)^{\hspace{-1.2mm}T}\Bigg]\nonumber\\
&\stackrel{(a)}{=}(1-\lambda)\sum_{h=0}^{H-1}\lambda^h\Ex_{d^\phi,\mathcal{P},\pi}\Big(\bm{x}_t\big(\bm{x}_{t}-\gamma^{h+1}\bm{x}_{t+h+1}\big)^{\hspace{-0.7mm}T}\Big)\nonumber\\
&+\lambda^H\Ex_{d^\phi,\mathcal{P},\pi}\left(\bm{x}_{t}\left(\bm{x}_{t}-\gamma^{H}\bm{x}_{t+H}\right)^T\right)\nonumber\\
&\stackrel{(b)}{=}(1-\lambda)\sum_{h=0}^{H-1}\lambda^h\Ex_{d^\phi,\mathcal{P},\phi,f_r}\Big[\hspace{-0.2mm}\bm{\xi}_{t,t+h+1}\bm{x}_{t}\big(\bm{x}_{t}-\gamma^{h+1}\bm{x}_{t+h+1}\big)^{\hspace{-0.2mm}T}\hspace{-0.2mm}\Big]\nonumber\\
&+\lambda^H\Ex_{d^\phi,\mathcal{P},\phi,f_r}\hspace{-0.5mm}\bigg[\hspace{-0.7mm}\bm{\xi}_{t,t+H}\bm{x}_{t}\hspace{-0.6mm}\big(\bm{x}_{t}-\gamma^{H}\bm{x}_{t+H}\big)^{\hspace{-0.5mm}T}\hspace{-0.5mm}\bigg]\nonumber\\
&\stackrel{(c)}{=}\Ex_{d^\phi,\mathcal{P},\phi,f_r}\Bigg[\bm{x}_{t}\Bigg(\left((1-\lambda)\sum_{h=0}^{H-1}\lambda^h\bm{\xi}_{t,t+h+1}+\lambda^H\bm{\xi}_{t,t+H}\right)\bm{x}_{t}\nonumber\\
&-\gamma(1-\lambda)\sum_{h=0}^{H-1}(\hspace{-0.2mm}\gamma\lambda\hspace{-0.2mm})^h\bm{\xi}_{t,t+h+1}\bm{x}_{t+h+1}-(\hspace{-0.2mm}\gamma\lambda\hspace{-0.2mm})^H\bm{\xi}_{t,t+H}\bm{x}_{t+H}\Bigg)^{\hspace{-1.5mm}T}\Bigg]\label{eq:A_est_off}
\end{align}
where in ($a$) and ($c$) we simply rearranged terms. And in ($b$) we introduced the importance sampling weights corresponding to the trajectory that started at some state $s_t$ and took $h$ steps before arriving at some other state $s_{t+h}$, which are given by:
\begin{align}
	\xi_{t,t+h}=\prod_{j=t}^{t+h-1}\hspace{-2mm}\pi(a_j|s_j)/\phi(a_j|s_j)
\end{align}
Hence, by removing the expectation in \eqref{eq:A_est_off}, we can get the following estimate of $A$ using a single $H$-step trajectory:
\begin{align}
	\widehat{A}_n&=x_{n}\bigg(\rho_{n,0}^{H}x_{n}-\gamma(1-\lambda)\sum_{h=0}^{H-1}(\gamma\lambda)^h\xi_{n,n+h+1}x_{n+h+1}\nonumber\\
	&-(\gamma\lambda)^H\xi_{n,n+H}x_{n+H}\bigg)^T
\end{align}
where we defined:
\begin{align}
\rho_{t,n}^H&=\bigg((1-\lambda)\sum_{h=n}^{H-1}\lambda^{h-n}\bm{\xi}_{t,t+h+1}+\lambda^{H-n}\bm{\xi}_{t,t+H}\bigg)
\end{align}
Following similar same steps to estimate $b$, we get:
\begin{align}
\widehat{b}_n&=x_{n}\sum_{h=0}^{H-1}(\gamma\lambda)^h\rho_{n,h}^{H}r_{n+h}
\end{align}
The estimate for $C$ does not require importance sampling because its expectation only depends on $d^\phi$. Hence, the sample estimate is given by $\widehat{C}_n=x_nx_n^T$ which completes the proof.\hfill\qed

\textcolor{black}{\section{Proof of Lemma \ref{lemma:optimality}}\label{app:optimality}
From Remark \ref{remark:constraints} we know that \eqref{eq:fixed_point_2} implies that the consensus constraints are satisfied by $\zeta^o$ and hence for some $\theta^o$ and $\omega^o$:
\begin{align}\label{eq:rel_1}
\zeta^o=\one_K\otimes[{\theta^o}^T,\upsilon^{-1/2}{\omega^o}^T]^T
\end{align}
Multiplying \eqref{eq:fixed_point_1} by $\mathcal{I}^T\define\one_K^T\otimes I_{2M}$ we get the following condition:
\begin{align}\label{eq:rel_2}
0=\mathcal{I}^T\left(\bar{\mathcal{L}}\left(\mathcal{G}\zeta^o-p\right)+\mu^{-1}\mathcal{V}\mathcal{Y}^o\right)\stackrel{(a)}{=}\mathcal{I}^T\left(\mathcal{G}\zeta^o-p\right)
\end{align}
where in $(a)$ we used $\mathcal{I}^T\bar{\mathcal{L}}=\mathcal{I}^T$ and $\mathcal{I}^T\mathcal{V}=0$. Combining \eqref{eq:rel_1} and \eqref{eq:rel_2} we get that $\theta^o$ and $\omega^o$ must satisfy:
\begin{align}
\bigg[\begin{array}{cc}
\eta\widehat{U} & -\widehat{A}^T\\
\widehat{A} & \widehat{C}\end{array}\bigg]\ba{c}\theta^o\\\omega^o\ea=\bigg[\begin{array}{c}
\eta\widehat{U}\theta_{\text{p}}\\b\end{array}\bigg]
\end{align}
since the left hand-side matrix is invertible (because $\widehat{A}$ and $\widehat{C}$ are invertible) $\zeta^o$ is unique. Left multiplying by the inverse of such matrix we get \eqref{eq:saddle_point}. Therefore, we conclude that $\zeta^o=\one_K\otimes[\widehat{\theta}^{oT} \widehat{\omega}^{oT}]^T$. The fact that $\mathcal{I}^T\left(\mathcal{G}\zeta^o-p\right)=0$ implies that $\mathcal{G}\zeta^o-p$ lies in the range space of $\mathcal{V}$ (because it is orthogonal to $\mathcal{I}$ which lies in the null space of $\mathcal{V}$), which in turn implies that a vector $\mathcal{Y}$ that satisfies \eqref{eq:fixed_point_1} exists. We conclude the proof by showing that there is unique $\mathcal{Y}^o$ that satisfies \eqref{eq:fixed_point_1} and lies in the range space of $\mathcal{V}$. We do this by contradiction. Assume that there are two fixed points $(\zeta^o,\mathcal{Y}_1=\mathcal{V}a_1)$ and $(\zeta^o,\mathcal{Y}_2=\mathcal{V}a_2)$ such that $\mathcal{Y}_1\neq\mathcal{Y}_2$, applying both to \eqref{eq:fixed_point_1} and subtracting the resulting equations we get:
\begin{align}
\mathcal{V}^2(a_1-a_2)&=0\iff\mathcal{V}(a_1-a_2)=0\implies\mathcal{Y}_1=\mathcal{Y}_2
\end{align}
which is a contradiction. This concludes the proof.}

\textcolor{black}{\section{Proof of Lemma \ref{lemma:coord_transform}}\label{app:coord_transform}
We start by doing the following matrix decomposition:
\begin{align}
&\bigg[\hspace{-2mm}\begin{array}{cc}
\bar{L} & \hspace{-2mm}-V\\
V\bar{L} & \hspace{-2mm}\bar{L}\end{array}\hspace{-2mm}\bigg]\hspace{-1mm}\stackrel{(a)}{=}\hspace{-1mm}\Bigg[\hspace{-2mm}\begin{array}{cc}H(\frac{I_K+\Lambda}{2})H^T &\hspace{-2mm}-H\left(\frac{I_K-\Lambda}{2}\right)^{\frac{1}{2}}H^T\\
H\left(\frac{I_K+\Lambda}{2}\right)\left(\frac{I_K-\Lambda}{2}\right)^{\frac{1}{2}}H^T & \hspace{-2mm}H(\frac{I_K+\Lambda}{2})H^T\end{array}\hspace{-2mm}\Bigg]\nonumber\\
&=\Bigg[\hspace{-2mm}\begin{array}{cc}H &\hspace{-2mm}0\\
0 & \hspace{-2mm}H\end{array}\hspace{-2mm}\Bigg]\hspace{-1mm}\Bigg[\hspace{-2mm}\begin{array}{cc}
\frac{I_K+\Lambda}{2} &\hspace{-3mm}-\left(\frac{I_K-\Lambda}{2}\right)^{\frac{1}{2}}\\
\left(\frac{I_K+\Lambda}{2}\right)\left(\frac{I_K-\Lambda}{2}\right)^{\frac{1}{2}} & \hspace{-3mm}\frac{I_K+\Lambda}{2}\end{array}\hspace{-2mm}\Bigg]
\hspace{-1mm}\Bigg[\hspace{-2mm}\begin{array}{cc}H &\hspace{-2mm} 0\\
0 & \hspace{-2mm}H\end{array}\hspace{-2mm}\Bigg]^T\\
&\stackrel{(b)}{=}H_e\text{diag}\{E_k\}_{k=1}^KH_e^T\label{sup_eq:block_diag}
\end{align}
where in $(a)$ we used Remark \ref{remark:L} and the definition of $V$ \eqref{eq:defs} and in $(b)$ we rearranged the order of the eigenvectors and their corresponding eigenvalues through permutations to get the following matrices:
\begin{align}
E_k&\hspace{-0.6mm}=\hspace{-1mm}\Bigg[\hspace{-2mm}\begin{array}{cc}\frac{1+\lambda_k}{2} & -\left(\frac{1-\lambda_k}{2}\right)^{\frac{1}{2}}\\
\left(\frac{1+\lambda_k}{2}\right)\left(\frac{1-\lambda_k}{2}\right)^{\frac{1}{2}} & \frac{1+\lambda_k}{2}\end{array}\hspace{-2mm}\Bigg]\\
H_e&\hspace{-0.6mm}=\hspace{-1mm}\bigg[\hspace{-2mm}\begin{array}{ccccccc}
K^{-\frac{1}{2}}\one_{K} & 0 & h_2 & 0 & \cdots& h_K & 0\\
0 & K^{-\frac{1}{2}}\one_{K} & 0 & h_2 & \cdots & 0 & h_K\end{array}\hspace{-2mm}\bigg]
\end{align}
where $\lambda_k$ is the $k$-th eigenvalue of $L$ and $h_k$ its corresponding eigenvector. Note that $E_1=I$. Moreover, the matrices $E_k$ have two distinct eigenvalues given by $(1+\lambda_k\pm\sqrt{1-\lambda_k^2})/2$ and therefore we can diagonalize $E_k$ using its Jordan Canonical Form as $D_k=Z_k^{-1}E_kZ_k$ for some $Z_k$. Therefore, defining $Z=\text{diag}\{Z_k\}_{k=1}^K$ (where $Z_1=I$) we arrive at the following diagonalization:
\begin{align}
	&\bigg[\hspace{-2mm}\begin{array}{cc}
	\bar{L} & \hspace{-2mm}-V\\
	V\bar{L} & \hspace{-2mm}\bar{L}\end{array}\hspace{-2mm}\bigg]=H_eZ\text{diag}\{I,D_2,\cdots,D_K\}Z^{-1}H_e^T
\end{align}
We can extend this diagonalization to the network-wide matrix:
\begin{align}
&\bigg[\hspace{-2mm}\begin{array}{cc}
\mathcal{\bar{L}} & \hspace{-2mm}-\mathcal{V}\\
\mathcal{V}\mathcal{\bar{L}} & \hspace{-2mm}\mathcal{\bar{L}}\end{array}\hspace{-2mm}\bigg]\hspace{-1mm}=\hspace{-1mm}\bigg[\hspace{-2mm}\begin{array}{cc}
\bar{L} & \hspace{-2mm}-V\\
V\bar{L} & \hspace{-2mm}\bar{L}\end{array}\hspace{-2mm}\bigg]\otimes I=\underbrace{(H_eZ\otimes I)}_{\define\mathcal{H}}\mathcal{D}\underbrace{(Z^{-\hspace{-0.5mm}1}\hspace{-0.5mm}H_e^T\hspace{-0.5mm}\otimes I)}_{=\mathcal{H}^{-1}}\\
&\mathcal{D}=\text{diag}\{I,\mathcal{D}_1\},\hspace{2mm}\mathcal{D}_1=\text{diag}\{D_k\}_{k=2}^K\otimes I
\end{align}
Expanding $\mathcal{H}$ and $\mathcal{H}^{-1}$ we get:
\begin{align}
&\mathcal{H}\hspace{-1mm}=\hspace{-1mm}\bigg[\hspace{-2mm}\begin{array}{ccc}K^{-\frac{1}{2}}\mathcal{I} & \hspace{-4mm}0 & \hspace{-1mm}\mathcal{H}_{u}\\ 0 & \hspace{-4mm}K^{-\frac{1}{2}}\mathcal{I} & \hspace{-1mm}\mathcal{H}_{d}\end{array}\hspace{-2mm}\bigg],\hspace{2mm}
\mathcal{H}^{\hspace{-0.5mm}-\hspace{-0.5mm}1}\hspace{-1mm}=\hspace{-1mm}\bigg[\hspace{-2mm}\begin{array}{ccc}K^{-\frac{1}{2}}\mathcal{I} & \hspace{-4mm}0 & \hspace{-1mm}\mathcal{H}_{l}\\ 0 & \hspace{-4mm}K^{-\frac{1}{2}}\mathcal{I} & \hspace{-1mm}\mathcal{H}_{r}\end{array}\hspace{-2mm}\bigg]^T
\end{align}
where $\mathcal{H}_{l},\mathcal{H}_{r},\mathcal{H}_{u},\mathcal{H}_{d}\in\mathds{R}^{2KM\times4KM-4M}$ are some constant matrices. Now we use this decomposition to transform recursion \eqref{eq:error_rec_1}:
\begin{align}\label{sup_eq:coordinate_transform}
\underbrace{\mathcal{H}^{-1}\bigg[\hspace{-2mm}\begin{array}{c}
\widetilde{\zeta}_{i+1}\\
\widetilde{\mathcal{Y}}_{i+1}
\end{array}\hspace{-2mm}\bigg]}_{\define[\bar{x}_{i+1}^T,\tilde{x}_{i+1}^T,\hat{x}_{i+1}^T]^T}\hspace{-4mm}=&
\bigg(\mathcal{D}-
\mu\mathcal{H}^{-1}\underbrace{\bigg[\hspace{-2mm}\begin{array}{cc}
	\bar{\mathcal{L}}\mathcal{G}& 0\\
	\mathcal{V}\bar{\mathcal{L}}\mathcal{G} & 0 \end{array}\hspace{-2mm}\bigg]}_{\define\mathcal{T}}\mathcal{H}\bigg)
\mathcal{H}^{-1}\bigg[\hspace{-2mm}\begin{array}{c}
\widetilde{\zeta}_{i}\\
\widetilde{\mathcal{Y}}_{i}\end{array}\hspace{-2mm}\bigg]
\end{align}
where $\bar{x}_{i}=K^{\frac{1}{2}}\mathcal{I}^T\widetilde{\zeta}_{i}$ and $\hat{x}_{i}=\mathcal{H}_l^T\widetilde{\zeta}_{i}+\mathcal{H}_r^T\widetilde{\mathcal{Y}}_{i}$. Using $\mathcal{H}$ we can further write $\widetilde{\zeta}_{i}=K^{\frac{1}{2}}\mathcal{I}\bar{x}_{i}+\mathcal{H}_u\hat{x}_{i}$ and $\widetilde{\mathcal{Y}}_{i}=\mathcal{H}_d\hat{x}_{i}$, from which we get the following bounds:
\begin{align}
\|\widetilde{\zeta}_{i}\|^2&\leq \|\bar{x}_{i}\|^2+\|\mathcal{H}_{u}\|^2\|\hat{x}_{i}\|^2,\hspace{5mm}\|\widetilde{\mathcal{Y}}_{i}\|^2\leq\|\mathcal{H}_{d}\|^2\|\hat{x}_{i}\|^2
\end{align}
Expanding $\mathcal{H}^{-1}\mathcal{T}\mathcal{H}$ we get:
\begin{align}
\mathcal{H}^{-1}\mathcal{T}\mathcal{H}&\hspace{-1mm}=\hspace{-1mm}\Bigg[\hspace{-2mm}\begin{array}{ccc}
K^{-1}\mathcal{I}^T\bar{\mathcal{L}}\mathcal{G}\mathcal{I} & 0 & 
K^{-\frac{1}{2}}\mathcal{I}^T\bar{\mathcal{L}}\mathcal{G}\mathcal{H}_{u}\\
K^{-\frac{1}{2}}\mathcal{I}^T\mathcal{V}\bar{\mathcal{L}}\mathcal{G}\mathcal{I} & 0 & 
K^{-\frac{1}{2}}\mathcal{I}^T\mathcal{V}\bar{\mathcal{L}}\mathcal{G}\mathcal{H}_{u}\\
K^{-\frac{1}{2}}(\mathcal{H}_{l}^T+\mathcal{H}_{r}^T\mathcal{V})\bar{\mathcal{L}}\mathcal{G}\mathcal{I} & 0 & (\mathcal{H}_{l}^T+\mathcal{H}_{r}^T\mathcal{V})\bar{\mathcal{L}}\mathcal{G}\mathcal{H}_{u}\end{array}\hspace{-2mm}\Bigg]
\end{align}
Since $\mathcal{I}^T\mathcal{V}=0$ all elements in the mid row in the above equation are equal to zero and therefore $\tilde{x}_i=\tilde{x}_0$. Also since all the elements in the mid column are equal to zero, it turns out that $\bar{x}_i$ and $\widehat{x}_i$ are independent of $\tilde{x}_0$. Hence, we can disregard variable $\tilde{x}_i$ and arrive at the desired recursion, which completes the proof.}

\textcolor{black}{\section{Proof of Lemma \ref{lemma:squared_recursion}}\label{app:squared_recursion}
We start by expanding the top recursion of \eqref{eq:diag_rec} and take the squared norm on both sides of the equation to get:
\begin{align}
&\|\check{x}_{i+1}\|^2=\|\left(I_{2M}-\mu \Lambda_GK^{-1}\right)\check{x}_{i}-\mu K^{-\frac{1}{2}}Z^{-1}\mathcal{I}^T\mathcal{G}\mathcal{H}_{u}\hat{x}_{i}\|^2\nonumber\\
&\|\check{x}_{i+1}\|^2=\left\|\frac{t}{t}\left(I_{2M}-\mu \Lambda_GK^{-1}\right)\check{x}_{i}-\frac{(1-t)\mu Z^{-1}\mathcal{I}^T\mathcal{G}\mathcal{H}_{u}}{(1-t)\sqrt{K}}\hat{x}_{i}\right\|^2\nonumber\\
&\|\check{x}_{i+1}\|^2\hspace{-1mm}\stackrel{(a)}{\leq}t^{-1}\big\|\big(I_{2M}\hspace{-0.5mm}-\hspace{-0.5mm}\mu \Lambda_GK^{-1}\big)\check{x}_{i}\big\|^2+\mu^2\frac{\|Z^{-1}\mathcal{I}^T\hspace{-1mm}\mathcal{G}\mathcal{H}_{u}\hat{x}_{i}\|^2}{(1-t)K}\nonumber\\
&\|\check{x}_{i+1}\|^2\hspace{-1mm}\stackrel{(b)}{\leq}\hspace{-1mm}t^{\hspace{-0.5mm}-\hspace{-0.5mm}1}\big\|I_{2M}\hspace{-0.5mm}-\hspace{-0.5mm}\mu \Lambda_GK^{\hspace{-0.5mm}-\hspace{-0.5mm}1}\big\|^2\|\check{x}_{i}\|^{\hspace{-0.5mm}2}+\frac{\mu^2\|Z^{\hspace{-0.5mm}-\hspace{-0.5mm}1}\mathcal{I}^T\hspace{-1mm}\mathcal{G}\mathcal{H}_{u}\hspace{-0.5mm}\|^2}{(1-t)K}\|\hat{x}_{i}^e\|^2
\end{align}
where we defined a scalar $t\in(0,1)$, in $(a)$ we used Jensen's inequality and in $(b)$ we used the Cauchy-Schwarz inequality. If $\mu<K/\rho(\Lambda_G)$ (where $\rho(\Lambda_G)$ is the spectral radius of $\Lambda_G$) then $\big\|I_{2M}\hspace{-0.5mm}-\hspace{-0.5mm}\mu \Lambda_GK^{-1}\big\|<1$ and hence setting $t=\big\|I_{2M}\hspace{-0.5mm}-\hspace{-0.5mm}\mu \Lambda_GK^{-1}\big\|$ we get:
\begin{align}
\|\check{x}_{i+1}\|^2&\leq\rho(I_{2M}-\mu K^{-1}\Lambda_G)\|\check{x}_{i}\|^2+\mu a_2\|\hat{x}_{i}\|^2\label{eq:in_1}
\end{align}
where we defined $a_2=\|Z^{-1}\mathcal{I}^T\hspace{-1mm}\mathcal{G}\mathcal{H}_{u}\|^2/\lambda_{\text{min}}(\Lambda_G)$. We now repeat the procedure for the bottom recursion of \eqref{eq:diag_rec}:
\begin{align}
	\|\hat{x}_{i+1}\|^2\hspace{-1mm}&\stackrel{(d)}{\leq}\hspace{-1mm}\frac{2\mu^2}{(1-t_2)K}\|(\mathcal{H}_{l}^T\hspace{-1mm}+\hspace{-0.5mm}\mathcal{H}_{r}^T\mathcal{V})\bar{\mathcal{L}}\mathcal{G}\mathcal{I}Z\|^{\hspace{-0.3mm}2}\hspace{-0.5mm}\|\check{x}_{i}^e\|^2\hspace{-1mm}+\hspace{-0.5mm}\frac{\|\mathcal{D}_1\|^2}{t_2}\|\hat{x}_{i}^e\|^2\nonumber\\
	&\hspace{5mm}+2(1-t_2)^{-1}\mu^2\|(\mathcal{H}_{l}^T+\mathcal{H}_{r}^T\mathcal{V})\bar{\mathcal{L}}\mathcal{G}\mathcal{H}_{u}\|^2\|\hat{x}_{i}^e\|^2\nonumber\\
	&\stackrel{(f)}{\leq}\mu^2a_3\|\check{x}_{i}^e\|^2+\sqrt{\lambda_2(\bar{L})}\|\hat{x}_{i}^e\|^2+\mu^2a_4\|\hat{x}_{i}^e\|^2\label{eq:in_2}
\end{align}
where in $(d)$ we used Jensen's and the Cauchy-Schwarz inequalities and we also introduced $t_2\in(0,1)$ and in $(f)$ we chose $t_2=\|\mathcal{D}_1\|=\sqrt{\lambda_2(\bar{L})}$. We further defined:
\begin{align}
	a_3&=\frac{2\|(\mathcal{H}_{l}^T+\mathcal{H}_{r}^T\mathcal{V})\bar{\mathcal{L}}\mathcal{G}\mathcal{I}Z\|^2}{1-\sqrt{\lambda_2(\bar{L})}},\hspace{2mm}a_4=\frac{2\|(\mathcal{H}_{l}^T+\mathcal{H}_{r}^T\mathcal{V})\bar{\mathcal{L}}\mathcal{G}\mathcal{H}_{u}\|^2}{1-\sqrt{\lambda_2(\bar{L})}}\nonumber
\end{align}
Writing \eqref{eq:in_1} and \eqref{eq:in_2} in matrix form completes the proof.}

\ifCLASSOPTIONcaptionsoff
  \newpage
\fi



\bibliographystyle{IEEEtran}
\bibliography{multi_pol_eval}
%\begin{thebibliography}{1}
%\end{thebibliography}

%% biography section
%% 
%% If you have an EPS/PDF photo (graphicx package needed) extra braces are
%% needed around the contents of the optional argument to biography to prevent
%% the LaTeX parser from getting confused when it sees the complicated
%% \includegraphics command within an optional argument. (You could create
%% your own custom macro containing the \includegraphics command to make things
%% simpler here.)

\vspace{-1mm}\begin{IEEEbiographynophoto}{Lucas Cassano} biography not available at the time of publication.
\end{IEEEbiographynophoto}
\vspace{-10mm}\begin{IEEEbiographynophoto}{Kun Yuan} biography not available at the time of publication.
\end{IEEEbiographynophoto}
\vspace{-10mm}\begin{IEEEbiographynophoto}{Ali H. Sayed} biography not available at the time of publication.
\end{IEEEbiographynophoto}

%\newpage
%\hspace{1cm}
%
%% You can push biographies down or up by placing
%% a \vfill before or after them. The appropriate
%% use of \vfill depends on what kind of text is
%% on the last page and whether or not the columns
%% are being equalized.
%
\vfill
%
%% Can be used to pull up biographies so that the bottom of the last one
%% is flush with the other column.
%%\enlargethispage{-5in}



\section*{Supplementary material}

# Supplementary Material

## I. PROOF OF CONVERGENCE

We start by making a simplification in terms of notation for the sake of clarity. Since it is clear that at this point we are dealing with the empirical problem (42) we will drop the *hat* and refer to $\widehat{A}_{k,l}$, $\widehat{b}_{k,l}$ and $\widehat{C}_{k,l}$ as $A_{k,l}$, $b_{k,l}$ and $C_{k,l}$ respectively. We now rewrite the equations the *FDPE* Algorithm as a unique network recursion as follows:

$$\boldsymbol{\zeta}_{i+1}^e = \bar{\mathcal{L}}\big(\boldsymbol{\zeta}_0^0 - \mu(\boldsymbol{\mathcal{G}}_0^0\boldsymbol{\zeta}_0^0 - \boldsymbol{t}_0^0)\big), \qquad\qquad\qquad \text{for } i{=}0 \text{ and } e{=}0 \qquad \text{(S.1a)}$$

$$\boldsymbol{\zeta}_{i+1}^e = \bar{\mathcal{L}}\big(2\boldsymbol{\zeta}_i^e - \boldsymbol{\zeta}_{i-1}^e - \mu(\boldsymbol{\mathcal{G}}_i^e\boldsymbol{\zeta}_i^e - \boldsymbol{\mathcal{G}}_{i-1}^e\boldsymbol{\zeta}_{i-1}^e - \boldsymbol{t}_i^0 + \boldsymbol{t}_{i-1}^0)\big), \qquad \text{for } i{>}0 \text{ and } e{=}0 \qquad \text{(S.1b)}$$

$$\boldsymbol{\zeta}_{i+1}^e = \bar{\mathcal{L}}\big(2\boldsymbol{\zeta}_i^e - \boldsymbol{\zeta}_{i-1}^e - \mu(\boldsymbol{\mathcal{T}}(\boldsymbol{\zeta}_i^e) - \boldsymbol{\mathcal{G}}_{J-1}^0\boldsymbol{\zeta}_{J-1}^0 - p + \boldsymbol{t}_{J-1}^0)\big), \qquad \text{for } i{=}0 \text{ and } e{=}1 \qquad \text{(S.1c)}$$

$$\boldsymbol{\zeta}_{i+1}^e = \bar{\mathcal{L}}\big(2\boldsymbol{\zeta}_i^e - \boldsymbol{\zeta}_{i-1}^e - \mu(\boldsymbol{\mathcal{T}}(\boldsymbol{\zeta}_i^e) - \boldsymbol{\mathcal{T}}(\boldsymbol{\zeta}_{i-1}^e))\big), \qquad\qquad\qquad \text{else} \qquad \text{(S.1d)}$$

where we defined $v$, $\mu$, $\bar{L}$ and $\bar{\mathcal{L}}$ are defined in tha main document and we further define:

$$\boldsymbol{\zeta}_{k,i}^e \triangleq \begin{bmatrix} \boldsymbol{\theta}_{k,i}^e \\ \frac{1}{\sqrt{v}}\boldsymbol{\omega}_{k,i}^e \end{bmatrix}, \qquad \boldsymbol{\phi}_{k,i}^e \triangleq \begin{bmatrix} \boldsymbol{\phi}_{k,i}^{\theta,e} \\ \frac{1}{\sqrt{v}}\boldsymbol{\phi}_{k,i}^{\omega,e} \end{bmatrix}, \qquad \boldsymbol{\psi}_{k,i}^e \triangleq \begin{bmatrix} \boldsymbol{\psi}_{k,i}^{\theta,e} \\ \frac{1}{\sqrt{v}}\boldsymbol{\psi}_{k,i}^{\omega,e} \end{bmatrix} \qquad \text{(S.2a)}$$

$$G_k \triangleq \tau_k \begin{bmatrix} \eta U_k & -\sqrt{v}A_k^T \\ \sqrt{v}A_k & vC_k \end{bmatrix}, \qquad G \triangleq \sum_{k=1}^K G_k, \qquad p_k \triangleq \tau_k \begin{bmatrix} \eta U_k\theta_{\mathrm{p}} \\ \sqrt{v}b_k \end{bmatrix} \qquad \text{(S.2b)}$$

$$\boldsymbol{t}_{k,i}^e \triangleq \tau_k \begin{bmatrix} U_{k,\boldsymbol{\sigma}_k^e(i)}\theta_{\mathrm{prior}} \\ \sqrt{v}b_{k,\boldsymbol{\sigma}_k^e(i)} \end{bmatrix}, \qquad \boldsymbol{G}_{k,i}^e \triangleq \tau_k \begin{bmatrix} \eta U_{k,\boldsymbol{\sigma}_k^e(i)} & -\sqrt{v}A_{k,\boldsymbol{\sigma}_k^e(i)}^T \\ \sqrt{v}A_{k,\boldsymbol{\sigma}_k^e(i)} & vC_{k,\boldsymbol{\sigma}_k^e(i)} \end{bmatrix} \qquad \text{(S.2c)}$$

$$\boldsymbol{T}(\boldsymbol{\zeta}_{k,i}^e) \triangleq \boldsymbol{G}_{k,i}^e(\boldsymbol{\zeta}_{k,i}^e - \boldsymbol{\zeta}_{k,0}^e) + \frac{1}{J}\sum_{n=1}^J \boldsymbol{G}_{k,n}^{e-1}\boldsymbol{\zeta}_{k,n}^{e-1} \qquad \text{(S.2d)}$$

$$\boldsymbol{\zeta}_i^e \triangleq \mathrm{col}\{\boldsymbol{\zeta}_{1,i}^e,...,\boldsymbol{\zeta}_{K,i}^e\}, \qquad \boldsymbol{\phi}_i^e \triangleq \mathrm{col}\{\boldsymbol{\phi}_{1,i}^e,...,\boldsymbol{\phi}_{K,i}^e\}, \qquad \boldsymbol{\psi}_i^e \triangleq \mathrm{col}\{\boldsymbol{\psi}_{1,i}^e,...,\boldsymbol{\psi}_{K,i}^e\} \qquad \text{(S.2e)}$$

$$\mathcal{G} \triangleq \mathrm{diag}\{G_1,\cdots,G_K\}, \qquad p \triangleq \mathrm{col}\{p_1,p_2,\cdots,p_K\} \qquad \text{(S.2f)}$$

$$\boldsymbol{\mathcal{G}}_i^e \triangleq \mathrm{diag}\{\boldsymbol{G}_{1,i}^e,\cdots,\boldsymbol{G}_{K,i}^e\}, \qquad \boldsymbol{\mathcal{T}}(\boldsymbol{\zeta}_i^e) \triangleq \mathrm{col}\{\boldsymbol{T}(\boldsymbol{\zeta}_{1,i}^e),\cdots,\boldsymbol{T}(\boldsymbol{\zeta}_{K,i}^e)\}, \qquad \boldsymbol{t}_i^e \triangleq \mathrm{col}\{\boldsymbol{t}_{1,i}^e,\boldsymbol{t}_{2,i}^e,\cdots,\boldsymbol{t}_{K,i}^e\} \qquad \text{(S.2g)}$$

we further clarify that the notation $A_{k,\boldsymbol{\sigma}_k^e(i)}$ refers to the $k$-th agent's estimate of $A$ using all the samples from the $\boldsymbol{\sigma}_k^e(i)$-th mini-batch (if no mini-batches are being used, $A_{k,\boldsymbol{\sigma}_k^e(i)}$ is just a sample estimate, otherwise it's the empirical average using the samples corresponding to the $\boldsymbol{\sigma}_k^e(i)$-th mini-batch). There are two reasons why we are getting four different update equations in (S.1). In the first place, during the very first iteration (i.e., $e{=}0$ and $i{=}0$) there's no variance reduction or correction step, which accounts for (S.1a). In the rest of the iterations within the first epoch (i.e., (S.1b)) there is a correction step, however there is still no variance reduction. Update equation (S.1c) reflects the fact that during the first iterate of the second epoch the old gradients in the correction step don't have variance reduction (yet the new ones do). Equation (S.1d) is for the rest of the recursions where both the correction step and variance reduction are present. Note that (S.1) describes a second order recursion. From here we take steps to rewrite the update equations as a first order recursion driven by gradient noise.

**Lemma 1.** *If $\boldsymbol{\mathcal{Y}}_0^0$ is initialized to $\boldsymbol{\mathcal{Y}}_0^0{=}0$, then recursion* (S.1) *is equivalent to:*

$$\boldsymbol{\zeta}_{i+1}^e = \begin{cases} \bar{\mathcal{L}}(\boldsymbol{\zeta}_i^e - \mu(\boldsymbol{\mathcal{G}}\boldsymbol{\zeta}_i^e - \boldsymbol{t}_i^e)) - \mathcal{V}\boldsymbol{\mathcal{Y}}_i^e, & \text{for } e{=}0 \\ \bar{\mathcal{L}}(\boldsymbol{\zeta}_i^e - \mu(\boldsymbol{\mathcal{T}}(\boldsymbol{\zeta}_i^e) - p)) - \mathcal{V}\boldsymbol{\mathcal{Y}}_i^e, & \text{else} \end{cases} \qquad \text{(S.3a)}$$

$$\boldsymbol{\mathcal{Y}}_{i+1}^e = \boldsymbol{\mathcal{Y}}_i^e + \mathcal{V}\boldsymbol{\zeta}_{i+1}^e \tag{S.3b}$$

*Proof.* See subsection I-A. ∎

Note that (S.3a) is analogous to equation (47) from the main document, with the difference that in this case the gradients are subject to gradient noise. The first line in (S.3a) accounts only for a finite number of updates. Therefore, to analyze the convergence properties of the algorithm we only need to focus on the second line of (S.3a). In other words, for the purpose of proving convergence we will only consider updates for $e > 0$.

Now we define the following error quantities and gradient noise:

$$\widetilde{\boldsymbol{\zeta}}_{i+1}^e = \zeta^o - \boldsymbol{\zeta}_{i+1}^e \tag{S.4}$$

$$\widetilde{\boldsymbol{\mathcal{Y}}}_{i+1}^e = \mathcal{Y}^o - \boldsymbol{\mathcal{Y}}_{i+1}^e \tag{S.5}$$

$$\boldsymbol{s}(\zeta_i^e) = \boldsymbol{\mathcal{T}}(\zeta_i^e) - \mathcal{G}\boldsymbol{\zeta}_i^e \tag{S.6}$$

Following steps (48), (49) and (50) from the main document we get:

$$\begin{bmatrix} \widetilde{\boldsymbol{\zeta}}_{i+1}^e \\ \widetilde{\boldsymbol{\mathcal{Y}}}_{i+1}^e \end{bmatrix} = \begin{bmatrix} \bar{\mathcal{L}}(I - \mu\mathcal{G}) & -K\mathcal{V} \\ \mathcal{V}\bar{\mathcal{L}}(I - \mu\mathcal{G}) & \bar{\mathcal{L}} \end{bmatrix} \begin{bmatrix} \widetilde{\boldsymbol{\zeta}}_i^e \\ \widetilde{\boldsymbol{\mathcal{Y}}}_i^e \end{bmatrix} + \mu \begin{bmatrix} \bar{\mathcal{L}}\boldsymbol{s}(\zeta_i^e) \\ \mathcal{V}\bar{\mathcal{L}}\boldsymbol{s}(\zeta_i^e) \end{bmatrix} \tag{S.7}$$

Note that (S.7) has the same form as (50) from the main document plus a term due to the gradient noise. Since the driving matrix in (S.7) has a structure that makes it difficult to calculate its eigenstructure. We circumvent this issue by doing a coordinate transformation as the following lemma indicates.

**Lemma 2.** *Through a coordinate transformation, recursion* (S.7) *can be transformed to obtain the following recursion:*

$$\begin{bmatrix} \bar{\boldsymbol{x}}_{i+1}^e \\ \hat{\boldsymbol{x}}_{i+1}^e \end{bmatrix} = \begin{bmatrix} I_{2M} - \mu GK^{-1} & -K^{-1/2}\mu\mathcal{I}^T\mathcal{G}\mathcal{H}_u \\ -\mu K^{-1/2}(\mathcal{H}_l^T + \mathcal{H}_r^T\mathcal{V})\bar{\mathcal{L}}\mathcal{G}\mathcal{I} & \mathcal{D}_1 - \mu(\mathcal{H}_l^T + \mathcal{H}_r^T\mathcal{V})\bar{\mathcal{L}}\mathcal{G}\mathcal{H}_u \end{bmatrix} \begin{bmatrix} \bar{\boldsymbol{x}}_i^e \\ \hat{\boldsymbol{x}}_i^e \end{bmatrix} + \mu \begin{bmatrix} \mathcal{I}^T K^{-1/2} \\ (\mathcal{H}_l^T + \mathcal{H}_r^T\mathcal{V})\bar{\mathcal{L}} \end{bmatrix} \boldsymbol{s}(\zeta_i^e) \tag{S.8}$$

*where* $\mathcal{H}_l, \mathcal{H}_r, \mathcal{H}_u, \mathcal{H}_d \in \mathbb{R}^{2KM \times 4KM - 4M}$ *are some constant matrices and* $\mathcal{D}_1$ *is a diagonal matrix with* $\|\mathcal{D}_1\|_2^2 = \lambda_2(\bar{L}) < 1$. *And also we have:*

$$\bar{\boldsymbol{x}}_i^e \overset{\Delta}{=} K^{-1/2}\mathcal{I}^T\widetilde{\boldsymbol{\zeta}}_i^e \tag{S.9a}$$

$$\hat{\boldsymbol{x}}_i^e \overset{\Delta}{=} \mathcal{H}_l^T\widetilde{\boldsymbol{\zeta}}_i^e + \mathcal{H}_r^T\widetilde{\boldsymbol{\mathcal{Y}}}_i^e \tag{S.9b}$$

$$\mathcal{I} \overset{\Delta}{=} \mathbb{1}_K \otimes I_{2M} \tag{S.9c}$$

*Proof.* See subsection I-B. ∎

Due to Theorem 2.1 from [1], if Assumption 4 and the following condition are satisfied:

$$\upsilon\lambda_{\min}(C) > \eta\lambda_{\max}(U) + 2\sqrt{\upsilon\lambda_{\min}(C)\lambda_{\max}(AC^{-1}A^T)} \tag{S.10}$$

then matrix $G$ is diagonalizable with strictly positive eigenvalues. Hence, we can write $G = Z\Lambda_G Z^{-1}$. Therefore, defining $\check{\boldsymbol{x}}_i = Z^{-1}\bar{\boldsymbol{x}}_i$ we can transform (S.8) into:

$$\begin{bmatrix} \check{\boldsymbol{x}}_{i+1}^e \\ \hat{\boldsymbol{x}}_{i+1}^e \end{bmatrix} = \begin{bmatrix} I_{2M} - \mu K^{-1}\Lambda_G & -\mu K^{-\frac{1}{2}}Z^{-1}\mathcal{I}^T\mathcal{G}\mathcal{H}_u \\ \frac{-\mu}{\sqrt{K}}(\mathcal{H}_l^T + \mathcal{H}_r^T\mathcal{V})\bar{\mathcal{L}}\mathcal{G}\mathcal{I}Z & \mathcal{D}_1 - \mu(\mathcal{H}_l^T + \mathcal{H}_r^T\mathcal{V})\bar{\mathcal{L}}\mathcal{G}\mathcal{H}_u \end{bmatrix} \begin{bmatrix} \check{\boldsymbol{x}}_i^e \\ \hat{\boldsymbol{x}}_i^e \end{bmatrix} + \mu \begin{bmatrix} Z^{-1}\mathcal{I}^T K^{-1/2} \\ (\mathcal{H}_l^T + \mathcal{H}_r^T\mathcal{V})\bar{\mathcal{L}} \end{bmatrix} \boldsymbol{s}(\zeta_i^e) \tag{S.11}$$

**Lemma 3.** *If* $\mu \le K/\rho(\Lambda_G)$ *then the following inequality holds:*

$$\begin{bmatrix} \|\check{\boldsymbol{x}}_{i+1}^e\|^2 \\ \|\hat{\boldsymbol{x}}_{i+1}^e\|^2 \end{bmatrix} \preceq \begin{bmatrix} 1 - \mu a_1 K^{-1} & \mu a_2 \\ \mu^2 a_6 & \sqrt{\lambda_2(\bar{L})} + \mu^2 a_5 \end{bmatrix} \begin{bmatrix} \|\check{\boldsymbol{x}}_i^e\|^2 \\ \|\hat{\boldsymbol{x}}_i^e\|^2 \end{bmatrix} + \mu \begin{bmatrix} a_3 & \|\mathcal{H}_u\|^2 a_3 \\ \mu a_7 & \mu\|\mathcal{H}_u\|^2 a_7 \end{bmatrix} \begin{bmatrix} \|\check{\boldsymbol{x}}_i^e - \bar{\boldsymbol{x}}_0^e\|^2 \\ \|\hat{\boldsymbol{x}}_i^e - \hat{\boldsymbol{x}}_0^e\|^2 \end{bmatrix}$$

$$+\mu\begin{bmatrix} a_4 & \|\mathcal{H}_u\|^2 a_4 \\ \mu a_8 & \mu\|\mathcal{H}_u\|^2 a_8 \end{bmatrix}\frac{1}{J}\sum_{n=1}^{J}\begin{bmatrix} \|\bar{\boldsymbol{x}}_J^{e-1}-\bar{\boldsymbol{x}}_{n-1}^{e-1}\|^2 \\ \|\hat{\boldsymbol{x}}_J^{e-1}-\hat{\boldsymbol{x}}_{n-1}^{e-1}\|^2 \end{bmatrix} \tag{S.12}$$

*where $a_1$ is the smallest eigenvalue of $G$ and $a_2$, $a_3$, $a_4$, $a_5$, $a_6$, $a_7$ and $a_8$ are positive constants.*

*Proof.* See Appendix I-C. ■

Note that (S.12) is still not sufficient to prove convergence since we have no bounds for the evolution of the norms of differences (for instance $\|\bar{\boldsymbol{x}}_i^e-\bar{\boldsymbol{x}}_0^e\|$). The following lemma solves this inconvenience.

**Lemma 4.** *There exists a $\mu_0$ such that if the step-size $\mu$ satisfies $\mu<\mu_0$ then it holds:*

$$\begin{bmatrix} \|\bar{\boldsymbol{x}}_0^{e+1}\|^2 \\ \|\hat{\boldsymbol{x}}_0^{e+1}\|^2 \\ w^e \end{bmatrix}\preceq\begin{bmatrix} M_1^J+\mu f_1 JM_2M_4(I-M_1)^{-1} & \mu J\begin{bmatrix} a_4+\mu^2 f_2 a_3 \\ \mu a_8+\mu^3 f_2 a_7 \end{bmatrix} \\ f_1\begin{bmatrix} 1 & \|\mathcal{H}_u\|^2 \end{bmatrix}M_4(I-M_1)^{-1} & \mu^2 f_2 \end{bmatrix}\begin{bmatrix} \|\bar{\boldsymbol{x}}_0^e\|^2 \\ \|\hat{\boldsymbol{x}}_0^e\|^2 \\ w^{e-1} \end{bmatrix} \tag{S.13}$$

*where $f_1$ and $f_2$ are positive scalars and we also defined:*

$$w^e=\frac{1}{N}\sum_{n=0}^{J-1}\big(\|\bar{\boldsymbol{x}}_J^e-\bar{\boldsymbol{x}}_j^e\|^2+\|\mathcal{H}_u\|^2\|\hat{\boldsymbol{x}}_J^e-\hat{\boldsymbol{x}}_j^e\|^2\big) \tag{S.14a}$$

$$M_1=\begin{bmatrix} 1-\mu K^{-1}a_1 & \mu a_2 \\ \mu^2 a_6 & \sqrt{\lambda_2(\overline{L})}+\mu^2 a_5 \end{bmatrix} \tag{S.14b}$$

$$M_2=\begin{bmatrix} a_3 & \|\mathcal{H}_u\|^2 a_3 \\ \mu a_7 & \mu\|\mathcal{H}_u\|^2 a_7 \end{bmatrix} \tag{S.14c}$$

$$M_4=\begin{bmatrix} 3\mu^2 K^{-2}\|\Lambda_G\|^2 & 3\mu^2 K^{-1}\|Z^{-1}\mathcal{I}^T\mathcal{G}\mathcal{H}_u\|^2 \\ \mu^2 a_6 & 1-\sqrt{\lambda_2(\overline{L})}+\mu^2 a_5 \end{bmatrix} \tag{S.14d}$$

*Proof.* See Appendix I-E. ■

**Theorem 1.** *If the step-size $\mu$ is small enough then the iterates $\begin{bmatrix} \|\bar{\boldsymbol{x}}_0^{e+1}\|^2 & \|\hat{\boldsymbol{x}}_0^{e+1}\|^2 & w^e \end{bmatrix}$ and $\boldsymbol{\omega}_{k,0}^e$ generated by (S.13) converge linearly to $\begin{bmatrix} 0 & 0 & 0 \end{bmatrix}$, with a rate upper bounded by $\max[\rho(I-\mu K^{-1}\Lambda_G)+\mathcal{O}(\mu^2),\sqrt{\lambda_2(\overline{L})}+\mathcal{O}(\mu)]$.*

*Proof.* See Appendix I-G. ■

Finally definitions (S.9a), (S.4), (S.5) and the definitions of $\boldsymbol{\zeta}_{k,i}^e$ and $\boldsymbol{\zeta}_i^e$ in (S.2) combined with Theorem 1 imply that the iterates $\boldsymbol{\theta}_{k,0}^e$ and $\boldsymbol{\omega}_{k,0}^e$ generated by *Fast Diffusion for Policy Evaluation* converge linearly to the saddle point of problem (42) for every agent $k$; which completes the proof.

### A. Proof of Lemma 1.

We initialize $\boldsymbol{\zeta}_0^0$ in (S.3) to the same value used in (S.1) and $\boldsymbol{\mathcal{Y}}_0^0=0$, hence for $i=0$ and $e=0$ the claim is trivially true. Replacing the initial conditions in (S.1) and (S.3) for $e=0$ and $i=1$ we get in both cases the same update, which is:

$$\boldsymbol{\zeta}_1^0=\bar{\mathcal{L}}\big(\boldsymbol{\zeta}_0^0-\mu\mathcal{Q}\big(\mathcal{G}\boldsymbol{\zeta}_0^0-\boldsymbol{t}_0^0\big)\big) \tag{S.15}$$

which proves the claim for $e=0$ and $i=1$. Now we prove the claim for $e=0$ and $i>1$

$$\begin{aligned}\boldsymbol{\zeta}_{i+1}^e&=\bar{\mathcal{L}}\big(\boldsymbol{\zeta}_i^e-\mu\mathcal{Q}\big(\mathcal{G}\boldsymbol{\zeta}_i^e-\boldsymbol{t}_i^e\big)\big)-KV\boldsymbol{\mathcal{Y}}_i^e \\ &=\bar{\mathcal{L}}\big(\boldsymbol{\zeta}_i^e-\mu\mathcal{Q}\big(\mathcal{G}\boldsymbol{\zeta}_i^e-\boldsymbol{t}_i^e\big)\big)-KV\boldsymbol{\mathcal{Y}}_i^e+\boldsymbol{\zeta}_i^e-\boldsymbol{\zeta}_i^e \\ &=\bar{\mathcal{L}}\big(\boldsymbol{\zeta}_i^e-\boldsymbol{\zeta}_{i-1}^e-\mu\mathcal{Q}\big(\mathcal{G}\boldsymbol{\zeta}_i^e-\mathcal{G}\boldsymbol{\zeta}_{i-1}^e-\boldsymbol{t}_i^e+\boldsymbol{t}_{i-1}^e\big)\big)-KV\big(\boldsymbol{\mathcal{Y}}_i^e-\boldsymbol{\mathcal{Y}}_{i-1}^e\big)+\boldsymbol{\zeta}_i^e\end{aligned}$$

$$\begin{aligned}
&=\bar{\mathcal{L}}\big(2\boldsymbol{\zeta}_i^e-\boldsymbol{\zeta}_{i-1}^e-\mu\mathcal{Q}\big(\boldsymbol{\mathcal{G}}\boldsymbol{\zeta}_i^e-\boldsymbol{\mathcal{G}}\boldsymbol{\zeta}_{i-1}^e-\boldsymbol{t}_i^e+\boldsymbol{t}_{i-1}^e\big)\big)-K\mathcal{V}(\boldsymbol{\mathcal{Y}}_i^e-\boldsymbol{\mathcal{Y}}_{i-1}^e)+(I-\bar{\mathcal{L}})\boldsymbol{\zeta}_i^e\\
&\overset{(a)}{=}\bar{\mathcal{L}}\big(2\boldsymbol{\zeta}_i^e-\boldsymbol{\zeta}_{i-1}^e-\mu\mathcal{Q}\big(\boldsymbol{\mathcal{G}}\boldsymbol{\zeta}_i^e-\boldsymbol{\mathcal{G}}\boldsymbol{\zeta}_{i-1}^e-\boldsymbol{t}_i^e+\boldsymbol{t}_{i-1}^e\big)\big)-K\mathcal{V}(\boldsymbol{\mathcal{Y}}_i^e-\boldsymbol{\mathcal{Y}}_{i-1}^e)+K\mathcal{V}^2\boldsymbol{\zeta}_i^e\\
&\overset{(b)}{=}\bar{\mathcal{L}}\big(2\boldsymbol{\zeta}_i^e-\boldsymbol{\zeta}_{i-1}^e-\mu\mathcal{Q}\big(\boldsymbol{\mathcal{G}}\boldsymbol{\zeta}_i^e-\boldsymbol{\mathcal{G}}\boldsymbol{\zeta}_{i-1}^e-\boldsymbol{t}_i^e+\boldsymbol{t}_{i-1}^e\big)\big)
\end{aligned}$$
(S.16)

where in $(a)$ we used the definition of $\mathcal{V}$ and in $(b)$ we used (S.3b). For $\boldsymbol{\zeta}_1^1$ we have

$$\begin{aligned}
\boldsymbol{\zeta}_1^1&=\bar{\mathcal{L}}\big(\boldsymbol{\zeta}_0^1-\mu\mathcal{Q}\big(\boldsymbol{\mathcal{T}}(\boldsymbol{\zeta}_0^1)-p\big)\big)-K\mathcal{V}\boldsymbol{\mathcal{Y}}_0^1\\
&=\bar{\mathcal{L}}\big(\boldsymbol{\zeta}_0^1-\mu\mathcal{Q}\big(\boldsymbol{\mathcal{T}}(\boldsymbol{\zeta}_0^1)-p\big)\big)-K\mathcal{V}\boldsymbol{\mathcal{Y}}_0^1+\boldsymbol{\zeta}_0^1-\boldsymbol{\zeta}_0^1\\
&=\bar{\mathcal{L}}\big(\boldsymbol{\zeta}_0^1-\boldsymbol{\zeta}_{J-1}^0-\mu\mathcal{Q}\big(\boldsymbol{\mathcal{T}}(\boldsymbol{\zeta}_0^1)-\boldsymbol{\mathcal{G}}_{J-1}^0\boldsymbol{\zeta}_{J-1}^0-p+\boldsymbol{t}_{J-1}^0\big)\big)\big)-K\mathcal{V}(\boldsymbol{\mathcal{Y}}_0^1-\boldsymbol{\mathcal{Y}}_{J-1}^0)+\boldsymbol{\zeta}_0^1\\
&=\bar{\mathcal{L}}\big(2\boldsymbol{\zeta}_0^1-\boldsymbol{\zeta}_{J-1}^0-\mu\mathcal{Q}\big(\boldsymbol{\mathcal{T}}(\boldsymbol{\zeta}_0^1)-\boldsymbol{\mathcal{G}}_{J-1}^0\boldsymbol{\zeta}_{J-1}^0-p+\boldsymbol{t}_{J-1}^0\big)\big)\big)-K\mathcal{V}(\boldsymbol{\mathcal{Y}}_0^1-\boldsymbol{\mathcal{Y}}_{J-1}^0)+(I-\bar{\mathcal{L}})\boldsymbol{\zeta}_0^1\\
&=\bar{\mathcal{L}}\big(2\boldsymbol{\zeta}_0^1-\boldsymbol{\zeta}_{J-1}^0-\mu\mathcal{Q}\big(\boldsymbol{\mathcal{T}}(\boldsymbol{\zeta}_0^1)-\boldsymbol{\mathcal{G}}_{J-1}^0\boldsymbol{\zeta}_{J-1}^0-p+\boldsymbol{t}_{J-1}^0\big)\big)\big)
\end{aligned}$$
(S.17)

Finally, for $e\geq 1$ we get

$$\begin{aligned}
\boldsymbol{\zeta}_{i+1}^e&=\bar{\mathcal{L}}(\boldsymbol{\zeta}_i^e-\mu\mathcal{Q}(\boldsymbol{\mathcal{T}}(\boldsymbol{\zeta}_i^e)-p))-K\mathcal{V}\boldsymbol{\mathcal{Y}}_i^e\\
&=\bar{\mathcal{L}}(\boldsymbol{\zeta}_i^e-\mu\mathcal{Q}(\boldsymbol{\mathcal{T}}(\boldsymbol{\zeta}_i^e)-p))-K\mathcal{V}\boldsymbol{\mathcal{Y}}_i^e+\boldsymbol{\zeta}_i^e-\boldsymbol{\zeta}_i^e\\
&=\bar{\mathcal{L}}(\boldsymbol{\zeta}_i^e-\boldsymbol{\zeta}_{i-1}^e-\mu\mathcal{Q}(\boldsymbol{\mathcal{T}}(\boldsymbol{\zeta}_i^e)-\boldsymbol{\mathcal{T}}(\boldsymbol{\zeta}_{i-1}^e)))-K\mathcal{V}(\boldsymbol{\mathcal{Y}}_i^e-\boldsymbol{\mathcal{Y}}_{i-1}^e)+\boldsymbol{\zeta}_i^e\\
&=\bar{\mathcal{L}}(2\boldsymbol{\zeta}_i^e-\boldsymbol{\zeta}_{i-1}^e-\mu\mathcal{Q}(\boldsymbol{\mathcal{T}}(\boldsymbol{\zeta}_i^e)-\boldsymbol{\mathcal{T}}(\boldsymbol{\zeta}_{i-1}^e)))-K\mathcal{V}(\boldsymbol{\mathcal{Y}}_i^e-\boldsymbol{\mathcal{Y}}_{i-1}^e)+(I-\bar{\mathcal{L}})\boldsymbol{\zeta}_i^e\\
&=\bar{\mathcal{L}}(2\boldsymbol{\zeta}_i^e-\boldsymbol{\zeta}_{i-1}^e-\mu\mathcal{Q}(\boldsymbol{\mathcal{T}}(\boldsymbol{\zeta}_i^e)-\boldsymbol{\mathcal{T}}(\boldsymbol{\zeta}_{i-1}^e)))-K\mathcal{V}(\boldsymbol{\mathcal{Y}}_i^e-\boldsymbol{\mathcal{Y}}_{i-1}^e)+K\mathcal{V}^2\boldsymbol{\zeta}_i^e\\
&=\bar{\mathcal{L}}(2\boldsymbol{\zeta}_i^e-\boldsymbol{\zeta}_{i-1}^e-\mu\mathcal{Q}(\boldsymbol{\mathcal{T}}(\boldsymbol{\zeta}_i^e)-\boldsymbol{\mathcal{T}}(\boldsymbol{\zeta}_{i-1}^e)))
\end{aligned}$$
(S.18)

Note that (S.16), (S.17) and (S.18) coincide with (S.1). This completes the proof.

## B. Proof of Lemma 2

This proof is a simple extension from the one presented in Appendix F of the main document. Therefore we use the same matrix decomposition defined in the document to get:

$$\mathcal{H}^{-1}\begin{bmatrix}\widetilde{\boldsymbol{\zeta}}_{i+1}^e\\\widetilde{\boldsymbol{\mathcal{Y}}}_{i+1}^e\end{bmatrix}=\Bigg(\mathcal{D}-\mu\mathcal{H}^{-1}\underbrace{\begin{bmatrix}\bar{\mathcal{L}}\mathcal{G}&0_{2KM\times 2KM}\\\mathcal{V}\bar{\mathcal{L}}\mathcal{G}&0_{2KM\times 2KM}\end{bmatrix}}_{\triangleq\mathcal{T}}\mathcal{H}\Bigg)\mathcal{H}^{-1}\begin{bmatrix}\widetilde{\boldsymbol{\zeta}}_i^e\\\widetilde{\boldsymbol{\mathcal{Y}}}_i^e\end{bmatrix}+\mu\mathcal{H}^{-1}\underbrace{\begin{bmatrix}\bar{\mathcal{L}}\\\mathcal{V}\bar{\mathcal{L}}\end{bmatrix}}_{\triangleq\mathcal{P}}\boldsymbol{s}(\boldsymbol{\zeta}_i^e)$$
(S.19)

We now define the new coordinates:

$$\mathcal{H}^{-1}\begin{bmatrix}\widetilde{\boldsymbol{\zeta}}_{i+1}^e\\\widetilde{\boldsymbol{\mathcal{Y}}}_{i+1}^e\end{bmatrix}=\begin{bmatrix}\bar{\boldsymbol{x}}_{i+1}^e\overset{\triangle}{=}K^{-1/2}\mathcal{I}^T\widetilde{\boldsymbol{\zeta}}_{i+1}^e\\\widecheck{\boldsymbol{x}}_{i+1}^e\overset{\triangle}{=}K^{-1/2}\mathcal{I}^T\widetilde{\boldsymbol{\mathcal{Y}}}_{i+1}^e\\\hat{\boldsymbol{x}}_{i+1}^e\overset{\triangle}{=}\mathcal{H}_l^T\widetilde{\boldsymbol{\zeta}}_{i+1}^e+\mathcal{H}_r^T\widetilde{\boldsymbol{\mathcal{Y}}}_{i+1}^e\end{bmatrix}$$
(S.20)

Expanding $\mathcal{H}^{-1}\mathcal{P}$ and following exactly the steps from Appendix F in the main document we get the desired relation.

$$\begin{bmatrix}\bar{\boldsymbol{x}}_{i+1}^e\\\hat{\boldsymbol{x}}_{i+1}^e\end{bmatrix}=\begin{bmatrix}I_{2M}-\mu GK^{-1}&-\mu\mathcal{I}^T\mathcal{G}\mathcal{H}_uK^{-1/2}\\-\mu K^{-1/2}(\mathcal{H}_l^T+\mathcal{H}_r^T\mathcal{V})\bar{\mathcal{L}}\mathcal{G}\mathcal{I}&\mathcal{D}_1-\mu(\mathcal{H}_l^T+\mathcal{H}_r^T\mathcal{V})\bar{\mathcal{L}}\mathcal{G}\mathcal{H}_u\end{bmatrix}\begin{bmatrix}\bar{\boldsymbol{x}}_i^e\\\hat{\boldsymbol{x}}_i^e\end{bmatrix}+\mu\begin{bmatrix}\mathcal{I}^TK^{-1/2}\\(\mathcal{H}_l^T+\mathcal{H}_r^T\mathcal{V})\bar{\mathcal{L}}\end{bmatrix}\boldsymbol{s}(\boldsymbol{\zeta}_i^e)$$
(S.21)

## C. Proof of Lemma 3

We start by stating the following useful Lemma.

**Lemma 5.** *The following bounds hold:*

$$\left\|\mathcal{I}^T s(\zeta_i^e)\right\|^2 \leq 2K \max_{k,n}(\|\widetilde{G}_{k,n}\|^2) \left\|\bar{x}_0^e - \bar{x}_i^e\right\|^2 + 2K \max_{k,n}(\|G_{k,n}\|^2) J^{-1} \sum_{n=1}^{J}\left\|\bar{x}_J^{e-1} - \bar{x}_{n-1}^{e-1}\right\|^2$$

$$+ 2K\|\mathcal{H}_u\|^2 \max_{k,n}(\|\widetilde{G}_{k,n}\|^2) \|\hat{x}_0^e - \hat{x}_i^e\|^2 + 2K\|\mathcal{H}_u\|^2 J^{-1} \max_{k,n}(\|G_{k,n}\|^2) \sum_{n=1}^{J}\|\hat{x}_J^{e-1} - \hat{x}_{n-1}^{e-1}\|^2 \tag{S.22}$$

$$\left\|(\mathcal{H}_l^T + \mathcal{H}_r^T \mathcal{V})\bar{\mathcal{L}} s(\zeta_i^e)\right\|^2 \leq 2\left\|(\mathcal{H}_l^T + \mathcal{H}_r^T \mathcal{V})\bar{\mathcal{L}}\right\|^2 \max_{k,n}(\|\widetilde{G}_{k,n}\|^2)\left(\|\bar{x}_0^e - \bar{x}_i^e\|^2 + \|\mathcal{H}_u\|^2\|\hat{x}_0^e - \hat{x}_i^e\|^2\right)$$

$$+ 2\left\|(\mathcal{H}_l^T + \mathcal{H}_r^T \mathcal{V})\bar{\mathcal{L}}\right\|^2 \max_{k,n}(\|G_{k,n}\|^2)\frac{1}{J}\sum_{n=1}^{J}\left(\|\bar{x}_0^e - \bar{x}_i^e\|^2 + \|\mathcal{H}_u\|^2\|\hat{x}_{n-1}^{e-1} - \hat{x}_J^{e-1}\|^2\right) \tag{S.23}$$

*where $\widetilde{G}_{k,n} = (G - G_{k,n})$.*

*Proof.* See Appendix I-D. ∎

Now we expand the top recursion of (S.8) and take the squared norm on both sides of the equation to get:

$$\|\check{x}_{i+1}^e\|^2 = \left\|\left(I_{2M} - \mu\Lambda_G K^{-1}\right)\check{x}_i^e - \mu Z^{-1} \mathcal{I}^T \mathcal{G} \mathcal{H}_u K^{-1/2} \hat{x}_i^e + \mu Z^{-1} \mathcal{I}^T K^{-\frac{1}{2}} s(\zeta_i^e)\right\|^2 \tag{S.24}$$

$$\|\check{x}_{i+1}^e\|^2 = \left\|\frac{t}{t}\left(I_{2M} - \mu\Lambda_G K^{-1}\right)\check{x}_i^e + \frac{(1-t)/2}{(1-t)/2}\left(-\mu Z^{-1} \mathcal{I}^T \mathcal{G} \mathcal{H}_u K^{-\frac{1}{2}} \hat{x}_i^e\right) + \frac{(1-t)/2}{(1-t)/2}\mu Z^{-1} \mathcal{I}^T K^{-\frac{1}{2}} s(\zeta_i^e)\right\|^2 \tag{S.25}$$

$$\|\check{x}_{i+1}^e\|^2 \overset{(a)}{\leq} t^{-1}\left\|\left(I_{2M} - \mu\Lambda_G K^{-1}\right)\check{x}_i^e\right\|^2 + 2(1-t)^{-1}\mu^2 K^{-1}\left\|Z^{-1} \mathcal{I}^T \mathcal{G} \mathcal{H}_u \hat{x}_i^e\right\|^2 + 2(1-t)^{-1}\mu^2 K^{-1}\left\|Z^{-1} \mathcal{I}^T s(\zeta_i^e)\right\|^2 \tag{S.26}$$

$$\|\check{x}_{i+1}^e\|^2 \overset{(b)}{\leq} t^{-1}\left\|I_{2M} - \mu\Lambda_G K^{-1}\right\|^2\|\check{x}_i^e\|^2 + 2(1-t)^{-1}\mu^2 K^{-1}\left\|Z^{-1} \mathcal{I}^T \mathcal{G} \mathcal{H}_u \hat{x}_i^e\right\|^2 + 2(1-t)^{-1}\mu^2 K^{-1}\left\|Z^{-1} \mathcal{I}^T s(\zeta_i^e)\right\|^2 \tag{S.27}$$

where $t \in (0,1)$. In $(a)$ we used Jensen's inequality and in $(b)$ we used Schwarz inequality. If $\mu < K/\rho(\Lambda_G)$ then we can choose $t = \|I_{2M} - \mu\Lambda_G K^{-1}\| = 1 - \mu a_1 K^{-1} < 1$, where $a_1$ is the minimum eigenvalue of $G$:

$$a_1 = \Lambda_{G,\min} \tag{S.28}$$

Hence, we can write:

$$\|\check{x}_{i+1}^e\|^2 \leq (1 - \mu a_1 K^{-1})\|\check{x}_i^e\|^2 + 4\mu a_1^{-1}\left\|Z^{-1} \mathcal{I}^T \mathcal{G} \mathcal{H}_u \hat{x}_i^e\right\|^2 + \mu 4 a_1^{-1}\left\|Z^{-1} \mathcal{I}^T s(\zeta_i^e)\right\|^2 \tag{S.29}$$

$$\leq (1 - \mu a_1 K^{-1})\|\check{x}_i^e\|^2 + \underbrace{\mu 4 a_1^{-1}\|Z^{-1} \mathcal{I}^T \mathcal{G} \mathcal{H}_u\|^2}_{\triangleq a_2}\|\hat{x}_i^e\|^2 + \mu 4 a_1^{-1}\|Z^{-1}\|^2\left\|\mathcal{I}^T s(\zeta_i^e)\right\|^2$$

$$\overset{(c)}{\leq} (1 - \mu a_1 K^{-1})\|\check{x}_i^e\|^2 + \mu a_2\|\hat{x}_i^e\|^2 + \underbrace{\mu 8 K a_1^{-1}\|Z^{-1}\|^2 \max_{k,n}(\|\widetilde{G}_{k,n}\|^2)}_{\triangleq a_3}\|\bar{x}_0^e - \check{x}_i^e\|^2$$

$$+\underbrace{\mu 8Ka_1^{-1}\|Z^{-1}\|^2\max_{k,n}\big(\big\|G_{k,n}\big\|^2\big)}_{\triangleq a_4}J^{-1}\sum_{n=1}^{J}\big\|\check{\boldsymbol{x}}_J^{e-1}-\check{\boldsymbol{x}}_{n-1}^{e-1}\big\|^2$$

$$+\mu 8Ka_1^{-1}\|Z^{-1}\|^2\|\mathcal{H}_u\|^2\max_{k,n}\big(\|\widetilde{G}_{k,n}\|^2\big)\|\hat{\boldsymbol{x}}_0^e-\hat{\boldsymbol{x}}_i^e\|^2$$

$$+\mu 8Ka_1^{-1}\|Z^{-1}\|^2\|\mathcal{H}_u\|^2\max_{k,n}\big(\|G_{k,n}\|^2\big)J^{-1}\sum_{n=1}^{J}\big\|\hat{\boldsymbol{x}}_J^{e-1}-\hat{\boldsymbol{x}}_{n-1}^{e-1}\big\|^2$$

$$=(1-\mu a_1K^{-1})\|\check{\boldsymbol{x}}_i^e\|^2+\mu a_2\|\hat{\boldsymbol{x}}_i^e\|^2+\mu a_3\Big(\big\|\check{\boldsymbol{x}}_0^e-\check{\boldsymbol{x}}_i^e\big\|^2+\|\mathcal{H}_u\|^2\big\|\hat{\boldsymbol{x}}_0^e-\hat{\boldsymbol{x}}_i^e\big\|^2\Big)$$

$$+\mu a_4J^{-1}\sum_{n=1}^{J}\Big(\big\|\check{\boldsymbol{x}}_J^{e-1}-\check{\boldsymbol{x}}_{n-1}^{e-1}\big\|^2+\|\mathcal{H}_u\|^2\big\|\hat{\boldsymbol{x}}_J^{e-1}-\hat{\boldsymbol{x}}_{n-1}^{e-1}\big\|^2\Big) \tag{S.30}$$

where in $(c)$ we used (S.22). Note that the inequality obtained constitutes the top inequality of (S.12). We now repeat the procedure for the bottom recursion of (S.12):

$$\|\hat{\boldsymbol{x}}_{i+1}^e\|^2=\Big\|-\mu K^{-\frac{1}{2}}(\mathcal{H}_l^T+\mathcal{H}_r^T\mathcal{V})\bar{\mathcal{L}}\mathcal{G}\mathcal{I}Z\check{\boldsymbol{x}}_i^e+(\mathcal{D}_1-\mu(\mathcal{H}_l^T+\mathcal{H}_r^T\mathcal{V})\bar{\mathcal{L}}\mathcal{G}\mathcal{H}_u)\hat{\boldsymbol{x}}_i^e+\mu(\mathcal{H}_l^T+\mathcal{H}_r^T\mathcal{V})\bar{\mathcal{L}}\boldsymbol{s}(\boldsymbol{\zeta}_i^e)\Big\|^2$$

$$=\Big\|-\frac{(1-t)/3}{(1-t)/3}\mu K^{-\frac{1}{2}}(\mathcal{H}_l^T+\mathcal{H}_r^T\mathcal{V})\bar{\mathcal{L}}\mathcal{G}\mathcal{I}Z\check{\boldsymbol{x}}_i^e+\frac{t}{t}\mathcal{D}_1\hat{\boldsymbol{x}}_i^e-\frac{(1-t)/3}{(1-t)/3}\big(\mu(\mathcal{H}_l^T+\mathcal{H}_r^T\mathcal{V})\bar{\mathcal{L}}\mathcal{G}\mathcal{H}_u\big)\hat{\boldsymbol{x}}_i^e$$

$$+\frac{(1-t)/3}{(1-t)/3}\mu(\mathcal{H}_l^T+\mathcal{H}_r^T\mathcal{V})\bar{\mathcal{L}}\boldsymbol{s}(\boldsymbol{\zeta}_i^e)\Big\|^2$$

$$\leq 3(1-t)^{-1}\|\mu K^{-\frac{1}{2}}(\mathcal{H}_l^T+\mathcal{H}_r^T\mathcal{V})\bar{\mathcal{L}}\mathcal{G}\mathcal{I}Z\check{\boldsymbol{x}}_i^e\|^2+t^{-1}\|\mathcal{D}_1\hat{\boldsymbol{x}}_i^e\|^2+3(1-t)^{-1}\|\mu(\mathcal{H}_l^T+\mathcal{H}_r^T\mathcal{V})\bar{\mathcal{L}}\mathcal{G}\mathcal{H}_u\hat{\boldsymbol{x}}_i^e\|^2$$

$$+3(1-t)^{-1}\|\mu(\mathcal{H}_l^T+\mathcal{H}_r^T\mathcal{V})\bar{\mathcal{L}}\boldsymbol{s}(\boldsymbol{\zeta}_i^e)\|^2$$

$$\leq 3(1-t)^{-1}\mu^2K^{-1}\|(\mathcal{H}_l^T+\mathcal{H}_r^T\mathcal{V})\bar{\mathcal{L}}\mathcal{G}\mathcal{I}Z\|^2\|\check{\boldsymbol{x}}_i^e\|^2+t^{-1}\|\mathcal{D}_1\|^2\|\hat{\boldsymbol{x}}_i^e\|^2$$

$$+3(1-t)^{-1}\mu^2\|(\mathcal{H}_l^T+\mathcal{H}_r^T\mathcal{V})\bar{\mathcal{L}}\mathcal{G}\mathcal{H}_u\|^2\|\hat{\boldsymbol{x}}_i^e\|^2+3(1-t)^{-1}\mu^2\|(\mathcal{H}_l^T+\mathcal{H}_r^T\mathcal{V})\bar{\mathcal{L}}\boldsymbol{s}(\boldsymbol{\zeta}_i^e)\|^2$$

$$\overset{(f)}{\leq}\sqrt{\lambda_2(\bar{L})}\|\hat{\boldsymbol{x}}_i^e\|^2+\mu^2\underbrace{\frac{3\|(\mathcal{H}_l^T+\mathcal{H}_r^T\mathcal{V})\bar{\mathcal{L}}\mathcal{G}\mathcal{H}_u\|^2}{\big(1-\sqrt{\lambda_2(\bar{L})}\big)}}_{\triangleq a_5}\|\hat{\boldsymbol{x}}_i^e\|^2+\mu^2\underbrace{\frac{3K^{-1}\|(\mathcal{H}_l^T+\mathcal{H}_r^T\mathcal{V})\bar{\mathcal{L}}\mathcal{G}\mathcal{I}Z\|^2}{1-\sqrt{\lambda_2(\bar{L})}}}_{\triangleq a_6}\|\check{\boldsymbol{x}}_i^e\|^2$$

$$+\mu^2\underbrace{\frac{6\|(\mathcal{H}_l^T+\mathcal{H}_r^T\mathcal{V})\bar{\mathcal{L}}\|^2\max_{k,n}\big(\|\widetilde{G}_{k,n}\|^2\big)}{\big(1-\sqrt{\lambda_2(\bar{L})}\big)}}_{\triangleq a_7}\Big(\big\|\check{\boldsymbol{x}}_0^e-\check{\boldsymbol{x}}_i^e\big\|^2+\|\mathcal{H}_u\|^2\big\|(\hat{\boldsymbol{x}}_0^e-\hat{\boldsymbol{x}}_i^e)\big\|^2\Big)$$

$$+\mu^2\underbrace{\frac{6\|(\mathcal{H}_l^T+\mathcal{H}_r^T\mathcal{V})\bar{\mathcal{L}}\|^2\max_{k,n}\big(\|G_{k,n}\|^2\big)}{\big(1-\sqrt{\lambda_2(\bar{L})}\big)}}_{\triangleq a_8}J^{-1}\sum_{n=1}^{J}\Big(\big\|\check{\boldsymbol{x}}_0^e-\check{\boldsymbol{x}}_i^e\big\|^2+\|\mathcal{H}_u\|^2\big\|(\hat{\boldsymbol{x}}_{n-1}^{e-1}-\hat{\boldsymbol{x}}_J^{e-1})\big\|^2\Big) \tag{S.31}$$

where in $(f)$ we chose $t=\|\mathcal{D}_1\|=\sqrt{\lambda_2(\bar{L})}$ and used (S.23). Noting that the inequality obtained constitutes the bottom inequality of (S.12) completes the proof.

### D. Proof of Lemma 5

Now we proceed to bound the gradient noise related terms.

$$\|\mathcal{I}^T\boldsymbol{s}(\boldsymbol{\zeta}_i^e)\|^2=\big\|\mathcal{I}^T\big(\boldsymbol{\mathcal{T}}(\zeta_i^e)-\mathcal{G}\boldsymbol{\zeta}_i^e\big)\big\|^2=\Big\|\sum_{k=1}^{K}\boldsymbol{T}(\zeta_{k,i}^e)-G_k\zeta_{k,i}^e\Big\|^2\overset{(a)}{\leq}\sum_{k=1}^{K}K\|\widehat{\boldsymbol{T}}(\zeta_{k,i}^e)-G_k\zeta_{k,i}^e\|^2 \tag{S.32}$$

where in $(a)$ we used Jensen's inequality. Next we analyze the individual terms $\left\|\boldsymbol{T}(\zeta_{k,i}^e)-G_k\zeta_{k,i}^e\right\|^2$:

$$\left\|\boldsymbol{T}(\zeta_{k,i}^e)-G_k\boldsymbol{\zeta}_{k,i}^e\right\|^2=\left\|\boldsymbol{G}_{k,i}^e(\boldsymbol{\zeta}_{k,i}^e-\boldsymbol{\zeta}_{k,0}^e)+\frac{1}{J}\sum_{n=1}^J\boldsymbol{G}_{k,n}^{e-1}\boldsymbol{\zeta}_{k,n-1}^{e-1}-G_k\boldsymbol{\zeta}_{k,i}^e\right\|^2$$

$$\overset{(b)}{=}\left\|\boldsymbol{G}_{k,i}^e(\boldsymbol{\zeta}_{k,i}^e-\boldsymbol{\zeta}_{k,0}^e)+G_k\left(\boldsymbol{\zeta}_{k,0}^e-\boldsymbol{\zeta}_{k,\hat{N}}^{e-1}\right)+\frac{1}{J}\sum_{n=1}^J\boldsymbol{G}_{k,n}^{e-1}\boldsymbol{\zeta}_{k,n-1}^{e-1}-G_k\boldsymbol{\zeta}_{k,i}^e\right\|^2$$

$$=\left\|\boldsymbol{G}_{k,i}^e(\boldsymbol{\zeta}_{k,i}^e-\boldsymbol{\zeta}_{k,0}^e)+G_k(\boldsymbol{\zeta}_{k,0}^e-\boldsymbol{\zeta}_{k,i}^e)+\frac{1}{J}\sum_{n=1}^J\boldsymbol{G}_{k,n}^{e-1}\left(\boldsymbol{\zeta}_{k,n-1}^{e-1}-\boldsymbol{\zeta}_{k,J}^{e-1}\right)\right\|^2$$

$$\overset{(c)}{=}\left\|(G_k-\boldsymbol{G}_{k,i}^e)(\boldsymbol{\zeta}_{k,0}^e-\boldsymbol{\zeta}_{k,i}^e)+\frac{1}{J}\sum_{n=1}^J\boldsymbol{G}_{k,n}^{e-1}\left(\boldsymbol{\zeta}_{k,n-1}^{e-1}-\boldsymbol{\zeta}_{k,J}^{e-1}\right)\right\|^2$$

$$\overset{(d)}{\leq}2\left\|\widetilde{\boldsymbol{G}}_{k,i}^e(\boldsymbol{\zeta}_{k,0}^e-\boldsymbol{\zeta}_{k,i}^e)\right\|^2+\frac{2}{J}\sum_{n=1}^J\left\|\boldsymbol{G}_{k,n}^{e-1}\left(\boldsymbol{\zeta}_{k,n-1}^{e-1}-\boldsymbol{\zeta}_{k,J}^{e-1}\right)\right\|^2$$

$$\overset{(e)}{\leq}2\left\|\widetilde{\boldsymbol{G}}_{k,i}^e\right\|^2\left\|\boldsymbol{\zeta}_{k,0}^e-\boldsymbol{\zeta}_{k,i}^e\right\|^2+\frac{2}{J}\sum_{n=1}^J\left\|\boldsymbol{G}_{k,n}^{e-1}\right\|^2\left\|\boldsymbol{\zeta}_{k,n-1}^{e-1}-\boldsymbol{\zeta}_{k,J}^{e-1}\right\|^2$$

$$\leq2\max_n\left(\left\|\widetilde{G}_{k,n}\right\|^2\right)\left\|\boldsymbol{\zeta}_{k,0}^e-\boldsymbol{\zeta}_{k,i}^e\right\|^2+\max_n\left(\left\|G_{k,n}\right\|^2\right)\frac{2}{J}\sum_{n=1}^J\left\|\boldsymbol{\zeta}_{k,n-1}^{e-1}-\boldsymbol{\zeta}_{k,J}^{e-1}\right\|^2 \tag{S.33}$$

where in $(b)$ we used the fact that $\boldsymbol{\zeta}_{k,0}^e=\boldsymbol{\zeta}_{k,\hat{N}}^{e-1}$, in $(c)$ we defined $G_k-\boldsymbol{G}_{k,i}^e=\widetilde{\boldsymbol{G}}_{k,i}^e$, in $(d)$ we used Jensen's inequality and in $(e)$ we used Schwarz inequality. Combining (S.32) and (S.33) we get:

$$\|\mathcal{I}^T\boldsymbol{s}(\boldsymbol{\zeta}_i^e)\|^2\leq\sum_{k=1}^K K\left(2\max_n\left(\left\|\widetilde{G}_{k,n}\right\|^2\right)\left\|\boldsymbol{\zeta}_{k,0}^e-\boldsymbol{\zeta}_{k,i}^e\right\|^2+\max_n\left(\left\|G_{k,n}\right\|^2\right)\frac{2}{J}\sum_{n=1}^J\left\|\boldsymbol{\zeta}_{k,n-1}^{e-1}-\boldsymbol{\zeta}_{k,J}^{e-1}\right\|^2\right)$$

$$\overset{(a)}{\leq}2K\max_{k,n}\left(\left\|\widetilde{G}_{k,n}\right\|^2\right)\left\|\boldsymbol{\zeta}_0^e-\boldsymbol{\zeta}_i^e\right\|^2+\frac{2K}{J}\max_{k,n}\left(\left\|G_{k,n}\right\|^2\right)\sum_{n=1}^J\left\|\boldsymbol{\zeta}_{n-1}^{e-1}-\boldsymbol{\zeta}_J^{e-1}\right\|^2$$

$$=2K\max_{k,n}\left(\left\|\widetilde{G}_{k,n}\right\|^2\right)\left\|\widetilde{\boldsymbol{\zeta}}_i^e-\widetilde{\boldsymbol{\zeta}}_0^e\right\|^2+\frac{2K}{J}\max_{k,n}\left(\left\|G_{k,n}\right\|^2\right)\sum_{n=1}^J\left\|\widetilde{\boldsymbol{\zeta}}_J^{e-1}-\widetilde{\boldsymbol{\zeta}}_{n-1}^{e-1}\right\|^2 \tag{S.34}$$

where in $(a)$ we used Hölder's inequality. Now we need an equation relating $\left\|\widetilde{\boldsymbol{\zeta}}_i^e-\widetilde{\boldsymbol{\zeta}}_0^e\right\|$ with $\left\|\bar{\boldsymbol{x}}_0^e-\bar{\boldsymbol{x}}_i^e\right\|$ and $\left\|\hat{\boldsymbol{x}}_0^e-\hat{\boldsymbol{x}}_i^e\right\|$ which we get as follows:

$$\left\|\widetilde{\boldsymbol{\zeta}}_i^e-\widetilde{\boldsymbol{\zeta}}_0^e\right\|^2\overset{(a)}{=}\left\|K^{-1/2}\mathcal{I}(\bar{\boldsymbol{x}}_0^e-\bar{\boldsymbol{x}}_i^e)+\mathcal{H}_u(\hat{\boldsymbol{x}}_0^e-\hat{\boldsymbol{x}}_i^e)\right\|^2\overset{(b)}{=}K^{-1}\left\|\mathcal{I}(\bar{\boldsymbol{x}}_0^e-\bar{\boldsymbol{x}}_i^e)\right\|^2+\left\|\mathcal{H}_u(\hat{\boldsymbol{x}}_0^e-\hat{\boldsymbol{x}}_i^e)\right\|^2$$

$$\overset{(c)}{\leq}\left\|\bar{\boldsymbol{x}}_0^e-\bar{\boldsymbol{x}}_i^e\right\|^2+\left\|\mathcal{H}_u\right\|^2\left\|(\hat{\boldsymbol{x}}_0^e-\hat{\boldsymbol{x}}_i^e)\right\|^2 \tag{S.35}$$

where in $(a)$ we used (S.20), in $(b)$ we used $\mathcal{I}^T\mathcal{H}_u=0$, in $(c)$ we used the fact that the maximum eigenvalue of $\mathcal{I}^T\mathcal{I}$ is $K$. Combining (S.34) and (S.35) we finally get:

$$\|\mathcal{I}^T\boldsymbol{s}(\boldsymbol{\zeta}_i^e)\|^2\leq2K\max_{k,n}\left(\left\|\widetilde{G}_{k,n}\right\|^2\right)\left\|\bar{\boldsymbol{x}}_0^e-\bar{\boldsymbol{x}}_i^e\right\|^2+2K\max_{k,n}\left(\left\|G_{k,n}\right\|^2\right)J^{-1}\sum_{n=1}^J\left\|\bar{\boldsymbol{x}}_J^{e-1}-\bar{\boldsymbol{x}}_{n-1}^{e-1}\right\|^2$$

$$+2K\|\mathcal{H}_u\|^2\max_{k,n}\left(\left\|\widetilde{G}_{k,n}\right\|^2\right)\left\|\hat{\boldsymbol{x}}_0^e-\hat{\boldsymbol{x}}_i^e\right\|^2+2K\|\mathcal{H}_u\|^2J^{-1}\max_{k,n}\left(\left\|G_{k,n}\right\|^2\right)\sum_{n=1}^J\left\|\hat{\boldsymbol{x}}_J^{e-1}-\hat{\boldsymbol{x}}_{n-1}^{e-1}\right\|^2 \tag{S.36}$$

which is (S.22). We now proceed to bound the other noise term.

$$\|(\mathcal{H}_l^T+\mathcal{H}_r^T\mathcal{V})\bar{\mathcal{L}}\boldsymbol{s}(\boldsymbol{\zeta}_i^e)\|^2=\left\|(\mathcal{H}_l^T+\mathcal{H}_r^T\mathcal{V})\bar{\mathcal{L}}(\mathcal{T}(\boldsymbol{\zeta}_i^e)-\mathcal{G}\boldsymbol{\zeta}_i^e)\right\|^2\leq\left\|(\mathcal{H}_l^T+\mathcal{H}_r^T\mathcal{V})\bar{\mathcal{L}}\right\|^2\|\mathcal{T}(\boldsymbol{\zeta}_i^e)-\mathcal{G}\boldsymbol{\zeta}_i^e\|^2$$

$$=\left\|(\mathcal{H}_l^T+\mathcal{H}_r^T\mathcal{V})\bar{\mathcal{L}}\right\|^2\left\|\begin{bmatrix}\boldsymbol{T}(\zeta_{1,i}^e)-G_1\zeta_{1,i}^e\\\vdots\\\boldsymbol{T}(\zeta_{K,i}^e)-G_K\zeta_{K,i}^e\end{bmatrix}\right\|^2=\left\|(\mathcal{H}_l^T+\mathcal{H}_r^T\mathcal{V})\bar{\mathcal{L}}\right\|^2\sum_{k=1}^K\|\boldsymbol{T}(\zeta_{k,i}^e)-G_k\zeta_{k,i}^e\|^2$$

$$\overset{(a)}{\leq}2\left\|(\mathcal{H}_l^T+\mathcal{H}_r^T\mathcal{V})\bar{\mathcal{L}}\right\|^2\sum_{k=1}^K\left(\max_n\left(\|\widetilde{G}_{k,n}\|^2\right)\|\boldsymbol{\zeta}_{k,0}^e-\boldsymbol{\zeta}_{k,i}^e\|^2+\max_n\left(\|G_{k,n}\|^2\right)\frac{1}{J}\sum_{n=1}^J\|\boldsymbol{\zeta}_{k,n-1}^{e-1}-\boldsymbol{\zeta}_{k,J}^{e-1}\|^2\right)$$

$$\overset{(b)}{\leq}2\left\|(\mathcal{H}_l^T+\mathcal{H}_r^T\mathcal{V})\bar{\mathcal{L}}\right\|^2\left(\max_{k,n}\left(\|\widetilde{G}_{k,n}\|^2\right)\|\boldsymbol{\zeta}_0^e-\boldsymbol{\zeta}_i^e\|^2+\max_{k,n}\left(\|G_{k,n}\|^2\right)\frac{1}{J}\sum_{n=1}^J\|\boldsymbol{\zeta}_{n-1}^{e-1}-\boldsymbol{\zeta}_J^{e-1}\|^2\right)$$

$$\overset{(c)}{\leq}2\left\|(\mathcal{H}_l^T+\mathcal{H}_r^T\mathcal{V})\bar{\mathcal{L}}\right\|^2\left[\max_{k,n}\left(\|\widetilde{G}_{k,n}\|^2\right)\left(\|\bar{\boldsymbol{x}}_0^e-\bar{\boldsymbol{x}}_i^e\|^2+\|\mathcal{H}_u\|^2\|(\hat{\boldsymbol{x}}_0^e-\hat{\boldsymbol{x}}_i^e)\|^2\right)\right.$$

$$\left.+\max_{k,n}\left(\|G_{k,n}\|^2\right)\frac{1}{J}\sum_{n=1}^J\left(\|\bar{\boldsymbol{x}}_0^e-\bar{\boldsymbol{x}}_i^e\|^2+\|\mathcal{H}_u\|^2\|(\hat{\boldsymbol{x}}_{n-1}^{e-1}-\hat{\boldsymbol{x}}^{e-1})\|^2\right)\right]\tag{S.37}$$

where in $(a)$ we used (S.33), in $(b)$ we used Hölder's inequality and in $(c)$ we used (S.35). Note that the bound obtained is (S.23).

### E. Proof Lemma 4

We start by introducing the definitions

$$y_j^e\overset{\Delta}{=}\begin{bmatrix}\|\bar{\boldsymbol{x}}_j^e\|^2\\\|\hat{\boldsymbol{x}}_j^e\|^2\end{bmatrix}\tag{S.38a}$$

$$z_j^e\overset{\Delta}{=}\begin{bmatrix}\|\bar{\boldsymbol{x}}_j^e-\bar{\boldsymbol{x}}_0^e\|^2\\\|\hat{\boldsymbol{x}}_j^e-\hat{\boldsymbol{x}}_0^e\|^2\end{bmatrix}\tag{S.38b}$$

$$w_j^e\overset{\Delta}{=}\begin{bmatrix}\|\bar{\boldsymbol{x}}_j^e-\bar{\boldsymbol{x}}_j^e\|^2\\\|\hat{\boldsymbol{x}}_J^e-\hat{\boldsymbol{x}}_j^e\|^2\end{bmatrix}\tag{S.38c}$$

$$M_1\overset{\Delta}{=}\begin{bmatrix}1-\mu a_1K^{-1}&\mu a_2\\\mu^2a_6&\sqrt{\lambda_2(\bar{L})}+\mu^2a_5\end{bmatrix}\tag{S.38d}$$

$$M_2\overset{\Delta}{=}\begin{bmatrix}a_3&\|\mathcal{H}_u\|^2a_3\\\mu a_7&\mu\|\mathcal{H}_u\|^2a_7\end{bmatrix}=\begin{bmatrix}a_3\\\mu a_7\end{bmatrix}\begin{bmatrix}1\\\|\mathcal{H}_u\|^2\end{bmatrix}^T\tag{S.38e}$$

$$M_3\overset{\Delta}{=}\begin{bmatrix}a_4&\|\mathcal{H}_u\|^2a_4\\\mu a_8&\mu\|\mathcal{H}_u\|^2a_8\end{bmatrix}=\begin{bmatrix}a_4\\\mu a_8\end{bmatrix}\begin{bmatrix}1\\\|\mathcal{H}_u\|^2\end{bmatrix}^T\tag{S.38f}$$

#### 1) Bound for $y_j^e$:

We iterate equation (S.12) to get

$$y_j^e\preceq M_1\left(M_1y_{j-2}^e+\mu M_2z_{j-2}^e+\mu\frac{1}{J}\sum_{n=0}^{J-1}M_3w_n^{e-1}\right)+\mu M_2z_{j-1}^e+\mu M_3\frac{1}{J}\sum_{n=0}^{J-1}w_n^{e-1}$$

$$=M_1^2y_{j-2}^e+\mu\sum_{n=1}^2M_1^{n-1}M_2z_{j-n}^e+\mu\sum_{k=0}^1M_1^kM_3\frac{1}{J}\sum_{n=0}^{J-1}w_n^{e-1}$$

$$\preceq M_1^j y_0^e + \mu \sum_{n=1}^{j} M_1^{n-1} M_2 z_{j-n}^e + \mu \sum_{k=0}^{j-1} M_1^k M_3 \frac{1}{J} \sum_{n=0}^{J-1} w_n^{e-1} \tag{S.39}$$

Now we impose the following conditions on $\mu$:

$$(1-\mu K^{-1} a_1) a_3 + \mu^2 a_2 a_7 \leq a_3 \rightarrow \mu \leq \frac{a_1 a_3}{K a_2 a_7} \tag{S.40a}$$

$$\sqrt{\lambda_2(\overline{L})} a_7 + \mu a_3 a_6 + \mu^2 a_5 a_7 \leq a_7 \tag{S.40b}$$

$$(1-\mu K^{-1} a_1) a_4 + \mu^2 a_2 a_8 \leq a_4 \rightarrow \mu \leq \frac{a_1 a_4}{K a_2 a_8} \tag{S.40c}$$

$$\sqrt{\lambda_2(\overline{L})} a_8 + \mu a_4 a_6 + \mu^2 a_5 a_8 \leq a_8 \tag{S.40d}$$

Notice that since constants $a_1$ through $a_8$ are all strictly positive, there's always a step-size $\mu$ small enough such that the above conditions are satisfied. These conditions imply that the following entry-wise matrix inequalities hold:

$$M_1 M_2 = \begin{bmatrix} 1-\mu K^{-1} a_1 & \mu a_2 \\ \mu^2 a_6 & \left(\sqrt{\lambda_2(\overline{L})} + \mu^2 a_5\right) \end{bmatrix} \begin{bmatrix} a_3 \\ \mu a_7 \end{bmatrix} \begin{bmatrix} 1 \\ \|\mathcal{H}_u\|^2 \end{bmatrix}^T$$

$$= \begin{bmatrix} (1-\mu K^{-1} a_1) a_3 + \mu^2 a_2 a_7 \\ \mu^2 a_3 a_6 + \left(\sqrt{\lambda_2(\overline{L})} + \mu^2 a_5\right) \mu a_7 \end{bmatrix} \begin{bmatrix} 1 \\ \|\mathcal{H}_u\|^2 \end{bmatrix}^T \preceq \begin{bmatrix} a_3 \\ \mu a_7 \end{bmatrix} \begin{bmatrix} 1 \\ \|\mathcal{H}_u\|^2 \end{bmatrix}^T = M_2 \tag{S.41}$$

$$M_1 M_3 = \begin{bmatrix} 1-\mu K^{-1} a_1 & \mu a_2 \\ \mu^2 a_6 & \left(\sqrt{\lambda_2(\overline{L})} + \mu^2 a_5\right) \end{bmatrix} \begin{bmatrix} a_4 \\ \mu a_8 \end{bmatrix} \begin{bmatrix} 1 \\ K \end{bmatrix}^T$$

$$= \begin{bmatrix} (1-\mu K^{-1} a_1) a_4 + \mu^2 a_2 a_8 \\ \mu^2 a_4 a_6 + \left(\sqrt{\lambda_2(\overline{L})} + \mu^2 a_5\right) \mu a_8 \end{bmatrix} \begin{bmatrix} 1 \\ \|\mathcal{H}_u\|^2 \end{bmatrix}^T \preceq \begin{bmatrix} a_4 \\ \mu a_8 \end{bmatrix} \begin{bmatrix} 1 \\ \|\mathcal{H}_u\|^2 \end{bmatrix}^T = M_3 \tag{S.42}$$

and hence we can write:

$$y_j^e \preceq M_1^j y_0^e + \mu M_2 \sum_{n=1}^{j-1} z_n^e + \mu j M_3 \frac{1}{J} \sum_{n=0}^{J-1} w_n^{e-1} \tag{S.43}$$

*2) Bound for $z_j^e$:*
We can proceed to get a similar inequality for $z_j^e$. For this we start by using the top equation of (S.11):

$$\check{\boldsymbol{x}}_{i+1}^e - \check{\boldsymbol{x}}_i^e = -\mu \Lambda_G K^{-1} \check{\boldsymbol{x}}_i^e - \mu Z^{-1} \mathcal{I}^T \mathcal{G} \mathcal{H}_u K^{-1/2} \hat{\boldsymbol{x}}_i^e + \mu Z^{-1} \mathcal{I}^T K^{-1/2} \boldsymbol{s}(\boldsymbol{\zeta}_i^e) \tag{S.44}$$

Using the above equation we can write the following expression for $\check{\boldsymbol{x}}_{i+1}^e - \check{\boldsymbol{x}}_0^e$:

$$\check{\boldsymbol{x}}_{i+1}^e - \check{\boldsymbol{x}}_0^e = \sum_{j=0}^{i} \check{\boldsymbol{x}}_{j+1}^e - \check{\boldsymbol{x}}_j^e = \mu \sum_{j=0}^{i} -\Lambda_G K^{-1} \check{\boldsymbol{x}}_j^e - Z^{-1} \mathcal{I}^T \mathcal{G} \mathcal{H}_u K^{-1/2} \hat{\boldsymbol{x}}_j^e + Z^{-1} \mathcal{I}^T K^{-1/2} \boldsymbol{s}(\boldsymbol{\zeta}_j^e) \tag{S.45}$$

Taking squared norm on both sides we get:

$$\|\check{\boldsymbol{x}}_{i+1}^e - \check{\boldsymbol{x}}_0^e\|^2 = \mu^2 \left\| \sum_{j=0}^{i} -\Lambda_G K^{-1} \check{\boldsymbol{x}}_j^e - Z^{-1} \mathcal{I}^T \mathcal{G} \mathcal{H}_u K^{-1/2} \hat{\boldsymbol{x}}_j^e + Z^{-1} \mathcal{I}^T K^{-1/2} \boldsymbol{s}(\boldsymbol{\zeta}_j^e) \right\|^2$$

$$\overset{(a)}{\leq} 3\mu^2 (i+1) \sum_{j=0}^{i} \left( K^{-2} \|\Lambda_G \check{\boldsymbol{x}}_j^e\|^2 + K^{-1} \|Z^{-1} \mathcal{I}^T \mathcal{G} \mathcal{H}_u\|^2 \|\hat{\boldsymbol{x}}_j^e\|^2 + K^{-1} \|Z^{-1} \mathcal{I}^T \boldsymbol{s}(\boldsymbol{\zeta}_j^e)\|^2 \right)$$

$$\overset{(b)}{\leq} 3\mu^2(i+1)\sum_{j=0}^{i}\big(K^{-2}\|\Lambda_G\|^2\|\check{\boldsymbol{x}}_j^e\|^2 + K^{-1}\|Z^{-1}\mathcal{I}^T\mathcal{G}\mathcal{H}_u\|^2\|\hat{\boldsymbol{x}}_j^e\|^2 + K^{-1}\|Z^{-1}\|^2\|\mathcal{I}^T\boldsymbol{s}(\boldsymbol{\zeta}_j^e)\|^2\big)$$

$$\overset{(c)}{\leq} 3\mu^2(i+1)\bigg(\sum_{j=0}^{i}K^{-2}\|\Lambda_G\|^2\|\check{\boldsymbol{x}}_j^e\|^2 + K^{-1}\|Z^{-1}\mathcal{I}^T\mathcal{G}\mathcal{H}_u\|^2\|\hat{\boldsymbol{x}}_j^e\|^2$$

$$+2\|Z^{-1}\|^2\max_{k,n}\big(\|\widetilde{G}_{k,n}\|^2\big)\|\check{\boldsymbol{x}}_0^e - \check{\boldsymbol{x}}_j^e\|^2 + 2\|Z^{-1}\|^2\max_{k,n}\big(\|G_{k,n}\|^2\big)J^{-1}\sum_{n=1}^{J}\|\check{\boldsymbol{x}}_J^{e-1} - \check{\boldsymbol{x}}_{n-1}^{e-1}\|^2$$

$$+2\|Z^{-1}\|^2\|\mathcal{H}_u\|^2\max_{k,n}\big(\|\widetilde{G}_{k,n}\|^2\big)\|\hat{\boldsymbol{x}}_0^e - \hat{\boldsymbol{x}}_j^e\|^2$$

$$+2\|Z^{-1}\|^2\|\mathcal{H}_u\|^2\max_{k,n}\big(\|G_{k,n}\|^2\big)J^{-1}\sum_{n=1}^{J}\|\hat{\boldsymbol{x}}_J^{e-1} - \hat{\boldsymbol{x}}_{n-1}^{e-1}\|^2\bigg) \tag{S.46}$$

where in $(a)$ we used the Jensen's inequality, in $(b)$ Cauchy-Schawrtz inequality and in $(c)$ we used (S.22). We now proceed in a similar fashion to get an inequality for $\|\hat{\boldsymbol{x}}_{i+1}^e - \hat{\boldsymbol{x}}_0^e\|^2$.

$$\hat{\boldsymbol{x}}_{i+1}^e - \hat{\boldsymbol{x}}_i^e = -\mu K^{-1/2}(\mathcal{H}_l^T + \mathcal{H}_r^T\mathcal{V})\bar{\mathcal{L}}\mathcal{G}\mathcal{I}Z\check{\boldsymbol{x}}_i^e + [\mathcal{D}_1 - I_{4M(K-1)} - \mu(\mathcal{H}_l^T + \mathcal{H}_r^T\mathcal{V})\bar{\mathcal{L}}\mathcal{G}\mathcal{H}_u]\hat{\boldsymbol{x}}_i^e$$
$$+\mu(\mathcal{H}_l^T + \mathcal{H}_r^T\mathcal{V})\bar{\mathcal{L}}\boldsymbol{s}(\boldsymbol{\zeta}_i^e) \tag{S.47}$$

Similarly, as we did before, we get

$$\|\hat{\boldsymbol{x}}_{i+1}^e - \hat{\boldsymbol{x}}_0^e\|^2 = \bigg\|\sum_{j=0}^{i}-\mu K^{-\frac{1}{2}}(\mathcal{H}_l^T + \mathcal{H}_r^T\mathcal{V})\bar{\mathcal{L}}\mathcal{G}\mathcal{I}Z\check{\boldsymbol{x}}_j^e + [\mathcal{D}_1 - I - \mu(\mathcal{H}_l^T + \mathcal{H}_r^T\mathcal{V})\bar{\mathcal{L}}\mathcal{G}\mathcal{H}_u]\hat{\boldsymbol{x}}_j^e + \mu(\mathcal{H}_l^T + \mathcal{H}_r^T\mathcal{V})\bar{\mathcal{L}}\boldsymbol{s}(\boldsymbol{\zeta}_i^e)\bigg\|^2$$

$$\overset{(a)}{=} (i+1)\sum_{j=0}^{i}\big\|-\mu K^{-\frac{1}{2}}(\mathcal{H}_l^T + \mathcal{H}_r^T\mathcal{V})\bar{\mathcal{L}}\mathcal{G}\mathcal{I}Z\check{\boldsymbol{x}}_j^e + (\mathcal{D}_1 - I)\hat{\boldsymbol{x}}_j^e - \mu(\mathcal{H}_l^T + \mathcal{H}_r^T\mathcal{V})\bar{\mathcal{L}}\mathcal{G}\mathcal{H}_u\hat{\boldsymbol{x}}_j^e$$
$$+\mu(\mathcal{H}_l^T + \mathcal{H}_r^T\mathcal{V})\bar{\mathcal{L}}\boldsymbol{s}(\boldsymbol{\zeta}_i^e)\big\|^2$$

$$\overset{(b)}{\leq} (i+1)\bigg(1 - \sqrt{\lambda_2(\bar{L})} + \mu^2 a_5\bigg)\sum_{j=0}^{i}\|\hat{\boldsymbol{x}}_j^e\|^2 + \mu^2(i+1)a_6\sum_{j=0}^{i}\|\check{\boldsymbol{x}}_j^e\|^2$$

$$+\mu^2(i+1)a_7\sum_{j=0}^{i}\big(\|\check{\boldsymbol{x}}_0^e - \check{\boldsymbol{x}}_j^e\|^2 + \|\mathcal{H}_u\|^2\|(\hat{\boldsymbol{x}}_0^e - \hat{\boldsymbol{x}}_j^e)\|^2\big)$$

$$+\mu^2 a_8(i+1)^2 J^{-1}\sum_{n=1}^{J}\big(\|\check{\boldsymbol{x}}_0^e - \check{\boldsymbol{x}}_n^e\|^2 + \|\mathcal{H}_u\|^2\|(\hat{\boldsymbol{x}}_{n-1}^{e-1} - \hat{\boldsymbol{x}}_J^{e-1})\|^2\big) \tag{S.48}$$

where in $(a)$ we used Jensen's inequality and in $(b)$ we reused the calculation of the norm in (S.31). Now combining (S.46) and (S.48) in matrix form we get the following inequality

$$z_{i+1}^e \preceq (i+1)\underbrace{\begin{bmatrix} 3\mu^2 K^{-2}\|\Lambda_G\|^2 & 3\mu^2 K^{-1}\|Z^{-1}\mathcal{I}^T\mathcal{G}\mathcal{H}_u\|^2 \\ \mu^2 a_6 & 1 - \sqrt{\lambda_2(\bar{L})} + \mu^2 a_5 \end{bmatrix}}_{\triangleq M_4}\sum_{j=0}^{i}y_j^e$$

$$+(i+1)\mu^2\underbrace{\begin{bmatrix} 6\|Z^{-1}\|^2\max_{k,n}(\|G_{k,n}\|^2) \\ a_7 \end{bmatrix}}_{\triangleq M_5}\begin{bmatrix} 1 \\ \|\mathcal{H}_u\|^2 \end{bmatrix}^T\sum_{j=0}^{i}z_j^e$$

$$+(i+1)^2\mu^2\underbrace{\left[\begin{array}{c}6\|Z^{-1}\|^2\max_{k,n}\left(\|G_{k,n}\|^2\right)\\a_8\end{array}\right]\left[\begin{array}{c}1\\\|\mathcal{H}_u\|^2\end{array}\right]^T}_{\triangleq M_6}J^{-1}\sum_{n=0}^{J-1}w_n^{e-1}\qquad\text{(S.49)}$$

Combining (S.43) and (S.49) we get:

$$
\begin{aligned}
z_{i+1}^e\preceq&(i+1)M_4\sum_{j=0}^{i}\left(M_1^jy_0^e+\mu M_2\sum_{n=0}^{j-1}z_n^e+\mu jM_3\frac{1}{J}\sum_{n=0}^{J-1}w_n^{e-1}\right)+(i+1)\mu^2M_5\sum_{j=0}^{i}z_j^e+(i+1)^2\mu^2M_6\frac{1}{J}\sum_{n=0}^{J-1}w_n^{e-1}\\
=&(i+1)M_4\sum_{j=0}^{i}M_1^jy_0^e+(i+1)\mu M_4M_2\sum_{j=1}^{i}\sum_{n=0}^{j-1}b_n^e+(i+1)\mu^2M_5\sum_{j=0}^{i}z_j^e\\
&+(i+1)\mu\left(\frac{(i+1)i}{2}M_4M_3+(i+1)\mu M_6\right)\frac{1}{J}\sum_{n=0}^{J-1}w_n^{e-1}\\
=&(i+1)M_4\sum_{j=0}^{i}M_1^jy_0^e+(i+1)\mu\sum_{j=0}^{i}((i-j)M_4M_2+\mu M_5)z_j^e+(i+1)^2\mu\left(\frac{i}{2}M_4M_3+\mu M_6\right)\frac{1}{J}\sum_{n=0}^{J-1}w_n^{e-1}\\
\preceq&(i+1)M_4\sum_{j=0}^{i}M_1^jy_0^e+(i+1)\mu(iM_4M_2+\mu M_5)\sum_{j=0}^{i}z_j^e+(i+1)^2\mu\left(\frac{i}{2}M_4M_3+\mu M_6\right)\frac{1}{J}\sum_{n=0}^{J-1}w_n^{e-1}\\
\preceq&(i+1)M_4(I-M_1)^{-1}y_0^e+(i+1)\mu(iM_4M_2+\mu M_5)\sum_{j=0}^{i}z_j^e+(i+1)^2\mu\left(\frac{i}{2}M_4M_3+\mu M_6\right)\frac{1}{J}\sum_{n=0}^{J-1}w_n^{e-1}
\end{aligned}
$$
$$\text{(S.50)}$$

Summing across all $z_i^e$ terms in an epoch we get:

$$\frac{1}{J}\sum_{i=1}^{J}z_i^e\preceq\frac{J+1}{2}M_4(I-M_1)^{-1}y_0^e+\mu(J+1)\frac{J}{2}(\mu M_5+JM_4M_2)\frac{1}{J}\sum_{i=1}^{J}z_i^e+\mu\frac{J+1}{2}J\left(\frac{J}{2}M_4M_3+\mu M_6\right)\frac{1}{J}\sum_{n=0}^{J-1}w_n^{e-1}\qquad\text{(S.51)}$$

### F. Bound for $w_j^e$

We now proceed to obtain a recursion for $w_i^e$. Similarly, as we did in (S.49) but summing from $i$ to $J-1$, we write:

$$w_i^e\preceq(J-i)M_4\sum_{j=i}^{J-1}y_j^e+(J-i)\mu^2M_5\sum_{j=i}^{J-1}z_j^e+(J-i)^2\mu^2M_6J^{-1}\sum_{n=1}^{J}w_{n-1}^{e-1}\qquad\text{(S.52)}$$

$$
\begin{aligned}
w_i^e\preceq&(J-i)M_4\sum_{j=i}^{J-1}\left(M_1^jy_0^e+\mu M_2\sum_{n=0}^{j-1}z_n^e+\mu jM_3\frac{1}{J}\sum_{n=1}^{J}w_{n-1}^{e-1}\right)+(J-i)\mu^2M_5\sum_{j=i}^{J-1}z_j^e+(J-i)^2\mu^2M_6J^{-1}\sum_{n=1}^{J}w_{n-1}^{e-1}\\
=&(J-i)M_4\sum_{j=i}^{J-1}M_1^jy_0^e+(J-i)\mu M_4M_2\sum_{j=i}^{J-1}\sum_{n=0}^{j-1}z_n^e+(J-i)\mu^2M_5\sum_{j=i}^{J-1}z_j^e\\
&+(J-i)^2\left(\mu^2M_6+\mu\frac{(J-1+i)}{2}M_4M_3\right)J^{-1}\sum_{n=1}^{J}w_{n-1}^{e-1}\\
=&(J-i)M_4\sum_{j=i}^{J-1}M_1^jy_0^e+(J-i)\mu M_4M_2\sum_{j=1}^{J-2}\min(J-i,J-j-1)z_j^e+(J-i)\mu^2M_5\sum_{j=i}^{J-1}z_j^e
\end{aligned}
$$

$$+(J-i)^2\left(\mu^2 M_6+\mu\frac{(J-1+i)}{2}M_4 M_3\right)J^{-1}\sum_{n=1}^{J}w_{n-1}^{e-1}$$

$$\preceq(J-i)M_4\sum_{j=i}^{J-1}M_1^j y_0^e+(J-i)^2\mu M_4 M_2\sum_{j=1}^{J-2}z_j^e+(J-i)\mu^2 M_5\sum_{j=i}^{J-1}z_j^e$$

$$+\mu(J-i)^2\left(\mu M_6+\frac{(J-1+i)}{2}M_4 M_3\right)J^{-1}\sum_{n=1}^{J}w_{n-1}^{e-1}$$

$$\preceq(J-i)M_4(I-M_1)^{-1}y_0^e+(J-i)^2\mu M_4 M_2\sum_{j=1}^{J-2}z_j^e+(J-i)\mu^2 M_5\sum_{j=i}^{J-1}z_j^e$$

$$+\mu(J-i)^2\left(\mu M_6+\frac{(J-1+i)}{2}M_4 M_3\right)J^{-1}\sum_{n=1}^{J}w_{n-1}^{e-1} \tag{S.53}$$

Summing we finally get:

$$\frac{1}{J}\sum_{i=0}^{J-1}w_i^e\preceq J^{-1}\sum_{i=0}^{J-1}(J-i)M_4(I-M_1)^{-1}y_0^e+\frac{1}{J}\sum_{i=0}^{J-1}(J-i)^2\mu M_4 M_2\sum_{j=1}^{J-2}z_j^e+\frac{1}{J}\sum_{i=0}^{J-1}(J-i)\mu^2 M_5\sum_{j=i}^{J-1}z_j^e$$

$$+J^{-1}\sum_{i=0}^{J-1}\mu(J-i)^2\left(\mu M_6+\frac{(J-1+i)}{2}M_4 M_3\right)\frac{1}{J}\sum_{n=1}^{J}w_{n-1}^{e-1}$$

$$=J^{-1}\sum_{i=1}^{J}iM_4(I-M_1)^{-1}y_0^e+\sum_{i=1}^{J}i^2\mu M_4 M_2\frac{1}{J}\sum_{j=1}^{J}z_j^e+\sum_{i=0}^{J-1}(J-i)\mu^2 M_5\frac{1}{J}\sum_{j=i}^{J-1}z_j^e$$

$$+J^{-1}\mu\sum_{i=1}^{J}i^2(\mu M_6+(J-1)M_4 M_3)\frac{1}{J}\sum_{n=1}^{J}w_{n-1}^{e-1}$$

$$\preceq J^{-1}\sum_{i=1}^{J}iM_4(I-M_1)^{-1}y_0^e+\sum_{i=1}^{J}i^2\mu M_4 M_2\frac{1}{J}\sum_{j=1}^{J}z_j^e+\sum_{i=1}^{J}i\mu^2 M_5\frac{1}{J}\sum_{j=1}^{J-1}z_j^e$$

$$+J^{-1}\mu\sum_{i=1}^{J}i^2(\mu M_6+(J-1)M_4 M_3)\frac{1}{J}\sum_{n=1}^{J}w_{n-1}^{e-1}$$

$$\preceq\frac{J+1}{2}M_4(I-M_1)^{-1}y_0^e+\mu(J+1)\frac{J}{2}\left(\mu M_5+JM_4 M_2\right)\frac{1}{J}\sum_{j=1}^{J}z_j^e$$

$$+\mu(J+1)\frac{J}{2}(\mu M_6+(J-1)M_4 M_3)\frac{1}{J}\sum_{n=1}^{J}w_{n-1}^{e-1} \tag{S.54}$$

At this point we have recursions for $y_0^{e+1}$, $J^{-1}\sum_{i=1}^{J}z_i^e$ and $J^{-1}\sum_{i=0}^{J-1}w_i^e$ which we rewrite for convenience.

$$y_0^{e+1}\preceq M_1^J y_0^e+\mu J M_2\frac{1}{J}\sum_{i=1}^{J}z_i^e+\mu J M_3\frac{1}{J}\sum_{i=0}^{J-1}w_i^{e-1} \tag{S.55}$$

$$\frac{1}{J}\sum_{i=1}^{J}z_i^e\preceq\frac{J+1}{2}M_4(I-M_1)^{-1}y_0^e+\mu(J+1)\frac{J}{2}\left(\mu M_5+JM_4 M_2\right)\frac{1}{J}\sum_{i=1}^{J}z_i^e+\mu(J+1)\frac{J}{2}\left(\mu M_6+JM_4 M_3\right)\frac{1}{J}\sum_{i=0}^{J-1}w_i^{e-1} \tag{S.56}$$

$$\frac{1}{J}\sum_{i=0}^{J-1}w_i^e\preceq\frac{J+1}{2}M_4(I-M_1)^{-1}y_0^e+\mu(J+1)\frac{J}{2}\big(\mu M_5+JM_4M_2\big)\frac{1}{J}\sum_{i=1}^{J}z_i^e+\mu(J+1)\frac{J}{2}\big(\mu M_6+JM_4M_3\big)\frac{1}{J}\sum_{i=0}^{J-1}w_i^{e-1}$$
(S.57)

We now make the following definitions:

$$y^e\overset{\Delta}{=}y_0^e\in\mathbb{R}^2$$
(S.58a)

$$z^e\overset{\Delta}{=}\big[\ 1\quad\|\mathcal{H}_u\|^2\ \big]\,J^{-1}\sum_{n=1}^{J}z_n^e\in\mathbb{R}$$
(S.58b)

$$w^e\overset{\Delta}{=}\big[\ 1\quad\|\mathcal{H}_u\|^2\ \big]\,J^{-1}\sum_{n=0}^{J-1}w_n^e\in\mathbb{R}$$
(S.58c)

Noting that

$$\mu M_5+JM_4M_2=\mu\underbrace{\begin{bmatrix}6\|Z^{-1}\|^2\max_{k,n}(\|G_{k,n}\|^2)\\+3\mu J\big(K^{-2}\|\Lambda_G\|^2a_3+\mu a_7K^{-1}\|Z^{-1}\mathcal{I}^T\mathcal{G}\mathcal{H}_u\|^2\big)\\a_7+J\Big(\mu a_3a_6+a_7\big(1-\sqrt{\lambda_2(\overline{L})}\big)+\mu^2a_5\Big)\end{bmatrix}}_{\overset{\Delta}{=}\begin{bmatrix}t_1\\t_2\end{bmatrix}}\begin{bmatrix}1\\\|\mathcal{H}_u\|^2\end{bmatrix}^T$$
(S.59)

$$\mu M_6+JM_4M_3=\mu\underbrace{\begin{bmatrix}6\|Z^{-1}\|^2\max_{k,n}(\|G_{k,n}\|^2)\\+3\mu J\big(K^{-2}\|\Lambda_G\|^2a_4+\mu a_8K^{-1}\|Z^{-1}\mathcal{I}^T\mathcal{G}\mathcal{H}_u\|^2\big)\\a_8+J\Big(\mu a_4a_6+a_8\big(1-\sqrt{\lambda_2(\overline{L})}\big)+\mu^2a_5\Big)\end{bmatrix}}_{\overset{\Delta}{=}\begin{bmatrix}t_3\\t_4\end{bmatrix}}\begin{bmatrix}1\\\|\mathcal{H}_u\|^2\end{bmatrix}^T$$
(S.60)

using (S.38e) and (S.38f) and multiplying both sides of (S.56) and (S.57) by $\big[\ 1\quad\|\mathcal{H}_u\|^2\ \big]$ we get:

$$y^{e+1}\preceq M_1^J y^e+\mu J\begin{bmatrix}a_3\\\mu a_7\end{bmatrix}z^e+\mu J\begin{bmatrix}a_4\\\mu a_8\end{bmatrix}w^{e-1}$$
(S.61)

$$z^e\le\frac{J+1}{2}\begin{bmatrix}1\\\|\mathcal{H}_u\|^2\end{bmatrix}^T M_4(I-M_1)^{-1}y^e+\mu^2(J+1)\frac{J}{2}(t_1+\|\mathcal{H}_u\|^2t_2)z^e+\mu^2(J+1)\frac{J}{2}(t_3+\|\mathcal{H}_u\|^2t_4)w^{e-1}$$
(S.62)

$$w^e\le\frac{J+1}{2}\begin{bmatrix}1\\\|\mathcal{H}_u\|^2\end{bmatrix}^T M_4(I-M_1)^{-1}y^e+\mu^2(J+1)\frac{J}{2}(t_1+\|\mathcal{H}_u\|^2t_2)z^e+\mu^2(J+1)\frac{J}{2}(t_3+\|\mathcal{H}_u\|^2t_4)w^{e-1}$$
(S.63)

we now impose the following constraint on $\mu$:

$$\mu<\sqrt{\frac{2}{(J+1)J(t_1+\|\mathcal{H}_u\|^2t_2)}}$$
(S.64)

which implies that $(1-\mu^2K^{-1}(J+1)J(t_1+\|\mathcal{H}_u\|^2t_2))>0$ and hence we can derive the following inequality for $z^e$:

$$\underbrace{\big(1-\mu^2 2^{-1}(J+1)J(t_1+\|\mathcal{H}_u\|^2t_2)\big)}_{\overset{\Delta}{=}d_1}z^e\le 2^{-1}(J+1)\big(\big[\ 1\quad\|\mathcal{H}_u\|^2\ \big]M_4(I-M_1)^{-1}y^e+\mu^2J(t_3+\|\mathcal{H}_u\|^2t_4)w^e\big)$$
(S.65)

$$z^e \leq \frac{2^{-1}(J+1)}{d_1}\left(\begin{bmatrix} 1 & \|\mathcal{H}_u\|^2 \end{bmatrix} M_4(I-M_1)^{-1}y^e + \mu^2 J(t_3+\|\mathcal{H}_u\|^2 t_4)w^{e-1}\right) \tag{S.66}$$

Replacing (S.66) into (S.61) and (S.63) we get

$$y^{e+1} \preceq \left(M_1^J + \mu(2d_1)^{-1}(J+1)JM_2M_4(I-M_1)^{-1}\right)y^e$$
$$+ \mu J\left(\begin{bmatrix} a_4 \\ \mu a_8 \end{bmatrix} + \mu^2 J(2d_1)^{-1}(J+1)(t_3+\|\mathcal{H}_u\|^2 t_4)\begin{bmatrix} a_3 \\ \mu a_7 \end{bmatrix}\right)w^{e-1} \tag{S.67}$$

$$w^e \leq \frac{J+1}{2}\left(1+\mu^2(2d_1)^{-1}(J+1)J(t_1+\|\mathcal{H}_u\|^2 t_2)\right)\begin{bmatrix} 1 \\ \|\mathcal{H}_u\|^2 \end{bmatrix}^T M_4(I-M_1)^{-1}y^e$$
$$+ \mu^2(J+1)\frac{J}{2}(t_3+\|\mathcal{H}_u\|^2 t_4)\left(1+\mu^2(2d_1)^{-1}(J+1)J(t_1+\|\mathcal{H}_u\|^2 t_2)\right)w^{e-1}$$
$$= \frac{J+1}{2d_1}\left(\begin{bmatrix} 1 \\ \|\mathcal{H}_u\|^2 \end{bmatrix}^T M_4(I-M_1)^{-1}y^e + \mu^2 J(t_3+\|\mathcal{H}_u\|^2 t_4)w^{e-1}\right) \tag{S.68}$$

Or written in matrix form:

$$\begin{bmatrix} y^{e+1} \\ w^e \end{bmatrix} \preceq \begin{bmatrix} M_1^J+\mu f_1 JM_2M_4(I-M_1)^{-1} & \mu J\begin{bmatrix} a_4+\mu^2 f_2 a_3 \\ \mu a_8+\mu^3 f_2 a_7 \end{bmatrix} \\ f_1\begin{bmatrix} 1 & \|\mathcal{H}_u\|^2 \end{bmatrix}M_4(I-M_1)^{-1} & \mu^2 f_2 \end{bmatrix}\begin{bmatrix} y^e \\ w^{e-1} \end{bmatrix} \tag{S.69}$$

where we defined:

$$f_1 \overset{\Delta}{=} (2d_1)^{-1}(J+1) \tag{S.70a}$$

$$f_2 \overset{\Delta}{=} (2d_1)^{-1}(J+1)J(t_3+\|\mathcal{H}_u\|^2 t_4) \tag{S.70b}$$

which completes the proof.

### G. Proof Theorem 1

We start by multiplying both sides of (S.69) by $\begin{bmatrix} 1 & 1 & \mu\epsilon \end{bmatrix}$ as follows

$$\begin{bmatrix} 1 \\ 1 \\ \mu\epsilon \end{bmatrix}^T \begin{bmatrix} \|\tilde{\boldsymbol{x}}_0^{e+1}\|^2 \\ \|\hat{\boldsymbol{x}}_0^{e+1}\|^2 \\ w^e \end{bmatrix} \leq \begin{bmatrix} b_1 & b_2 & \mu\epsilon b_3 \end{bmatrix}\begin{bmatrix} \|\tilde{\boldsymbol{x}}_0^e\|^2 \\ \|\hat{\boldsymbol{x}}_0^e\|^2 \\ w^{e-1} \end{bmatrix} \leq \max(b_1,b_2,b_3)\begin{bmatrix} 1 \\ 1 \\ \mu\epsilon \end{bmatrix}^T \begin{bmatrix} \|\tilde{\boldsymbol{x}}_0^e\|^2 \\ \|\hat{\boldsymbol{x}}_0^e\|^2 \\ w^{e-1} \end{bmatrix} \tag{S.71}$$

where $\epsilon$ is positive constant which we will define later and $b_1$, $b_2$ and $b_3$ are:

$$\begin{bmatrix} b_1 & b_2 \end{bmatrix} = \begin{bmatrix} 1 & 1 \end{bmatrix}M_1^J + \mu f_1\left(J\begin{bmatrix} 1 & 1 \end{bmatrix}M_2 + \epsilon\begin{bmatrix} 1 & \|\mathcal{H}_u\|^2 \end{bmatrix}\right)M_4(I-M_1)^{-1}$$
$$= \begin{bmatrix} 1 & 1 \end{bmatrix}M_1^J + \mu f_1\left(J(a_3+\mu a_7)+\epsilon\right)\begin{bmatrix} 1 & \|\mathcal{H}_u\|^2 \end{bmatrix}M_4(I-M_1)^{-1} \tag{S.72}$$

$$b_3 = \epsilon^{-1}J(a_4+\mu^2 f_2 a_3)+\epsilon^{-1}J(\mu a_8+\mu^3 f_2 a_7)+\mu^2 f_2 = \epsilon^{-1}J(a_4+\mu a_8)+\mathcal{O}(\mu^2) \tag{S.73}$$

Note that if we constrain the step-size $\mu$ to be small enough so that $1-\mu K^{-1}a_1+\mu^2 a_6<1$ and $\sqrt{\lambda_2(\overline{L})}+\mu^2 a_5+\mu a_2<1$, we get that $\begin{bmatrix} 1 & 1 \end{bmatrix}M_1^J \preceq \begin{bmatrix} 1 & 1 \end{bmatrix}M_1$ and hence we can further write:

$$\begin{bmatrix} b_1 & b_2 \end{bmatrix} \preceq \begin{bmatrix} 1 & 1 \end{bmatrix}M_1 + \mu f_1(J(a_3+\mu a_7)+\epsilon)\begin{bmatrix} 1 & \|\mathcal{H}_u\|^2 \end{bmatrix}M_4(I-M_1)^{-1} \tag{S.74}$$

The objective now is to prove that $\max(b_1,b_2,b_3)<1$ and hence the algorithm converges linearly to the minimizer. We now proceed to obtain expressions for $b_1$ and $b_2$. For this, we expand the terms in (S.74). For this we expand $\begin{bmatrix} 1 & \|\mathcal{H}_u\|^2 \end{bmatrix}M_4(I-M_1)^{-1}$:

$$(I-M_1)^{-1} = \begin{bmatrix} \mu K^{-1}a_1 & -\mu a_2 \\ -\mu^2 a_6 & 1-\sqrt{\lambda_2(\overline{L})}-\mu^2 a_5 \end{bmatrix}^{-1}$$

$$= \underbrace{\frac{1}{a_1 K^{-1}\left(1-\sqrt{\lambda_2(\bar L)}-\mu^2 a_5\right)-\mu^2 a_2 a_6}}_{\triangleq d_2} \begin{bmatrix} \frac{1-\sqrt{\lambda_2(\bar L)}}{\mu}-\mu a_5 & a_2 \\ \mu a_6 & K^{-1}a_1 \end{bmatrix} \qquad (S.75)$$

$$\begin{bmatrix} 1 \\ \|\mathcal{H}_u\|^2 \end{bmatrix}^T M_4(I-M_1)^{-1}$$

$$=\frac{1}{d_2}\begin{bmatrix} 3\mu^2 K^{-2}\|\Lambda_G\|^2+\mu^2 a_6\|\mathcal{H}_u\|^2 \\ 3\mu^2 K^{-1}\|Z^{-1}\mathcal{I}^T\mathcal{G}\mathcal{H}_u\|^2+\|\mathcal{H}_u\|^2\left(1-\sqrt{\lambda_2(\bar L)}+\mu^2 a_5\right) \end{bmatrix}^T \begin{bmatrix} \frac{1-\sqrt{\lambda_2(\bar L)}}{\mu}-\mu a_5 & a_2 \\ \mu a_6 & K^{-1}a_1 \end{bmatrix}$$

$$=\frac{1}{d_2}\begin{bmatrix} \mu\left(1-\sqrt{\lambda_2(\bar L)}\right)(3K^{-2}\|\Lambda_G\|^2+2a_6\|\mathcal{H}_u\|^2) \\ +3\mu^3\left(a_6 K^{-1}\|Z^{-1}\mathcal{I}^T\mathcal{G}\mathcal{H}_u\|^2-K^{-2}\|\Lambda_G\|^2 a_5\right) \\ K^{-1}\|\mathcal{H}_u\|^2\left(1-\sqrt{\lambda_2(\bar L)}\right)a_1+\mu^2(3K^{-2}\|\Lambda_G\|^2+a_6\|\mathcal{H}_u\|^2)a_2 \\ +\mu^2\left(K^{-1}\|Z^{-1}\mathcal{I}^T\mathcal{G}\mathcal{H}_u\|^2+a_5\|\mathcal{H}_u\|^2\right)K^{-1}a_1 \end{bmatrix}^T$$
$$(S.76)$$

Now using (S.74) and (S.76) and setting $\epsilon=\frac{J(a_4+\mu a_8)}{1-\mu K^{-1}a_1+\mu^2 a_6}$ we can write the following expressions for $b_1$, $b_2$ and $b_3$:

$$b_1=1-\mu K^{-1}a_1+\mathcal{O}(\mu^2) \qquad (S.77a)$$

$$b_2=\sqrt{\lambda_2(\bar L)}+\mathcal{O}(\mu) \qquad (S.77b)$$

$$b_3=1-\mu K^{-1}a_1+\mathcal{O}(\mu^2) \qquad (S.77c)$$

From (S.77a) and (S.77c) it is clear that for small enough step-size $\mu$, $b_1<1$ and $b_3<1$. Also since $\lim_{\mu\to 0}b_2=\sqrt{\lambda_2(\bar L)}$, then it is clear that there is a step-size $\mu$ small enough such that $b_2<1$. Like we said before since $\max(b_1,b_2,b_3)<1$ this implies that $\begin{bmatrix} \|\check{\boldsymbol{x}}_0^{e+1}\|^2 & \|\hat{\boldsymbol{x}}_0^{e+1}\|^2 & w^e \end{bmatrix}$ converges linearly to $\begin{bmatrix} 0 & 0 & 0 \end{bmatrix}$ with a rate bounded by $\max[1-\mu K^{-1}a_1+\mathcal{O}(\mu^2),\sqrt{\lambda_2(\bar L)}+\mathcal{O}(\mu)]$. Remembering the definition of $a_1$ in (S.28) we can equivalently bound the convergence rate by $\max[\rho(I-\mu K^{-1}\Lambda_G)+\mathcal{O}(\mu^2),\sqrt{\lambda_2(\bar L)}+\mathcal{O}(\mu)]$.


\end{document}